\documentclass[twoside,11pt]{article}
\usepackage[abbrvbib, preprint]{jmlr2e}

\usepackage[normalem]{ulem}
\usepackage{soul}

\usepackage{graphicx,xcolor,caption,enumitem}
\usepackage{wrapfig}

\usepackage{mathtools}
\usepackage[textsize=tiny]{todonotes}

\usepackage{enumerate,color,xcolor}
\usepackage{graphicx}
\usepackage{subfigure}
\usepackage{url}
\usepackage{multirow}
\usepackage{hhline}

\usepackage{theoremref}
\usepackage{bbm}

\usepackage{todonotes}

\begin{document}
\title{Convergence Analysis for General Probability Flow ODEs of Diffusion Models in
Wasserstein Distances}




\author{%
 \name Xuefeng Gao \email xfgao@se.cuhk.edu.hk\\
 \addr Department of Systems Engineering and Engineering Management\\ The Chinese University of Hong Kong, Shatin, N.T. Hong Kong.
  \AND
\name Lingjiong Zhu  \email zhu@math.fsu.edu \\
 \addr Department of Mathematics\\ Florida State University, Tallahassee, FL, USA.
}

\maketitle

\begin{abstract}
Score-based generative modeling with probability flow ordinary differential equations (ODEs) has achieved remarkable success in a variety of applications. While various fast ODE-based samplers have been proposed in the literature and employed in practice, the theoretical understandings about convergence properties of the probability flow ODE are still quite limited. In this paper, we provide the first non-asymptotic convergence analysis for a general class of probability flow ODE samplers in 2-Wasserstein distance, assuming accurate score estimates and smooth log-concave data distributions. We then consider various examples and establish results on the iteration complexity of the corresponding ODE-based samplers. Our proof technique relies on spelling out explicitly the contraction rate for the continuous-time ODE and analyzing the discretization and score-matching errors using synchronous
 coupling; the challenge in our analysis mainly arises from the inherent non-autonomy of the probability flow ODE and the specific exponential integrator that we study. 
\end{abstract}


\section{Introduction}

Score-based generative models (SGMs) \citep{Sohl2015, SongErmon2019, Ho2020, SongICLR2021}, or diffusion models, have achieved remarkable success in a range of applications, particularly in the realm of image and audio generation \citep{rombach2022high, ramesh2022hierarchical, popov2021grad}. These models employ a unique approach where samples from a target data distribution are progressively corrupted with noise through a forward process. Subsequently, the models learn to reverse this corrupted process in order to generate new samples.

In this paper, we aim to provide theoretical guarantees for the probability flow ODE (ordinary differential equation) implementation of SGMs, proposed initially in \cite{SongICLR2021}. The forward process in SGMs, denoted by
$(\mathbf{x}_{t})_{t\in[0,T]}$,  follows the stochastic differential equation (SDE):
\begin{equation}\label{OU:SDE}
\mathrm{d}\mathbf{x}_{t}=-f(t)\mathbf{x}_{t}\mathrm{d}t+g(t)\mathrm{d}\mathbf{B}_{t}, \quad  \mathbf{x}_{0}\sim p_{0},
\end{equation}
where both $f(t)$ and $g(t)$ are scalar-valued non-negative continuous functions of time $t$, $(\mathbf{B}_{t})$ is the standard $d$-dimensional Brownian motion, and $p_0$ is the $d-$dimensional (unknown) target data distribution. Two popular choices of forward processes in the literature are Variance Exploding (VE) SDEs and Variance Preserving (VP) SDEs \citep{SongICLR2021}; see Section~\ref{sec:prelim} for more details. 
If we denote $p_{t}(\mathbf{x})$ as the probability density function of $\mathbf{x}_t$ in \eqref{OU:SDE}, then \cite{SongICLR2021} showed that there exists an ODE:
\begin{align} 
&\frac{\mathrm{d}\tilde{\mathbf{x}}_{t}}{\mathrm{d}t}=f(T-t)\tilde{\mathbf{x}}_{t}+ \frac{1}{2}(g(T-t))^{2}\nabla_{\mathbf{x}}\log p_{T-t}(\tilde{\mathbf{x}}_{t}),
\nonumber
\\
&\tilde{\mathbf{x}}_{0}\sim p_{T},\label{eq:ODEReverse}
\end{align}
where the solution $\tilde{\mathbf{x}}_{t}$ at time $t\in[0,T]$, is distributed according to $p_{T-t}$. 
The ODE \eqref{eq:ODEReverse} is called the \textit{probability flow ODE}.  
Note that the probability flow ODE involves the \textit{score function}, $\nabla_{\mathbf{x}}\log p_{t}(\mathbf{x})$, which is unknown. In practice, it can be approximated using neural networks which are trained with appropriate score-matching techniques \citep{ hyvarinen2005estimation, vincent2011connection, song2020sliced}. Once the score function is estimated, one can sample $\tilde{\mathbf{x}}_{0}$ from a normal distribution to initialize the ODE, and numerically solve the ODE forward in time with any ODE solvers such as Euler \cite{SongICLR2021} or Heun's 2nd order method \cite{Karras2022}. The resulting sample generated at time $T$ can be viewed as an approximate sample from the target distribution since $\tilde{\mathbf{x}}_{T} \sim p_0$.

A large body of work on diffusion models has recently investigated various methods for faster generation of samples based on the probability flow ODE; see, e.g., \citep{lu2022dpm, Karras2022, zhang2023fast, zhao2023unipc, song2023consistency}. While these methods are quite effective in practice, the theoretical understandings of this probability flow ODE approach are still quite limited. To our best knowledge, \cite{chen2023restoration} established the first non-asymptotic convergence guarantees for the probability flow ODE sampler in the Kullback-Leibler (KL) divergence, but it did not provide concrete polynomial dependency on the dimension $d$ and $1/\epsilon$, where $\epsilon$ is the prescribed error between the data distribution and the generated distribution. \cite{chen2023probability} considered a specific VP-SDE as the forward process (where $f\equiv 1$ and $g\equiv\sqrt{2}$), and provided polynomial-time guarantees for some variants of the probability flow ODE in the total variation {(TV)} distance. Specifically, the algorithms they analyzed rely on the use of stochastic corrector steps, so the samplers are not fully deterministic. \cite{li2023towards} also considered a specific VP-SDE as the forward processes and analyzed directly a discrete-time version of probability flow ODEs and obtained convergence guarantees for fully deterministic samplers in TV. Very recently, 
\cite{li2024accelerating} extended \cite{li2023towards} and 
established nonasymptotic convergence guarantees in TV for the accelerated DDIM-type deterministic sampler.



These theoretical studies have mostly focused on convergence analysis of probability flow ODE samplers in TV distance between the data distribution and the generated distribution. However, practitioners are often interested in the 2-Wasserstein ($\mathcal{W}_{2}$) distance. For instance, in image-related tasks, Fr\'{e}chet Inception Distance (FID) is a widely adopted performance metric for evaluating the quality of generated samples, where FID measures the $\mathcal{W}_{2}$ distance between the distribution of generated images and the distribution of real images \cite{heusel2017gans}. We also note that TV distance does not upper bound 2-Wasserstein distance on $\mathbb{R}^{d}$ (see e.g. \cite{GibbsSu2002}).
In addition, previous works \citep{chen2023probability, li2023towards, li2024accelerating} have studied specific choices of $f$ and $g$ in their convergence analysis of ODE-based samplers. However, it has been shown in \cite{SongICLR2021} that the empirical performance of probability flow ODE samplers depends crucially on the choice of $f$ and $g$ in \eqref{eq:ODEReverse}, which indicates the importance of selecting these parameters (noise schedules) of diffusion models.  
This leads to the following question which we study in this paper:
\begin{center}
Can we establish Wasserstein convergence guarantees for probability flow ODE samplers 
 with general functions $f$ and $g$? 
\end{center}

\textbf{Our Contributions.}  
\begin{itemize}
\item We establish non-asymptotic convergence guarantees for a general class of probability flow ODEs in 2-Wasserstein distance, assuming that the score function can be accurately learned and the data distribution has a smooth and strongly log-concave density (Theorem~\ref{thm:discrete:2}). In particular, we allow general functions $f$ and $g$ in the probability flow ODE \eqref{eq:ODEReverse}, and our results apply to both VP and VE SDE models.  
Theorem~\ref{thm:discrete:2} directly translates to an upper bound on the iteration complexity, which is the number of sampling steps needed for the ODE sampler to yield $\epsilon-$accuracy in 2-Wasserstein distance between the data distribution and the generative distribution of the SGMs. 

\item We specialize our result to ODE samplers with specific functions $f$ and $g$ that are commonly used in the literature, and we find the complexity of VE-SDEs is worse than that of VP-SDEs for the examples we analyze (see Table~\ref{table:summary:complexity} for details). This theoretical finding is consistent with the empirical observation in \cite{SongICLR2021}, where they found that the empirical performance of probability flow ODE samplers depends on the choice of $f$ and $g$, and the sample quality for VE-SDEs is much worse than VP-SDEs for high-dimensional data. 

\item We obtain an iteration complexity bound $\widetilde{\mathcal{O}}\left( \sqrt{d}/\epsilon \right)$ of the ODE sampler for each of the three examples of VP-SDEs studied,
where $d$ is the dimension of the data distribution and $\widetilde{\mathcal{O}}$ ignores the logarithmic factors and hides dependency on other parameters. We also show that (see Proposition~\ref{prop:lower:bound}) under mild assumptions there are no other choices of $f$ and $g$ so that the iteration complexity can be better than $\widetilde{\mathcal{O}}\left( \sqrt{d}/\epsilon \right)$.  

\item 
Our main proof technique relies on spelling out an explicit contraction rate in the continuous-time ODE
and providing a careful analysis in controlling the discretization and score-matching errors.
Our proof technique is inspired by iteration complexity results in the context
of Langevin algorithms for sampling in the literature \citep{DK2017}; the smooth strong log-concavity of data distribution allows us to obtain contraction for the probability flow ODE using the synchronous coupling. Yet our analysis is more sophisticated
and subtle even under the smooth strong log-concavity assumption. 
First, in the literature of sampling strongly-log-concave target densities,
where Langevin algorithms are often used, the underlying dynamics is time-homogeneous,
and the strong log-concavity and smoothness assumptions can be directly used (see e.g. \cite{DK2017}), 
whereas for probability flow ODEs, it is non-autonomous,
and the strong-log-concavity and smoothness constants need to be carefully analyzed and spelled out
which are time-dependent. Second, the usage of exponential integrator
also makes the analysis more subtle, as compared to the Euler discretization of the ODE. We incorporate the exponential integrator and the interpolation in the discretization to design an continuous-time auxiliary ODE, and then control the error between this auxiliary ODE and the probability flow ODE. 
Finally, in order to further analyze the VE-SDE and VP-SDE examples in Section~\ref{sec:examples} (see Corollaries~\ref{cor:VE:5:main:paper}-\ref{cor:VP:6:main:paper}, Proposition~\ref{prop:lower:bound}),
the inherent non-autonomous property of the probability flow ODE makes the analysis much trickier and subtle. One has to perform
a series of inequalities based on the formula in Theorem~\ref{thm:discrete:2} in order to spell out the leading
order terms to obtain the iteration complexities.
\end{itemize}

\subsection{Related Work}

In addition to the deterministic sampler based on probability flow ODEs, another major family of diffusion samplers is based on discretization of the reverse-time SDE, which is obtained by reversing the forward process \eqref{OU:SDE} in time \citep{Anderson1982, cattiaux2023time}). This leads to SDE-based stochastic samplers due to the noise in the reverse-time SDE. Compared with SDE-based samplers,  ODE-based deterministic samplers are often claimed to converge much faster, at the cost of slightly inferior sample quality; see e.g. \cite{yang2022diffusion}. 
In recent years, there has been a significant surge in research focused on the convergence theory of SDE-based stochastic samplers for diffusion models, particularly when assuming access to precise estimates of the score function; see, e.g., \cite{block2020generative, de2021diffusion, de2022convergence, leeconvergence, lee2023convergence, chen2022improved, chen2022sampling, li2023towards, chen2023score, gao2023wasserstein, benton2023linear, tang24}. These studies have mostly focused on the convergence analysis of SDE-based stochastic samplers in TV or KL divergence. \citep{de2022convergence, chen2022improved, chen2022sampling} provided Wasserstein
convergence bound for the SDE-based sampler for the DDPM model in \cite{Ho2020} for bounded data distribution, in
which case the 2-Wasserstein distance can be bounded by the TV distance.
 \cite{gao2023wasserstein} established convergence guarantees for SDE-based samplers for a general class of SGMs in 2-Wasserstein distance, for unbounded smooth log-concave data distribution. Our study differs from these studies in that we consider deterministic samplers based on the probability flow ODE implementation of SGMs.

Our work is also broadly related to flow based methods or flow matching, see e.g. \cite{lipman2022flow, albergo2022building}. The flow matching framework is more general than the probability flow ODE, because it approximates a flow between two arbitrary distributions. In particular, flow matching reduces to probability flow ODE for diffusion models when the starting distribution is Gaussian. The theoretical analysis of flow matching methods are still limited, and there are currently very few error bounds in Wasserstein distance for flow matching methods with a fully deterministic sampling scheme. Two very recent studies on such bounds are \cite{benton2023error} and \cite{albergo2023stochastic}. While both studies have considered the errors due to the approximate flow/velocity, they do not consider the error arising from numerically solving the ODE. Hence, these results are not about the iteration complexity of the sampling scheme, which we study for the probability flow ODE. We also mention that \cite{cheng2023convergence} provided convergence guarantees for a progressive flow
model, which differs from score-based diffusion models in that the flow model they consider is deterministic in the forward (data-to-noise) process. Finally,
\cite{nie2024blessing} showed that the KL divergence between the marginal distributions of two probability flow ODEs (with estimated score functions) with mismatched prior distributions remains constant as the time approaches zero.  In contrast, we show the contraction of probability flow ODEs in Wasserstein distance for unbounded (strongly log-concave) data distributions.

\textbf{Notations.} 
\begin{itemize}
\item
For any $d$-dimensional random vector $\mathbf{x}$ with finite second moment, 
the $L_{2}$-norm of $\mathbf{x}$ is defined as
$\Vert\mathbf{x}\Vert_{L_{2}}=\sqrt{\mathbb{E}\Vert\mathbf{x}\Vert^{2}}$, 
where $\Vert\cdot\Vert$ denotes the Euclidean norm. 
\item 
We denote $\mathcal{L}(\mathbf{x})$ as the law of $\mathbf{x}$.
\item 
For any two Borel probability measures $\mu_{1},\mu_{2}\in\mathcal{P}_{2}(\mathbb{R}^{d})$, 
the space consisting of all the Borel probability measures 
on $\mathbb{R}^{d}$ with the finite second moment
(based on the Euclidean norm), 
the standard $2$-Wasserstein
distance \cite{villani2008optimal} is defined by
$$
\mathcal{W}_{2}(\mu_{1},\mu_{2}):=\sqrt{\inf\mathbb{E}\left[\Vert\mathbf{x}_{1}-\mathbf{x}_{2}\Vert^{2}\right]},
$$
where the infimum is taken over all joint distributions of the random vectors $\mathbf{x}_{1},\mathbf{x}_{2}$ with marginal distributions $\mu_{1},\mu_{2}$. 
\item 
A differentiable function $F$ from $\mathbb{R}^{d}$ to $\mathbb{R}$ is said to be
$\mu$-strongly convex and $L$-smooth (i.e. $\nabla F$ is $L$-Lipschitz) if for every $\mathbf{u},\mathbf{v}\in\mathbb{R}^{d}$,
\begin{align*}
\frac{\mu}{2}\Vert \mathbf{u}-\mathbf{v}\Vert^{2} 
\leq F(\mathbf{u})-F(\mathbf{v}) - \nabla F(\mathbf{v})^{\top}(\mathbf{u}-\mathbf{v})
\leq 
\frac{L}{2}\Vert\mathbf{u}-\mathbf{v}\Vert^{2}.
\end{align*}
\end{itemize}

\section{Preliminaries}\label{sec:prelim}

Recall $p_{0} \in \mathcal{P} (\mathbb{R}^d) $ denotes the unknown data distribution which has a density, where $\mathcal{P}(\mathbb{R}^{d})$ is the space of all probability measures on $\mathbb{R}^{d}$. Given i.i.d. samples from $p_{0}$, the problem of generative modelling is to generate new samples that (approximately) follow the data distribution.

We consider the probability flow ODE (see \eqref{eq:ODEReverse}) based implementation of SGMs for sample generation, see e.g. \cite{SongICLR2021}. The functions $f$ and $g$ in \eqref{eq:ODEReverse} can be general in our study, and in particular, our study covers the following two classes of models commonly used in the literature: (1) $f(t) \equiv 0$ and $g(t) = \sqrt{\frac{d[\sigma^2(t)]}{\mathrm{d}t}}$ for some nondecreasing function $\sigma(t)$, e.g., $g(t)=ae^{bt}$ for some positive constants $a, b$. The corresponding forward SDE \eqref{OU:SDE} is referred to as Variance Exploding (VE) SDE. (2) $f(t) = \frac{1}{2}\beta(t) $ and $g(t) = \sqrt{ \beta(t) }$ for some nondecreasing function $\beta(t)$, e.g., $\beta(t) = at +b$ for some positive constants $a, b.$ The corresponding forward SDE \eqref{OU:SDE} is referred to as Variance Preserving (VP) SDE.

To implement the ODE \eqref{eq:ODEReverse}, one needs to 
(a) sample from a tractable distribution (aka prior distribution) to initialize the ODE, 
(b) estimate the score function $\nabla_{\mathbf{x}}\log p_{t}(\mathbf{x})$, and 
(c) numerically solve/integrate the ODE.  
We next discuss these three issues (or error sources).

First, we discuss (a), the prior distribution. Note that $p_T$ is unknown. To provide an unifying analysis for probability flow ODEs with general $f$ and $g$ (including both VE and VP SDE models), we choose the prior distribution to be a normal distribution $\hat{p}_{T}$ given as follows: 
\begin{equation} \label{eq:hatp}
\hat{p}_{T}:=\mathcal{N}\left(0,\int_{0}^{T}e^{-2\int_{t}^{T}f(s)\mathrm{d}s}(g(t))^{2}\mathrm{d}t\cdot I_{d}\right),
\end{equation}
and $I_d$ is the $d$-dimensional identity matrix. To see why this is a reasonable choice, we note that the forward SDE \eqref{OU:SDE} has an explicit solution 
\begin{equation}\label{SDE:solution}
\mathbf{x}_{t}=e^{-\int_{0}^{t}f(s)\mathrm{d}s}\mathbf{x}_{0}+\int_{0}^{t}e^{-\int_{s}^{t}f(v)\mathrm{d}v}g(s)\mathrm{d}\mathbf{B}_{s}.
\end{equation}
Hence, we can the take the distribution of the Brownian integral $\int_{0}^{t}e^{-\int_{s}^{t}f(v)\mathrm{d}v}g(s)\mathrm{d}\mathbf{B}_{s}$ in \eqref{SDE:solution}, which is $\hat{p}_{T}$, as an approximation of $p_T$, so that 
\begin{equation*}
\mathcal{W}_{2}(p_{T},\hat{p}_{T})
\leq
e^{-\int_{0}^{T}f(s)\mathrm{d}s}\Vert\mathbf{x}_{0}\Vert_{L_{2}}. 
\end{equation*}
(See Lemma~\ref{lem:0} in the Appendix for the details.) 
Hence, we consider the ODE:
\begin{align}
&\frac{\mathrm{d}\mathbf{y}_{t}}{\mathrm{d}t}=f(T-t)\mathbf{y}_{t}+\frac{1}{2}(g(T-t))^{2}\nabla\log p_{T-t}(\mathbf{y}_{t}), 
\nonumber
\\
&\mathbf{y}_{0}\sim\hat{p}_{T}, \label{eq:zt}
\end{align}
as an approximation of the ODE \eqref{eq:ODEReverse} which starts from $p_T$. 


\begin{remark}
If the forward process is a VP-SDE with a stationary distribution which is normal, one can also take it as the prior distribution and 
our main result in this paper can be adapted to this setting. 
\end{remark}

We next consider (b) score matching, i.e., approximate the unknown true score function $\nabla_{\mathbf{x}}\log p_{t}(\mathbf{x})$ using a time-dependent score model $\boldsymbol{s}_{\theta}(\mathbf{x},t)$, which is often a deep neural network parameterized by $\theta$. To train the score model, one can use, for instance, 
denoising score matching \cite{SongICLR2021}, where the training objective for optimizing the neural network is given by
\begin{align}
\min_{\theta} \int_0^T  \Big[ \lambda(t) \mathbb{E}_{\mathbf{x}_{0}}  \mathbb{E}_{\mathbf{x}_{t} | \mathbf{x}_{0}} \Big\Vert \boldsymbol{s}_{\theta}(\mathbf{x}_t,t)
-\nabla_{\mathbf{x}_{t}}\log p_{t|0}(\mathbf{x}_{t} | \mathbf{x}_{0})\Big\Vert^2   \Big] \mathrm{d}t.\label{eq:score-matching}
\end{align}
Here, $\lambda(\cdot): [0, T] \rightarrow \mathbb{R}_{>0}$ is some positive weighting function, $\mathbf{x}_{0} \sim p_0$ is the data distribution, and $p_{t|0}(\mathbf{x}_{t} | \mathbf{x}_{0})$ is the density of $\mathbf{x}_{t}$ given $\mathbf{x}_{0}$, which is Gaussian due to the choice of the forward SDE in \eqref{OU:SDE}. Because one has i.i.d. samples from $p_0$, the distribution of $\mathbf{x}_{0}$, the objective in \eqref{eq:score-matching} can be approximated by Monte Carlo methods and the resulting loss function can be then optimized.


After the score function is estimated, we introduce a continuous-time process that approximates \eqref{eq:zt}: 
\begin{align}
&\frac{\mathrm{d}\mathbf{z}_{t}}{\mathrm{d}t}=f(T-t)\mathbf{z}_{t}+\frac{1}{2}(g(T-t))^{2}\boldsymbol{s}_{\theta}(\mathbf{z}_{t},T-t),
\nonumber
\\
&\mathbf{z}_{0}\sim\hat{p}_{T},\label{eq:u}
\end{align}
where we replace the true score function in \eqref{eq:zt} by the estimated score $\boldsymbol{s}_{\theta}$. 


Finally, we discuss (c) numerically solve the ODE \eqref{eq:u} for generation of new samples. There are various methods proposed and employed in practice, including Euler \citep{SongICLR2021}, Heun's 2nd order method \citep{Karras2022}, DPM solver \citep{lu2022dpm}, exponential integrator \citep{zhang2023fast}, to name just a few.
In this paper,  we consider the following exponential integrator (i.e. exactly integrating the linear part) discretization of the ODE \eqref{eq:u} for our theoretical convergence analysis. 
Let $\eta>0$ be the stepsize and without loss of generality,
let us assume that $T=K\eta$, where $K$ is a positive integer. 
Next, we introduce an exponential integrator discretization of the ODE \eqref{eq:u}.
By freezing the nonlinear nonlinear term (i.e. the $\boldsymbol{s}_{\theta}(\mathbf{z}_{t},T-t)$ term on RHS of \eqref{eq:u}) and letting the linear term flow (i.e. the $\mathbf{z}_{t}$ term on RHS of \eqref{eq:u}), we obtain the following ODE approximation of ODE \eqref{eq:u}: for any $(k-1)\eta\leq t<k\eta$,
\begin{equation*}
 \frac{\mathrm{d}\hat{\mathbf{z}}_{t}}{\mathrm{d}t}=f(T-t)\hat{\mathbf{z}}_{t}+\frac{1}{2}(g(T-t))^{2}\boldsymbol{s}_{\theta}(\hat{\mathbf{z}}_{(k-1)\eta},T-(k-1)\eta).
\end{equation*}
 By solving the above (linear) ODE for $(k-1)\eta\leq t<k\eta$, we have 
 \begin{align}
\hat{\mathbf{z}}_{k\eta}
&=e^{\int_{(k-1)\eta}^{k\eta}f(T-t)\mathrm{d}t}\hat{\mathbf{z}}_{(k-1)\eta}
\nonumber
\\
&\qquad
+\frac{1}{2}\boldsymbol{s}_{\theta}(\hat{\mathbf{z}}_{(k-1)\eta},T-(k-1)\eta)
\cdot\int_{(k-1)\eta}^{k\eta}e^{\int_{t}^{k\eta}f(T-s)\mathrm{d}s}(g(T-t))^{2}\mathrm{d}t. 
\end{align}
By letting $\mathbf{u}_{k} = \hat{\mathbf{z}}_{k\eta}$ for any $k$, 
we obtain the iterations for the exponential integrator discretization $(\mathbf{u}_{k})_{k=0}^{\infty}$ of the ODE \eqref{eq:u}: 
for any $k=1,2,\ldots,K$, 
\begin{align}
\mathbf{u}_{k}
&=e^{\int_{(k-1)\eta}^{k\eta}f(T-t)\mathrm{d}t}\mathbf{u}_{k-1}
\nonumber
\\
&\qquad
+\frac{1}{2}\boldsymbol{s}_{\theta}(\mathbf{u}_{k-1},T-(k-1)\eta)
\cdot\int_{(k-1)\eta}^{k\eta}e^{\int_{t}^{k\eta}f(T-s)\mathrm{d}s}(g(T-t))^{2}\mathrm{d}t, \label{eq:yk}
\end{align}
where $\mathbf{u}_{0}\sim\hat{p}_{T}$.


We are interested in the convergence of the generated distribution $\mathcal{L}(\mathbf{u}_{K})$ to the data distribution $p_{0}$, where $\mathcal{L}(\mathbf{u}_{K})$ denotes the law or distribution of $\mathbf{u}_{K}.$
 Specifically, 
our goal is to bound the 2-Wasserstein distance $\mathcal{W}_{2}(\mathcal{L}(\mathbf{u}_{K}),p_{0})$, and investigate the number of iterates $K$ that is needed
in order to achieve $\epsilon$ accuracy, i.e. $\mathcal{W}_{2}(\mathcal{L}(\mathbf{u}_{K}),p_{0})\leq\epsilon$. 


\section{Main Results}

In this section we state our main results. The proofs are deferred to the Appendix. 


\subsection{Assumptions}\label{sec:assump}

The first assumption is on the density of data distribution $p_{0}$, which implies that
$\mathbf{x}_{0} \sim p_0$ is $L_{2}$-integrable.

\begin{assumption}\label{assump:p0}
Assume that the density $p_{0}$ is differentiable and positive everywhere. 
Moreover, $-\nabla \log p_{0}$ is $m_{0}$-strongly convex and $L_{0}$-smooth for some $m_0, L_0>0$. 
\end{assumption}


The assumption of strong-log-concavity data distribution has also been imposed in \citep{bruno2023diffusion, gao2023wasserstein} for convergence analysis of diffusion models. 
We need this assumption mainly because we consider Wasserstein convergence analysis of ODE-based deterministic samplers. In particular, the ODE \eqref{eq:zt} may not have a contraction in 2-Wasserstein distance without such an assumption; see Remark~\ref{remark-conv} in Section~\ref{sec:mainresult} for details.

Our next assumption is about the true score function $\nabla_{\mathbf{x}}\log p_{t}(\mathbf{x})$ for $t \in [0, T]$.
We assume that the score function $\nabla_{\mathbf{x}}\log p_{t}(\mathbf{x})$
is Lipschtiz in time, where
the Lipschitz constant has at most
linear growth in $\Vert\mathbf{x}\Vert$. 
Assumption~\ref{assump:M:1} is needed in controlling the discretization error of the ODE \eqref{eq:u}. For Gaussian data distributions $p_0$, one can compute the score function $\nabla_{\mathbf{x}}\log p_{t}(\mathbf{x})$ analytically based on \eqref{SDE:solution} and readily verify that this assumption holds.


\begin{assumption}\label{assump:M:1}
There exists some constant $L_{1}$ such that for all $\mathbf{x}$:
\begin{align}
\sup_{\substack{1\leq k\leq K\\(k-1)\eta\leq t\leq k\eta}}\left\Vert\nabla\log p_{T-t}(\mathbf{x})-\nabla\log p_{T-(k-1)\eta}(\mathbf{x})\right\Vert
\leq L_{1}\eta(1+\Vert\mathbf{x}\Vert).
\end{align}
\end{assumption}

Most studies on convergence analysis of diffusion models assume some form of $L_2$ error for score learning and focus on the sampling phase. Our next assumption is on this score-matching error. Recall $(\mathbf{u}_{k})$ are the iterates defined in \eqref{eq:yk}. 

\begin{assumption}\label{assump:M}
Assume that there exists $M>0$ such that
\begin{align} \label{eq:assump:M}
\sup_{k=1,\ldots,K}
\Big\Vert\nabla\log p_{T-(k-1)\eta}\left(\mathbf{u}_{k-1}\right)
-\boldsymbol{s}_{\theta}\left(\mathbf{u}_{k-1},T-(k-1)\eta\right)\Big\Vert_{L_{2}}\leq M.
\end{align}
\end{assumption}


 The main result in our paper will still hold
if Assumption \ref{assump:M} is replaced by 
\begin{equation}\label{weakened:assump:2}
\sup_{k=1,\ldots,K}
\left\Vert\nabla\log p_{k\eta}\left(\mathbf{x}_{k\eta}\right)-\boldsymbol{s}_{\theta}\left(\mathbf{x}_{k\eta},k\eta\right)\right\Vert_{L_{2}}\leq M,
\end{equation}
 under the additional assumption that $\boldsymbol{s}_{\theta}(\cdot,k\eta)$ is Lipschitz for every $k$. The condition \eqref{weakened:assump:2} does not involve $(\mathbf{u}_k)$ and it could be easier to interpret. Moreover, Assumption~\ref{assump:M} is related to the score perturbation lemma in the seminal work \cite{chen2023probability}. However, they need to assume the Hessian of score $\nabla^2 \log p_t(\mathbf{x})$ is bounded by $L$ for any $t$ and $\mathbf{x}$, where $L$ is independent of $T$, to obtain the desired dependency on dimension $d$ and $\epsilon$.

\subsection{Main Result}\label{sec:mainresult}

We are now ready to state our main result. 

\begin{theorem}\label{thm:discrete:2}
Suppose that Assumptions~\ref{assump:p0}, \ref{assump:M:1} and \ref{assump:M} hold
and the stepsize $\eta\leq\bar{\eta}$, where $\bar{\eta}>0$ has an explicit formula
given in \eqref{bar:eta} in Appendix~\ref{sec:key:quantities}.
Then, 
\begin{align} 
\mathcal{W}_{2}(\mathcal{L}(\mathbf{u}_{K}),p_{0})
\leq \underbrace{e^{-\int_{0}^{K\eta}\mu(t)\mathrm{d}t}  \cdot \Vert\mathbf{x}_{0}\Vert_{L_{2}} }_{ \textstyle \text{Initialization error}} 
+\underbrace{ E_{1}(f,g, K, \eta, L_1)}_{\textstyle \text{Discretization error}}
+\underbrace{ E_{2}(f,g, K, \eta, M, L_1)}_{\textstyle \text{Score matching error}}.\label{main:thm:upper:bound}
\end{align}
Here, $\mu(t)$ is given in \eqref{c:t:defn} in Appendix~\ref{sec:key:quantities} and
\begin{align}
&E_{1}(f,g, K, \eta, L_1) 
:= \sum_{k=1}^{K}
\prod_{j=k+1}^{K}\gamma_{j,\eta}\cdot e^{\int_{k\eta}^{K\eta}f(T-t)\mathrm{d}t}\nonumber
\\
&
\qquad\qquad\qquad\qquad\qquad\,\,\cdot\Bigg(\frac{L_{1}}{2}\eta\left(1+ \Vert\mathbf{x}_{0}\Vert_{L_{2}} 
+\omega(T)\right)\phi_{k,\eta}
+\frac{\sqrt{\eta}}{2}\nu_{k,\eta}\sqrt{\psi_{k,\eta}}\Bigg),\label{eq:error-disc}
\\
&E_{2}(f,g, K, \eta, M, L_1) 
:= \sum_{k=1}^{K}
\prod_{j=k+1}^{K}\gamma_{j,\eta}\cdot e^{\int_{k\eta}^{K\eta}f(T-t)\mathrm{d}t}
\cdot\frac{M}{2}\phi_{k,\eta},\label{eq:error-disc-2}
\end{align}
where $\phi_{k,\eta}$ is given in \eqref{phi:defn}, $\psi_{k,\eta}$ is given in \eqref{psi:defn}, $\gamma_{j,\eta}$ is given in \eqref{gamma:defn}, $L(t)$ is given in \eqref{eq:Lt}, $\delta_{j}(T-t)$ is defined in \eqref{mu:definition}, $\omega(T)$ is defined in \eqref{c:2:defn} and $\nu_{k,\eta}$ is given in \eqref{h:k:eta:main} in Appendix~\ref{sec:key:quantities}. 
\end{theorem}

While the bound in 
Theorem~\ref{thm:discrete:2} looks quite complex, it can be easily interpreted as follows.

The first term in \eqref{main:thm:upper:bound}, referred to as the initialization error, characterizes the convergence of the continuous-time
probability flow ODE $(\mathbf{y}_{t})$ in \eqref{eq:zt} to the distribution $p_{0}$ without discretization
or score-matching errors. Specifically, it bounds the error $\mathcal{W}_{2}(\mathcal{L}(\mathbf{y}_{T}),p_{0})$, which is introduced due to the initialization of the probability flow ODE $(\mathbf{y}_{t})$ at $\hat p_T$ instead of $p_T$ (see Proposition~\ref{thm:1:main:paper}). One can find from the definition \eqref{c:t:defn} that $\mu(t)>0$ and hence the initialization error goes to zero when we pick $T = K \eta$ to be sufficiently large, i.e.,
$e^{-\int_{0}^{K\eta}\mu(t)\mathrm{d}t}  \cdot \Vert\mathbf{x}_{0}\Vert_{L_{2}} \rightarrow 0$, as $K \eta \rightarrow \infty$.

The second and third terms in \eqref{main:thm:upper:bound} quantify
the discretization and score-matching errors in running the algorithm $(\mathbf{u}_{k})$ in \eqref{eq:yk}. 
Note that the assumption $\eta\leq\bar{\eta}$ in Theorem~\ref{thm:discrete:2} implies
that $\gamma_{j,\eta}\in(0,1)$ in \eqref{gamma:defn}, 
which plays the role of a contraction rate of the error $\left\Vert\mathbf{y}_{k\eta}- \mathbf{u}_{k} \right\Vert_{L_{2}}$ over iterations (see Proposition~\ref{prop:iterates:main:paper}). Conceptually, it guarantees that as the number of iterations increases, 
the discretization and score-matching errors in the iterates $(\mathbf{u}_{k})$
do not propagate and grow in time. 
One can show that for fixed $T = K \eta$, the discretization error $E_{1}(f,g, K, \eta, L_1) \rightarrow 0$, when $\eta\rightarrow 0$
and the score matching error $E_{2}(f,g,K,\eta,M,L_1)$ is linear in $M$ which goes to $0$ as $M\rightarrow 0$.


By combing the above three terms, we infer that we can first choose a large $T = K \eta$, and then choose a small stepsize $\eta$ so that $\mathcal{W}_{2}(\mathcal{L}(\mathbf{u}_{K}),p_{0})$ can be made small (provided that $M$ is small). 

\vspace{1mm}

\begin{remark}
Theorem~\ref{thm:discrete:2} holds for the ODE-based sampler with exponential integrator discretization. Our analysis also goes through for the simple Euler method. For other methods such as Heun's 2nd order solver and DPM solver, the methodology used in the current paper cannot be directly applied. In particular, we need to study their discretization errors and it will require different analysis. 
\end{remark}
\vspace{1mm}

\begin{remark}[Comparison of iteration complexities]
\cite{chen2023probability} analyzed probability flow ODE with stochastic corrector steps in TV and established iteration complexity of $\tilde O(d/\epsilon^2)$ (respectively $\tilde O(\sqrt{d}/\epsilon)$) when the stochastic
corrector step is based on the overdamped (respectively underdamped) Langevin diffusion. \cite{li2023towards} analyzed a discrete time version of fully deterministic ODE and established an iteration complexity of $\tilde O(d^2/\epsilon + d^3/\sqrt{\epsilon})$ in TV. Both studies consider specific VP-SDEs as forward processes. Our paper consider fully deterministic ODE-based samplers and Wasserstein convergence guarantees. Note that TV plus strong log-concavity does not imply 2-Wasserstein convergence. For VP-SDEs, Theorem~\ref{thm:discrete:2} implies that the iteration complexity is $\tilde O(\sqrt{d}/\epsilon)$ (for the examples considered) under our assumptions; see Table~\ref{table:summary:complexity} for details.



\end{remark}
\vspace{1mm}
\begin{remark}\label{remark-conv}
Assumption~\ref{assump:p0}
plays two roles for obtaining the upper bound in Theorem~\ref{thm:discrete:2}. 

First, the $m_{0}$-strong-convexity of $-\nabla\log p_{0}$ guarantees that $\mu(t)>0$ (which appears in the first term of \eqref{main:thm:upper:bound}), which
guarantees the $2$-Wasserstein contraction and the 
convergence of the continuous-time
probability flow ODE $(\mathbf{y}_{t})$ in \eqref{eq:zt} to the distribution $p_{0}$ without discretization
or score-matching errors. 
Indeed, one can easily verify from \eqref{c:t:defn} that
$\mu(t)\rightarrow 0$ as $m_{0}\rightarrow 0$, which indicates
that strong-convexity of $-\nabla\log p_{0}$ is necessary; otherwise the ODE \eqref{eq:zt} will not have a contraction. \cite{chen2023probability} addressed this issue by modifying the ODE sampler and adding stochastic correcter steps via Langevin dynamics to establish TV convergence under weaker assumptions on data distributions. The samplers they study hence are not fully deterministic as in standard probability flow ODEs. By contrast, we analyzed the fully deterministic ODE sampler and established Wasserstein convergence.


Second, the $m_{0}$-strong-convexity of $-\nabla\log p_{0}$, together with the $L_{0}$-Lipschitzness
of $\nabla\log p_{0}$, guarantees that the discretization and
score-matching error at each iterate $k$ can be explicitly controlled as in Theorem~\ref{thm:discrete:2}.  
In particular, Assumption~\ref{assump:p0} guarantees that $\gamma_{j,\eta}\in(0,1)$
when the stepsize $\eta\leq\bar{\eta}$, which controls the propagation of
the discretization and score-matching errors as the number of iterates grows. 
The $m_{0}$-strong-convexity of $-\nabla\log p_{0}$ is necessary 
since one can easily verify that $\bar{\eta}\rightarrow 0$ and $\gamma_{j,\eta}\notin(0,1)$ as $m_{0}\rightarrow 0$.
\end{remark}



\subsection{Examples}\label{sec:examples}  

In this section, we consider several examples of the forward processes and discuss the implications of Theorem~\ref{thm:discrete:2}. In particular, we consider a variety of choices for $f$ and $g$ in the forward SDE \eqref{OU:SDE}, and investigate the iteration complexity, i.e., the number of iterates $K$ that is needed
in order to achieve $\epsilon$ accuracy, i.e. $\mathcal{W}_{2}(\mathcal{L}(\mathbf{u}_{K}),p_{0})\leq\epsilon$. While the bound in Theorem~\ref{thm:discrete:2} is quite sophisticated in general, it can be made more explicit when we consider special $f$ and $g$.


\begin{table*}[tb]
\caption{\label{table:summary:complexity} 
Summary of the iteration complexity of the algorithm \eqref{eq:yk} for various examples in terms
of $\epsilon$ and dimension $d$. Here $f,g$ correspond to the drift and diffusion terms in the forward SDE \eqref{OU:SDE}, and $a$, $b$, $\rho$ are positive constants. $K$ is the number of iterates, $M$
is the score-matching approximation error, and $\eta$ is the stepsize required to achieve accuracy level $\epsilon$ (i.e. $\mathcal{W}_{2}(\mathcal{L}(\mathbf{u}_{K}),p_{0})\leq\epsilon$).}
\begin{center}
\begin{tabular}{|c|c|c|c|c|c|}
\hline
$f$ & $g$ & $K$ & $M$ & $\eta$ \\
\hline
\hline
0   & $ae^{bt}$  & $\mathcal{O}\left(\frac{d^{3/2}\log(\frac{d}{\epsilon})}{\epsilon^{3}}\right)$ & $\mathcal{O}\left(\frac{\epsilon^{2}}{\sqrt{d}}\right)$ & $\mathcal{O}\left(\frac{\epsilon^{3}}{d^{3/2}}\right)$  \\
\hline
0   & $(b+at)^{c}$  & $\mathcal{O}\left(\frac{d^{\frac{1}{(2c+1)}+\frac{3}{2}}}{\epsilon^{\frac{2}{2c+1}+3}}\right)$ & $\mathcal{O}\left(\frac{\epsilon^{2}}{\sqrt{d}}\right)$ & $\mathcal{O}\left(\frac{\epsilon^{3}}{d^{\frac{3}{2}}}\right)$  \\
\hline
\hline
$\frac{b}{2}$   & $\sqrt{b}$  & $\mathcal{O}\left(\frac{\sqrt{d}}{\epsilon}(\log(\frac{d}{\epsilon}))^{2}\right)$ & $\mathcal{O}\left(\frac{\epsilon}{\log(\sqrt{d}/\epsilon)}\right)$ & $\mathcal{O}\left(\frac{\epsilon}{\sqrt{d}\log(\sqrt{d}/\epsilon)}\right)$ \\
\hline
$\frac{b+at}{2}$   & $\sqrt{b+at}$  & $\mathcal{O}\left(\frac{\sqrt{d}}{\epsilon}(\log(\frac{d}{\epsilon}))^{\frac{3}{2}}\right)$ & $\mathcal{O}\left(\frac{\epsilon}{\log(\sqrt{d}/\epsilon)}\right)$ & $\mathcal{O}\left(\frac{\epsilon}{\sqrt{d}\log(\sqrt{d}/\epsilon)}\right)$  \\
\hline
$\frac{(b+at)^{\rho}}{2}$   & $(b+at)^{\frac{\rho}{2}}$  & $\mathcal{O}\left(\frac{\sqrt{d}}{\epsilon}(\log(\frac{d}{\epsilon}))^{\frac{\rho+2}{\rho+1}}\right)$ & $\mathcal{O}\left(\frac{\epsilon}{\log(\sqrt{d}/\epsilon)}\right)$ & $\mathcal{O}\left(\frac{\epsilon}{\sqrt{d}\log(\sqrt{d}/\epsilon)}\right)$ \\
\hline
\end{tabular}
\end{center}
\end{table*}

\textbf{Example 1.} We first consider a VE-SDE example from \cite{SongICLR2021}.
Let $f(t)\equiv 0$ and $g(t)=ae^{bt}$ for some $a,b>0$, 
we can obtain the following corollary from Theorem~\ref{thm:discrete:2}.

\begin{corollary}\label{cor:VE:5:main:paper}
Let $f(t)\equiv 0$ and $g(t)=ae^{bt}$ for some $a,b>0$.
Then, we have $\mathcal{W}_{2}(\mathcal{L}(\mathbf{u}_{K}),p_{0})\leq\mathcal{O}(\epsilon)$
after $K=\mathcal{O}\left(\frac{d^{3/2}\log(d/\epsilon)}{\epsilon^{3}}\right)$ iterations
provided that $M\leq\frac{\epsilon^{2}}{\sqrt{d}}$ and
$\eta\leq\frac{\epsilon^{3}}{d^{3/2}}$.
\end{corollary}

\textbf{Example 2.}  Next, we consider another VE-SDE example from \cite{gao2023wasserstein}, where $f\equiv 0$ and $g$ has polynomial growth in time.
This example is inspired by \cite{Karras2022}
that considers $f(t)\equiv 0, g(t)=\sqrt{2t}$, where the discretization time steps are defined according to a polynomial noise schedule. 

\begin{corollary}\label{cor:VE:8:main:paper}
Let $f(t)\equiv 0$ and $g(t)=(b+at)^{c}$ for some $a,b>0$, $c\geq 1/2$.
Then, we have $\mathcal{W}_{2}(\mathcal{L}(\mathbf{u}_{K}),p_{0})\leq\mathcal{O}(\epsilon)$
after $K=\mathcal{O}\left(\frac{d^{\frac{1}{(2c+1)}+\frac{3}{2}}}{\epsilon^{\frac{2}{2c+1}+3}}\right)$ iterations
provided that $M\leq\frac{\epsilon^{2}}{\sqrt{d}}$ and $\eta\leq\frac{\epsilon^{3}}{d^{\frac{3}{2}}}$.
\end{corollary}





\textbf{Example 3.} Next, we consider
a VP-SDE example, with constant $f,g$. 
This includes the special case $f\equiv 1$, $g\equiv\sqrt{2}$
that is considered in \cite{chen2023probability}. 
In particular, we consider  
$f\equiv\frac{b}{2}$, $g\equiv\sqrt{b}$ for some $b>0$. 
We obtain the following corollary from Theorem~\ref{thm:discrete:2}.

\begin{corollary}\label{cor:VP:const:main:paper}
Assume $f\equiv\frac{b}{2}$, $g\equiv\sqrt{b}$ for some $b>0$. 
Then, we have $\mathcal{W}_{2}(\mathcal{L}(\mathbf{u}_{K}),p_{0})\leq\mathcal{O}(\epsilon)$
after
$K=\mathcal{O}\left(\frac{\sqrt{d}}{\epsilon}(\log(\frac{d}{\epsilon}))^{2}\right)$
iterations provided that $M\leq\frac{\epsilon}{\log(\sqrt{d}/\epsilon)}$ and $\eta\leq\frac{\epsilon}{\sqrt{d}\log(\sqrt{d}/\epsilon)}$.
\end{corollary}

\textbf{Example 4.} Finally, we consider a VP-SDE example
where $f,g$ have polynomial growth. 
We consider
$f(t)=\frac{1}{2}(b+at)^{\rho}$ and $g(t)=\sqrt{(b+at)^{\rho}}$, 
where $a,b>0$,
which is proposed in \cite{gao2023wasserstein}. 
This includes the special case $f(t)=\frac{1}{2}(b+at)$ and $g(t)=\sqrt{b+at}$, 
i.e. $\rho=1$, 
that is studied in \cite{Ho2020}.
Then we can obtain the following corollary from Theorem~\ref{thm:discrete:2}. 

\begin{corollary}\label{cor:VP:6:main:paper}
Assume $f(t)=\frac{1}{2}(b+at)^{\rho}$ and $g(t)=\sqrt{(b+at)^{\rho}}$.
Then, we have $\mathcal{W}_{2}(\mathcal{L}(\mathbf{u}_{K}),p_{0})\leq\mathcal{O}(\epsilon)$
after
$K=\mathcal{O}\left(\frac{\sqrt{d}}{\epsilon}(\log(\frac{d}{\epsilon}))^{\frac{\rho+2}{\rho+1}}\right)$
iterations provided that $M\leq\frac{\epsilon}{\log(\sqrt{d}/\epsilon)}$ and $\eta\leq\frac{\epsilon}{\sqrt{d}\log(\sqrt{d}/\epsilon)}$.
\end{corollary}


We observe from Corollary~\ref{cor:VP:6:main:paper} that the 
complexity $K$ deceases as $\rho$ increases. 
However, this does not suggest that the optimal complexity
is achieved when $\rho\rightarrow\infty$ since
our complexity only keeps track the dependence on $d$ and $\epsilon$, 
and ignores any pre-factor that can depend on $\rho$ which might
go to infinity as $\rho\rightarrow\infty$. Indeed, 
it follows from $\eta\leq\bar{\eta}$ in Theorem~\ref{thm:discrete:2} that
$\eta\leq\frac{\log(2)}{\max_{0\leq t\leq T}f(t)}=\frac{2\log(2)}{(b+aT)^{\rho}}\rightarrow 0$
as $\rho\rightarrow\infty$ so that the complexity will explode
as $\rho\rightarrow\infty$.


In Table~\ref{table:summary:complexity}, we summarize the results about the iteration complexity for the examples discussed in Section~\ref{sec:examples}. 
An immediate observation from Table~\ref{table:summary:complexity} is that the iteration complexity depends on $f,g$ and the complexity of VE-SDEs is worse than that of VP-SDEs, at least for the examples we analyze. This theoretical finding is generally consistent with \cite{SongICLR2021}, where they observed empirically that the performance of probability flow ODE samplers (with Euler or Runge-Kutta solvers) depends on the choice of forward SDEs and the sample quality for VE-SDEs is much worse than VP-SDEs for high-dimensional data. 
Another observation from Table~\ref{table:summary:complexity} is that the best iteration complexity is of order $\widetilde{\mathcal{O}}(\sqrt{d}/\epsilon)$ from the examples we studied. One natural question is whether there are other choices of $f, g$ so that the iteration complexity becomes better than $\widetilde{\mathcal{O}}(\sqrt{d}/\epsilon)$. We next show that the answer to this question is negative, if we use the result in Theorem~\ref{thm:discrete:2}. 

\begin{proposition}\label{prop:lower:bound}
Under the assumptions in Theorem~\ref{thm:discrete:2},
we further assume
$\min_{t\geq 0}(g(t))^{2}L(t)>0$
and $\max_{0\leq s\leq t}\mu(s)\leq c_{1}\left(\int_{0}^{t}\mu(s)\mathrm{d}s\right)^{\rho}+c_{2}$
uniformly in $t$ for some $c_{1},c_{2},\rho>0$, where $\mu(s)$ is defined in \eqref{c:t:defn}.
We also assume that $\liminf_{T\rightarrow\infty}\int_{0}^{T}e^{-2\int_{s}^{T}f(v)\mathrm{d}v}(g(s))^{2}\mathrm{d}s>0$.
If we use the upper bound \eqref{main:thm:upper:bound}, 
then in order to achieve $\epsilon$ accuracy, i.e. $\mathcal{W}_{2}(\mathcal{L}(\mathbf{u}_{K}),p_{0})\leq\epsilon$, 
we must have
$K=\widetilde{\Omega}\left(\sqrt{d}/\epsilon\right)$, 
where $\widetilde{\Omega}$ ignores the logarithmic dependence on $\epsilon$ and $d$.
\end{proposition}

The assumptions in Proposition~\ref{prop:lower:bound} are mild and one can readily check that they are satisfied for all the examples in Table~\ref{table:summary:complexity}  that achieve the iteration complexity $\widetilde{O}\left(\frac{\sqrt{d}}{\epsilon}\right)$. 
If we ignore the dependence on the logarithmic factors
of $d$ and $\epsilon$, we can see from Table~\ref{table:summary:complexity}
that all the VP-SDE examples achieve the lower bound
in Proposition~\ref{prop:lower:bound}.

\section{Outline of the Proof of Theorem~\ref{thm:discrete:2}}

We provide an outline for the proof of our main result (Theorem~\ref{thm:discrete:2}).
At a high level, we analyze three sources of errors: (1) the initialization of the algorithm at $\hat p_T$ instead of $p_T$, (2) the estimation error of the score function, and (3) the discretization error of the continuous-time ODE \eqref{eq:u}. 

 First, we study the error introduced due to the initialization at $\hat p_T$ instead of $p_T$. 
Recall the probability flow ODE $\mathbf{y}_{t}$ given in \eqref{eq:zt}, 
which has the same dynamics as $\tilde{\mathbf{x}}_{t}$ (defined in \eqref{eq:ODEReverse})
but with a different prior distribution $\mathbf{y}_{0}\sim\hat{p}_{T}$ (in contrast to $\tilde{\mathbf{x}}_{0}\sim p_{T}$). 
    The following result bounds $\mathcal{W}_{2}(\mathcal{L}(\mathbf{y}_{T}),p_{0}).$

\begin{proposition}\label{thm:1:main:paper}
Assume $p_{0}$ is $m_{0}$-strongly-log-concave.  
Then, we have
\begin{align}\label{eq:contraction1:main:paper}
\mathcal{W}_{2}(\mathcal{L}(\mathbf{y}_{T}),p_{0})
\leq e^{-\int_{0}^{T}\mu(t)\mathrm{d}t}\Vert\mathbf{x}_{0}\Vert_{L_{2}},
\end{align}
where $\mu(t)$ is given in \eqref{c:t:defn} in Appendix~\ref{sec:key:quantities}. 
\end{proposition}

The main challenge in analyzing the ODE $\mathbf{y}_t$ lies in studying the term $\nabla\log p_{T-t}(\mathbf{y}_{t})$. In general, this term is neither linear in $\mathbf{y}_t$
nor admits a closed-form expression. However, 
when $p_{0}$ is strongly log-concave, it is known that
that $\nabla_{\mathbf{x}}\log p_{T-t}(\mathbf{x})$ is also strongly concave (see e.g. \cite{gao2023wasserstein}). This fact
allows us to establish
Proposition~\ref{thm:1:main:paper} whose proof will be given in Appendix~\ref{sec:thm1}.


Now we consider the algorithm \eqref{eq:yk} with iterates $(\mathbf{u}_{k})$, and bound the errors due to 
both score estimations and discretizations. 
For any $k=0,1,2,\ldots,K$, $\mathbf{u}_{k}$
has the same distribution as $\hat{\mathbf{u}}_{k\eta}$, 
where $\hat{\mathbf{u}}_{t}$ is a continuous-time process
with the dynamics:
\begin{align}
\frac{\mathrm{d}\hat{\mathbf{u}}_{t}}{\mathrm{d}t}
=f(T-t)\hat{\mathbf{u}}_{t}
+\frac{1}{2}(g(T-t))^{2}\boldsymbol{s}_{\theta}\left(\hat{\mathbf{u}}_{\lfloor t/\eta\rfloor\eta},T-\lfloor t/\eta\rfloor\eta\right),\label{continuous:hat:y:main:paper}
\end{align}
with the initial distribution $\hat{\mathbf{u}}_{0}\sim\hat{p}_{T}$.
We have the following result 
that provides an upper bound 
for $\Vert\mathbf{y}_{k\eta}-\hat{\mathbf{u}}_{k\eta} \Vert_{L_{2}}$ in terms of $\Vert\mathbf{y}_{(k-1)\eta}-\hat{\mathbf{u}}_{(k-1)\eta} \Vert_{L_{2}}$. 

\begin{proposition}\label{prop:iterates:main:paper}
Assume that $p_{0}$ is $m_{0}$-strongly-log-concave, 
and $\nabla\log p_{0}$ is $L_{0}$-Lipschitz.
Then, for any $k=1,2,\ldots,K$,
\begin{align}
\left\Vert\mathbf{y}_{k\eta}-\hat{\mathbf{u}}_{k\eta}\right\Vert_{L_{2}}
&\leq\gamma_{k,\eta}
\cdot
e^{\int_{(k-1)\eta}^{k\eta}f(T-t)\mathrm{d}t}\Vert\mathbf{y}_{(k-1)\eta}-\hat{\mathbf{u}}_{(k-1)\eta}\Vert_{L_{2}}
\nonumber
\\
&\qquad
+\frac{L_{1}}{2}\eta\left(1+\Vert\mathbf{x}_{0}\Vert_{L_{2}}
+\omega(T)\right)\phi_{k,\eta}
+\frac{M}{2}\phi_{k,\eta}
+\frac{\sqrt{\eta}}{2}\nu_{k,\eta}\sqrt{\psi_{k,\eta}},\label{L:2:iterates:main:paper}
\end{align} 
where $\gamma_{k,\eta}$, $\phi_{k,\eta}$ and $\psi_{k,\eta}$ are defined in \eqref{gamma:defn}, \eqref{phi:defn} and \eqref{psi:defn} respectively, 
and $\omega(T)$ is defined in \eqref{c:2:defn}
and $\nu_{k,\eta}$ is given in \eqref{h:k:eta:main} in Appendix~\ref{sec:key:quantities}.
\end{proposition}

Proposition~\ref{prop:iterates:main:paper}
provides the guarantees on how
the errors due to both score estimations and discretizations
propagate as the number of iterates $k$ increases. 
By iterating over $k=1,2,\ldots,K$, we immediately get:
\begin{align}
\left\Vert\mathbf{y}_{K\eta}-\hat{\mathbf{u}}_{K\eta}\right\Vert_{L_{2}} 
\le E_{1}(f,g, K, \eta, L_1)+E_{2}(f,g, K, \eta, M, L_1),
\label{after:iterates}
\end{align} 
where $E_{1}(f,g, K, \eta, L_1)$ and $E_{2}(f,g, K, \eta, M, L_1)$ are the discretization and score matching errors given in \eqref{eq:error-disc}-\eqref{eq:error-disc-2}.
Since $\hat{\mathbf{u}}_{k\eta}$ has the same distribution as $\mathbf{u}_{k}$, 
we have
\begin{equation}\label{W:2:less:than:L:2}
\mathcal{W}_{2}(\mathcal{L}(\mathbf{y}_{K\eta}),\mathcal{L}(\mathbf{u}_{K}))
\leq\left\Vert\mathbf{y}_{K\eta}-\hat{\mathbf{u}}_{K\eta}\right\Vert_{L_{2}}.    
\end{equation}

Finally, by the triangle inequality
for $2$-Wasserstein distance, we can decompose the $2$-Wasserstein error
in terms of the $2$-Wasserstein error due to the initialization of the algorithm
at $\hat{p}_{T}$ instead of $p_{T}$ and the $2$-Wasserstein error due to both score estimations and discretizations,
we obtain:
\begin{align}
\mathcal{W}_{2}(\mathcal{L}(\mathbf{u}_{K}),p_{0})
\leq
\mathcal{W}_{2}(\mathcal{L}(\mathbf{u}_{K}),\mathcal{L}(\mathbf{y}_{K\eta}))
+\mathcal{W}_{2}(\mathcal{L}(\mathbf{y}_{K\eta}),p_{0}),\label{ineq:triangle}
\end{align}
where we used $T=K\eta$.
Hence, Theorem~\ref{thm:discrete:2} follows 
by applying \eqref{eq:contraction1:main:paper}, \eqref{after:iterates}, \eqref{W:2:less:than:L:2} and \eqref{ineq:triangle}.

\vspace{1mm}

\begin{remark}
Lemma~4 in \cite{chen2023probability} provides a single-step discretization analysis of the probability flow ODE in Wasserstein distance, but the (local) bound it provides is crude because it applies Gronwall's inequality, which leads to an exponential growing factor when applied recursively. In contrast, our proof leverages the strong log-concavity of the data distribution and carefully analyzes the discretization error globally. Hence, our Wasserstein analysis is novel and differs from the analysis in \cite{chen2023probability}. 
\end{remark}

\section{Conclusion}\label{sec:conclusion}

This paper provides the first non-asymptotic convergence analysis for a general class of probability flow ODE samplers in 2-Wasserstein distance, assuming accurate score estimates and a smooth log-concave data distribution. Our analysis provides some insights about the iteration complexity of deterministic ODE-based samplers for different choices of forward SDEs in diffusion models. 


Our work serves as a first step to better understand the convergence of deterministic ODE-samplers in Wasserstein distance. It is a significant open question how to relax our current assumption of the strong-log-concave data distribution. Our proof techniques borrow the idea of synchronous coupling studied in the context of sampling from unnormalized densities using Langevin algorithms \citep{DK2017}. To obtain Wasserstein convergence rates for using Langevin algorithm to sample from non-log-concave distributions, one may use more sophisticated coupling methods such as reflection coupling to obtain contraction rates of (Langevin) SDEs, see e.g. \cite{eberle2016reflection}. However,
the probability flow ODE is an ODE, not an SDE, and it is not clear whether one can find an analogue of reflection coupling in the context of probability flow ODEs. Similarly, functional inequalities, another major approach to obtain convergence (typically in KL divergence) bounds for sampling with Langevin algorithms without strong-log-concavity (see e.g. Theorem 5.2.1 in \cite{Bakry2014}), are also not directly applicable to probability flow ODEs. Hence, one might need significantly different
techniques to obtain Wasserstein convergence rates for ODE-based samplers without the log-concavity assumption. 

In addition, while 
our work focuses on the sampling phase of diffusion models, another significant open problem is to investigate the training phase, i.e., understand when the score function can be accurately learned, and combine the results with the analysis of sampling to establish end-to-end guarantees for diffusion models (see e.g. \cite{chen2023score}). We leave these investigations to the future.

\subsubsection*{Acknowledgements}

Xuefeng Gao acknowledges support from the Hong Kong Research Grants Council [GRF 14201424, 14212522, 14200123].
Lingjiong Zhu is partially supported by the grants NSF DMS-2053454, NSF DMS-2208303.

\bibliography{generative}

\newpage

\appendix


\begin{center}
\Large \bf Convergence Analysis for General Probability Flow ODEs of Diffusion Models in Wasserstein Distances \vspace{3pt}\\ {\normalsize APPENDIX}
\end{center}

The Appendix is organized as follows:
\begin{itemize}
    \item In Appendix~\ref{sec:key:quantities}, we summarize the notations given in Table~\ref{table:quantities} in the main paper, that are used in presenting the main results.
    \item In Appendix~\ref{sec:proof:main}, we provide the proofs of the main results of the paper.
    \item We present some additional technical proofs in Appendix~\ref{sec:additional:proofs}.
    \item We provide the derivation of results for various examples in Section~\ref{sec:examples} {\color{black}in the main paper} and additional details in Appendix~\ref{appendix:examples}.
\end{itemize}


\section{Key Quantities}\label{sec:key:quantities}

{\color{black}In this section, we first summarize in the following Table~\ref{table:quantities} 
the the key quantities
that play a major role
in presenting the main results in the main paper.}

\begin{table*}[htb]
\caption{\label{table:quantities} 
Summary of quantities, their interpretations and the sources}
\begin{center}
\begin{tabular}{|c | c | c |} 
 \hline
Quantities & Interpretations & Sources/References \\ [0.5ex] 
 \hline
  \hline 
 $\bar{\eta}$ in Theorem~\ref{thm:discrete:2}  & Upper bound for the stepsize &  \eqref{bar:eta} \\ 
 \hline 
 $\mu(t)$ in \eqref{c:t:defn}  & Contraction rate of $\mathcal{W}_{2}(\mathcal{L}(\mathbf{y}_{T}),p_{0}) $ &  \eqref{eq:contraction1} \\ 
 \hline
$L(t) $ in \eqref{eq:Lt} & Lipschitz constant of $\nabla_{\mathbf{x}}\log p_{t}(\mathbf{x})$  & Lemma~\ref{lem:smooth} \\
 \hline
 \multirow{2}{*}{$\gamma_{j,\eta}$ in \eqref{gamma:defn}}   & Contraction rate of discretization  & \multirow{2}{*}{Proposition~\ref{prop:iterates:main:paper}} \\
 & and score-matching errors in $\mathbf{u}_{j}$ & \\
  \hline
 \multirow{2}{*}{$\phi_{k,\eta}$ in \eqref{phi:defn}}   & A component in the discretization  & \multirow{2}{*}{Theorem~\ref{thm:discrete:2}} \\
 & and score-matching errors in $\mathbf{u}_{j}$ & \\
  \hline
 \multirow{2}{*}{$\psi_{k,\eta}$ in \eqref{psi:defn}}   & A component in the discretization  & \multirow{2}{*}{Theorem~\ref{thm:discrete:2}} \\
 & and score-matching errors in $\mathbf{u}_{j}$ & \\
 \hline
 \multirow{2}{*}{$\delta_{j}(T-t)$ in \eqref{mu:definition}}   & A component in the contraction rate of discretization  & \multirow{2}{*}{Proposition~\ref{prop:iterates}} \\
 & and score-matching errors in $\mathbf{u}_{j}$ & \\
 \hline
  $\omega(T)$ in \eqref{c:2:defn} & $ \sup_{0\leq t\leq T}\Vert\mathbf{x}_{t}\Vert_{L_{2}}$  &  \eqref{c:2:source} \\  
 \hline
  $\nu_{k,\eta}$ in \eqref{h:k:eta:main} & Bound for $ \sup_{(k-1)\eta\leq t\leq k\eta} \left\Vert\mathbf{y}_{t}-\mathbf{y}_{(k-1)\eta}\right\Vert_{L_{2}}$  &  Lemma~\ref{lem:second:term} \\  
 \hline
\end{tabular}
\end{center}
\end{table*}

{\color{black}
Next, we provide the definitions for the key quantities in Table~\ref{table:quantities}}.

{\color{black}For any $k=1,2,\ldots,K$, we define:
\begin{align}
&\phi_{k,\eta}:=\int_{(k-1)\eta}^{k\eta}e^{\int_{t}^{k\eta}f(T-s)\mathrm{d}s}(g(T-t))^{2}\mathrm{d}t,\label{phi:defn}
\\
&\psi_{k,\eta}:=\int_{(k-1)\eta}^{k\eta}e^{2\int_{t}^{k\eta}f(T-s)\mathrm{d}s}
\cdot (g(T-t))^{4}(L(T-t))^{2}\mathrm{d}t,\label{psi:defn}
\end{align}
and for any $j=1,2,\ldots,K$, we also define:
\begin{align}
\gamma_{j,\eta}:=1-\int_{(j-1)\eta}^{j\eta}\delta_{j}(T-t)\mathrm{d}t
+\frac{L_{1}\eta}{2}\int_{(j-1)\eta}^{j\eta}(g(T-t))^{2}\mathrm{d}t,\label{gamma:defn}
\end{align}}

For any $0\leq t\leq T$, we define: 
\begin{align}
\mu(t) & :=\frac{m_{0}(g(t))^{2}}{2\left(e^{-2\int_{0}^{t}f(s)\mathrm{d}s}+m_{0}\int_{0}^{t}e^{-2\int_{s}^{t}f(v)\mathrm{d}v}(g(s))^{2}\mathrm{d}s\right)}, \label{c:t:defn} 
\\
m(t)&:=\frac{(g(t))^{2}}{\frac{1}{m_{0}}e^{-2\int_{0}^{t}f(s)\mathrm{d}s}+\int_{0}^{t}e^{-2\int_{s}^{t}f(v)\mathrm{d}v}(g(s))^{2}\mathrm{d}s}-2f(t), \label{eq:mt} 
\\
L(t)&:=\min\left\{\left(\int_{0}^{t}e^{-2\int_{s}^{t}f(v)\mathrm{d}v}(g(s))^{2}\mathrm{d}s\right)^{-1},
\left(e^{\int_{0}^{t}f(s)\mathrm{d}s}\right)^{2}L_{0}\right\},\label{eq:Lt}
\end{align}

We also define:
\begin{align}
&\bar{\eta}:=\min\left\{\bar{\eta}_{1},\bar{\eta}_{2}\right\},
\label{bar:eta}
\\
&\bar{\eta}_{1}:=
\min\left\{\frac{\log(2)}{\max_{0\leq t\leq T}f(t)},
\min_{0\leq t\leq T}\left\{\frac{\frac{\frac{1}{4}(g(t))^{2}}{\frac{1}{m_{0}}e^{-2\int_{0}^{t}f(s)\mathrm{d}s}+\int_{0}^{t}e^{-2\int_{s}^{t}f(v)\mathrm{d}v}(g(s))^{2}\mathrm{d}s}}
{\frac{1}{4}(g(t))^{4}(L(t))^{2}+\frac{L_{1}}{2}(g(t))^{2}}\right\}\right\},
\label{bar:eta:1}
\\
&\bar{\eta}_{2}:=\min_{0\leq t\leq T}
\left\{\frac{\frac{e^{-2\int_{0}^{t}f(s)\mathrm{d}s}}{m_{0}}+\int_{0}^{t}e^{-2\int_{s}^{t}f(v)\mathrm{d}v}(g(s))^{2}\mathrm{d}s}{\frac{1}{2}(g(t))^{2}}\right\}.
\label{bar:eta:2}
\end{align}

For any $k=1,2,\ldots, K$ and $(k-1)\eta\leq t\leq k\eta$, we define:
\begin{align}
\delta_{k}(T-t):=\frac{\frac{1}{2}e^{-\int_{(k-1)\eta}^{t}f(T-s)\mathrm{d}s}(g(T-t))^{2}}{\frac{1}{m_{0}}e^{-2\int_{0}^{T-t}f(s)\mathrm{d}s}+\int_{0}^{T-t}e^{-2\int_{s}^{T-t}f(v)\mathrm{d}v}(g(s))^{2}\mathrm{d}s}
-\frac{\eta}{4}(g(T-t))^{4}(L(T-t))^{2},\label{mu:definition} 
\end{align}
and finally, let us define:
\begin{align}
\theta(T) &:=\sup_{0\leq t\leq T}e^{-\frac{1}{2}\int_{0}^{t}m(T-s)\mathrm{d}s}e^{-\int_{0}^{T}f(s)\mathrm{d}s}\Vert\mathbf{x}_{0}\Vert_{L_{2}},\label{c:1:defn}
\\
\omega(T) & :=\sup_{0\leq t\leq T}\left(e^{-2\int_{0}^{t}f(s)\mathrm{d}s} \Vert\mathbf{x}_{0}\Vert_{L_{2}}^2 
+d\int_{0}^{t}e^{-2\int_{s}^{t}f(v)\mathrm{d}v}(g(s))^{2}\mathrm{d}s\right)^{1/2}, \label{c:2:defn} 
\end{align}
and for any $k=1,2,\ldots,K$,
\begin{align}
\nu_{k,\eta}&:=\left(\theta(T)+\omega(T)\right)\int_{(k-1)\eta}^{k\eta}\left[f(T-s)+\frac{1}{2}(g(T-s))^{2}L(T-s)\right]\mathrm{d}s\nonumber
\\
&\qquad\qquad
+\left(L_{1}T+\Vert\nabla\log p_{0}(\mathbf{0})\Vert\right)\int_{(k-1)\eta}^{k\eta}\frac{1}{2}(g(T-s))^{2}\mathrm{d}s.\label{h:k:eta:main}
\end{align}

\section{Proofs of the Main Results}\label{sec:proof:main}

\subsection{Proof of Theorem~\ref{thm:discrete:2}}

To prove Theorem~\ref{thm:discrete:2}, we study the three sources of errors discussed in Section~\ref{sec:prelim} for convergence analysis: (1) the initialization of the algorithm at $\hat p_T$ instead of $p_T$, (2) the estimation error of the score function, and (3) the discretization error of the continuous-time process \eqref{eq:u}.

First, we study the error introduced due to the initialization at $\hat p_T$ instead of $p_T$. 
Recall the probability flow ODE $\mathbf{y}_{t}$ given in \eqref{eq:zt}:
\begin{equation}
d\mathbf{y}_{t}=\left[f(T-t)\mathbf{y}_{t}+\frac{1}{2}(g(T-t))^{2}\nabla\log p_{T-t}(\mathbf{y}_{t})\right]\mathrm{d}t, \quad \mathbf{y}_{0}\sim\hat{p}_{T}.
\end{equation}
 As discussed in Section~\ref{sec:prelim},  the distribution
of $\mathbf{y}_{T}$ differs from $p_{0}$, because $\mathbf{y}_{0}\sim\hat{p}_{T} \ne p_T$. 
The following result provides a bound on $\mathcal{W}_{2}(\mathcal{L}(\mathbf{y}_{T}),p_{0}).$

\begin{proposition}[Restatement of Proposition~\ref{thm:1:main:paper}]\label{thm:1}
Assume that $p_{0}$ is $m_{0}$-strongly-log-concave.  
Then, we have
\begin{align}\label{eq:contraction1}
\mathcal{W}_{2}(\mathcal{L}(\mathbf{y}_{T}),p_{0})
\leq e^{-\int_{0}^{T}\mu(t)\mathrm{d}t}\Vert\mathbf{x}_{0}\Vert_{L_{2}},
\end{align}
where $\mu(t)$ is given in \eqref{c:t:defn}. 
\end{proposition}

Notice that the term $\Vert\mathbf{x}_{0}\Vert_{L_{2}}$ in Proposition~\ref{thm:1} 
is finite since Assumption~\ref{assump:p0} implies that
$\mathbf{x}_{0} \sim p_0$ is $L_{2}$-integrable (see e.g. Lemma~11 in \cite{distMCMC}).

The key idea of the proof of Proposition~\ref{thm:1} is to observe
that  when $p_{0}$ is strongly log-concave, the term $\nabla_{\mathbf{x}}\log p_{T-t}(\mathbf{x})$ is also strongly concave (see e.g. \cite{gao2023wasserstein}). This fact
allows us to establish
Proposition~\ref{thm:1}.  
The proof of Proposition~\ref{thm:1} will be given in Section~\ref{sec:thm1}.


Now we consider the algorithm \eqref{eq:yk} with iterates $(\mathbf{u}_{k})$, and bound the errors due to score estimations and discretizations together. 
For any $k=0,1,2,\ldots,K$, $\mathbf{u}_{k}$
has the same distribution as $\hat{\mathbf{u}}_{k\eta}$, 
where $\hat{\mathbf{u}}_{t}$ is a continuous-time process
with the dynamics:
\begin{equation}\label{continuous:hat:y}
d\hat{\mathbf{u}}_{t}=\left[f(T-t)\hat{\mathbf{u}}_{t}+\frac{1}{2}(g(T-t))^{2}\boldsymbol{s}_{\theta}\left(\hat{\mathbf{u}}_{\lfloor t/\eta\rfloor\eta},T-\lfloor t/\eta\rfloor\eta\right)\right]\mathrm{d}t,
\end{equation}
with the initial distribution $\hat{\mathbf{u}}_{0}\sim\hat{p}_{T}$.
We have the following result 
that provides an upper bound 
for $\Vert\mathbf{y}_{k\eta}-\hat{\mathbf{u}}_{k\eta} \Vert_{L_{2}}$ in terms of $\Vert\mathbf{y}_{(k-1)\eta}-\hat{\mathbf{u}}_{(k-1)\eta} \Vert_{L_{2}}$. This result plays a key role in the proof of Theorem~\ref{thm:discrete:2}.

\begin{proposition}[Restatement of Proposition~\ref{prop:iterates:main:paper}]\label{prop:iterates}
Assume that $p_{0}$ is $m_{0}$-strongly-log-concave, 
i.e. $-\log p_{0}$ is $m_{0}$-strongly convex
and $\nabla\log p_{0}$ is $L_{0}$-Lipschitz.
For any $k=1,2,\ldots,K$,
\begin{align}
\left\Vert\mathbf{y}_{k\eta}-\hat{\mathbf{u}}_{k\eta}\right\Vert_{L_{2}}
&\leq\left(1-\int_{(k-1)\eta}^{k\eta}\delta_{k}(T-t)\mathrm{d}t+\frac{L_{1}}{2}\eta\int_{(k-1)\eta}^{k\eta}(g(T-t))^{2}\mathrm{d}t\right)
\nonumber
\\
&\qquad\qquad\qquad\cdot
e^{\int_{(k-1)\eta}^{k\eta}f(T-t)\mathrm{d}t}\Vert\mathbf{y}_{(k-1)\eta}-\hat{\mathbf{u}}_{(k-1)\eta}\Vert_{L_{2}}
\nonumber
\\
&\qquad
+\frac{L_{1}}{2}\eta\left(1+\Vert\mathbf{x}_{0}\Vert_{L_{2}}
+\omega(T)\right)\int_{(k-1)\eta}^{k\eta}e^{\int_{t}^{k\eta}f(T-s)\mathrm{d}s}(g(T-t))^{2}\mathrm{d}t
\nonumber
\\
&\qquad\qquad\qquad\qquad
+\frac{M}{2}\int_{(k-1)\eta}^{k\eta}e^{\int_{t}^{k\eta}f(T-s)\mathrm{d}s}(g(T-t))^{2}\mathrm{d}t
\nonumber
\\
&\qquad\qquad
+\sqrt{\eta}\nu_{k,\eta}\left(\int_{(k-1)\eta}^{k\eta}\left[\frac{1}{2}e^{\int_{t}^{k\eta}f(T-s)\mathrm{d}s}(g(T-t))^{2}L(T-t)\right]^{2}\mathrm{d}t\right)^{1/2},\label{L:2:iterates}
\end{align} 
where $\delta_{k}(t)$, $0\leq t\leq T$, is defined in \eqref{mu:definition}, $\omega(T)$ is defined in \eqref{c:2:defn}
and $\nu_{k,\eta}$ is given in \eqref{h:k:eta:main}.
\end{proposition}

We remark that the coefficient in front of the term $\left\Vert\mathbf{y}_{(k-1)\eta}-\hat{\mathbf{u}}_{(k-1)\eta}\right\Vert_{L_{2}}$ in  \eqref{L:2:iterates} lies in between zero and one. 
Indeed, we can see that assumption $\eta\leq\bar{\eta}:=\min(\bar{\eta}_{1},\bar{\eta}_{2})$ in Theorem~\ref{thm:discrete:2} 
implies that $\eta\leq\bar{\eta}_{1}$, where $\bar{\eta}_{1}$ is defined in \eqref{bar:eta:1}
which yields that
\begin{equation}\label{assump:stepsize:1:2}
\eta\leq
\min_{0\leq t\leq T}\left\{\frac{\frac{\frac{1}{2}e^{-\eta\max_{0\leq t\leq T}f(t)}(g(t))^{2}}{\frac{1}{m_{0}}e^{-2\int_{0}^{t}f(s)\mathrm{d}s}+\int_{0}^{t}e^{-2\int_{s}^{t}f(v)\mathrm{d}v}(g(s))^{2}\mathrm{d}s}}
{\frac{1}{4}(g(t))^{4}(L(t))^{2}+\frac{L_{1}}{2}(g(t))^{2}}\right\},
\end{equation}
and assumption $\eta\leq\bar{\eta}:=\min(\bar{\eta}_{1},\bar{\eta}_{2})$ in Theorem~\ref{thm:discrete:2} 
implies that $\eta\leq\bar{\eta}_{2}$, where $\bar{\eta}_{2}$ is defined in \eqref{bar:eta:2}
which yields that
\begin{equation}\label{assump:stepsize:2:2}
\eta\leq
\min_{0\leq t\leq T}
\left\{\frac{\frac{1}{m_{0}}e^{-2\int_{0}^{t}f(s)\mathrm{d}s}+\int_{0}^{t}e^{-2\int_{s}^{t}f(v)\mathrm{d}v}(g(s))^{2}\mathrm{d}s}{\frac{1}{2}e^{-\eta\min_{0\leq t\leq T}f(t)}(g(t))^{2}}\right\},
\end{equation}
and it follows from \eqref{assump:stepsize:1:2}-\eqref{assump:stepsize:2:2}
and the definition of $\delta_{j}(t)$ in \eqref{mu:definition} that
$\delta_{j}(T-t)\geq\frac{L_{1}}{2}\eta(g(T-t))^{2}$ for every $j=1,2,\ldots,K$ and 
$(j-1)\eta\leq t\leq j\eta$
and $\eta\max_{(j-1)\eta\leq t\leq j\eta}\delta_{j}(t)<1$ 
for every $j=1,2,\ldots,K$
such that
for any $j=1,2,\ldots,K$, 
\begin{equation*}
0\leq
1-\int_{(j-1)\eta}^{j\eta}\delta_{j}(T-t)\mathrm{d}t+\frac{L_{1}}{2}\eta\int_{(j-1)\eta}^{j\eta}(g(T-t))^{2}\mathrm{d}t
\leq 1.
\end{equation*}



Now we are ready to prove Theorem~\ref{thm:discrete:2}.

\begin{proof}[Proof of Theorem~\ref{thm:discrete:2}]
Since $\hat{\mathbf{u}}_{k\eta}$ has the same distribution as $\mathbf{u}_{k}$, 
by applying \eqref{L:2:iterates} recursively, 
we have
\begin{align*}
&\mathcal{W}_{2}(\mathcal{L}(\mathbf{y}_{K\eta}),\mathcal{L}(\mathbf{u}_{K}))
\\
&\leq
\left\Vert\mathbf{y}_{K\eta}-\hat{\mathbf{u}}_{K\eta}\right\Vert_{L_{2}}
\nonumber
\\
&\leq
\sum_{k=1}^{K}
\prod_{j=k+1}^{K}\left(1-\int_{(j-1)\eta}^{j\eta}\delta_{j}(T-t)\mathrm{d}t+\frac{L_{1}}{2}\eta\int_{(j-1)\eta}^{j\eta}(g(T-t))^{2}\mathrm{d}t\right)\nonumber
\\
&\qquad
\cdot
e^{\int_{k\eta}^{K\eta}f(T-t)\mathrm{d}t}
\Bigg(\frac{L_{1}}{2}\eta\left(1+\Vert\mathbf{x}_{0}\Vert_{L_{2}}
+\omega(T)\right)\int_{(k-1)\eta}^{k\eta}e^{\int_{t}^{k\eta}f(T-s)\mathrm{d}s}(g(T-t))^{2}\mathrm{d}t
\nonumber
\\
&\qquad\qquad\qquad\qquad
+\frac{M}{2}\int_{(k-1)\eta}^{k\eta}e^{\int_{t}^{k\eta}f(T-s)\mathrm{d}s}(g(T-t))^{2}\mathrm{d}t
\nonumber
\\
&\qquad\qquad
+\sqrt{\eta}\nu_{k,\eta}\left(\int_{(k-1)\eta}^{k\eta}\left[\frac{1}{2}e^{\int_{t}^{k\eta}f(T-s)\mathrm{d}s}(g(T-t))^{2}L(T-t)\right]^{2}\mathrm{d}t\right)^{1/2}\Bigg).
\end{align*}
Moreover, we recall that $T=K\eta$ and by triangle inequality
for $2$-Wasserstein distance,
\begin{equation}\label{triangle:ineq}
\mathcal{W}_{2}(\mathcal{L}(\mathbf{u}_{K}),p_{0})
\leq
\mathcal{W}_{2}(\mathcal{L}(\mathbf{u}_{K}),\mathcal{L}(\mathbf{y}_{K\eta}))
+\mathcal{W}_{2}(\mathcal{L}(\mathbf{y}_{K\eta}),p_{0}).
\end{equation}
By applying
Proposition~\ref{thm:1} and \eqref{triangle:ineq}, we get
\begin{align}
\mathcal{W}_{2}(\mathcal{L}(\mathbf{u}_{K}),p_{0})
\leq 
e^{-\int_{0}^{K\eta}\mu(t)\mathrm{d}t}  \cdot \Vert\mathbf{x}_{0}\Vert_{L_{2}}  
+E_{1}(f,g, K, \eta, L_1)
+E_{2}(f,g, K, \eta, M, L_1),
\end{align}
where $\mu(t)$ is given in \eqref{c:t:defn} and we recall from \eqref{eq:error-disc}-\eqref{eq:error-disc-2} that
\begin{align}
& E_{1}(f,g, K, \eta, L_1) := \sum_{k=1}^{K}
\prod_{j=k+1}^{K}\gamma_{j,\eta}\cdot e^{\int_{k\eta}^{K\eta}f(T-t)\mathrm{d}t}\nonumber
\\
&\qquad\qquad\qquad\qquad\qquad\cdot\Bigg(\frac{L_{1}}{2}\eta\left(1+ \Vert\mathbf{x}_{0}\Vert_{L_{2}} 
+\omega(T)\right)\phi_{k,\eta}
+\frac{\sqrt{\eta}}{2}\nu_{k,\eta}\sqrt{\psi_{k,\eta}}\Bigg),
\\
&E_{2}(f,g, K, \eta, M, L_1) := \sum_{k=1}^{K}
\prod_{j=k+1}^{K}\gamma_{j,\eta}\cdot e^{\int_{k\eta}^{K\eta}f(T-t)\mathrm{d}t}
\cdot\frac{M}{2}\phi_{k,\eta},
\end{align}
where $\phi_{k,\eta}$ is given in \eqref{phi:defn}, $\psi_{k,\eta}$ is given in \eqref{psi:defn}, $\gamma_{j,\eta}$ is given in \eqref{gamma:defn}, $L(t)$ is given in \eqref{eq:Lt}, $\delta_{j}(T-t)$ is defined in \eqref{mu:definition}, $\omega(T)$ is defined in \eqref{c:2:defn} and $\nu_{k,\eta}$ is given in \eqref{h:k:eta:main}.
The proof is complete.
\end{proof}


\subsubsection{Proof of Proposition~\ref{thm:1}}\label{sec:thm1}

Before we proceed to the proof of Proposition~\ref{thm:1}, 
let us first introduce a technical lemma.

\begin{lemma}\label{lem:0}
It holds that:
\begin{equation}\label{eq:phatp:0} 
\mathcal{W}_{2}(p_{T},\hat{p}_{T})
\leq
e^{-\int_{0}^{T}f(s)\mathrm{d}s}\Vert\mathbf{x}_{0}\Vert_{L_{2}}.
\end{equation}    
\end{lemma}

\begin{proof}[Proof of Lemma~\ref{lem:0}]
We recall that $p_{T}$ is the distribution of $\mathbf{x}_{T}$
which has the expression (see \eqref{SDE:solution})
\begin{equation}\label{x:T:follow:1}
\mathbf{x}_{T}=e^{-\int_{0}^{T}f(s)\mathrm{d}s}\mathbf{x}_{0}+\int_{0}^{T}e^{-\int_{s}^{T}f(v)\mathrm{d}v}g(s)\mathrm{d}\mathbf{B}_{s},
\end{equation}
and $\hat{p}_{T}$ (see \eqref{eq:hatp}) is the distribution of
\begin{equation}\label{x:T:follow:2}
\hat{\mathbf{x}}_{T}=\int_{0}^{T}e^{-\int_{s}^{T}f(v)\mathrm{d}v}g(s)\mathrm{d}\mathbf{B}_{s}.
\end{equation}
Therefore, it follows from \eqref{x:T:follow:1} and \eqref{x:T:follow:2} that 
\begin{align*}
\mathcal{W}_{2}(p_{T},\hat{p}_{T})\leq 
\Vert\mathbf{x}_{T}-\hat{\mathbf{x}}_{T}\Vert_{L_{2}}
=e^{-\int_{0}^{T}f(s)\mathrm{d}s}\Vert\mathbf{x}_{0}\Vert_{L_{2}}.
\end{align*}
This completes the proof.
\end{proof}

Now, we are ready to prove Proposition~\ref{thm:1}.

\begin{proof}[Proof of Proposition~\ref{thm:1}]
We recall that
\begin{equation}\label{tilde:x:definition}
\mathrm{d}\tilde{\mathbf{x}}_{t}=\left[f(T-t)\tilde{\mathbf{x}}_{t}+\frac{1}{2}(g(T-t))^{2}\nabla\log p_{T-t}(\tilde{\mathbf{x}}_{t})\right]\mathrm{d}t,
\end{equation}
with the initial distribution $\tilde{\mathbf{x}}_{0}\sim p_{T}$ and
\begin{equation*}
\mathrm{d}\mathbf{y}_{t}=\left[f(T-t)\mathbf{y}_{t}+\frac{1}{2}(g(T-t))^{2}\nabla\log p_{T-t}(\mathbf{y}_{t})\right]\mathrm{d}t,
\end{equation*}
with the initial distribution $\mathbf{y}_{0}\sim\hat{p}_{T}$.

It is proved in \cite{gao2023wasserstein} that
$\log p_{T-t}(\mathbf{x})$
is $a(T-t)$-strongly-concave, where
\begin{equation}\label{eq:aT-t}
a(T-t):=\frac{1}{\frac{1}{m_{0}}e^{-2\int_{0}^{T-t}f(s)\mathrm{d}s}+\int_{0}^{T-t}e^{-2\int_{s}^{T-t}f(v)\mathrm{d}v}(g(s))^{2}\mathrm{d}s}.
\end{equation}

Next, let us recall from \eqref{eq:mt} the definition of $m(T-t)$:
\begin{equation}\label{eq:mT-t}
m(T-t):=(g(T-t))^{2}a(T-t)-2f(T-t),\qquad 0\leq t\leq T,
\end{equation}
where $a(T-t)$ is defined in \eqref{eq:aT-t}.
We can compute that
\begin{align*}
&\mathrm{d}\left(\Vert\tilde{\mathbf{x}}_{t}-\mathbf{y}_{t}\Vert^{2}e^{\int_{0}^{t}m(T-s)\mathrm{d}s}\right)
\nonumber
\\
&=m(T-t)e^{\int_{0}^{t}m(T-s)\mathrm{d}s}\Vert\tilde{\mathbf{x}}_{t}-\mathbf{y}_{t}\Vert^{2}\mathrm{d}t
+2e^{\int_{0}^{t}m(T-s)\mathrm{d}s}\langle\tilde{\mathbf{x}}_{t}-\mathbf{y}_{t},d\tilde{\mathbf{x}}_{t}-d\mathbf{y}_{t}\rangle 
\nonumber
\\
&=m(T-t)e^{\int_{0}^{t}m(T-s)\mathrm{d}s}\Vert\tilde{\mathbf{x}}_{t}-\mathbf{y}_{t}\Vert^{2}\mathrm{d}t
+2e^{\int_{0}^{t}m(T-s)\mathrm{d}s}\langle\tilde{\mathbf{x}}_{t}-\mathbf{y}_{t},f(T-t)(\tilde{\mathbf{x}}_{t}-\mathbf{y}_{t})\rangle \mathrm{d}t
\nonumber
\\
&\qquad
+2e^{\int_{0}^{t}m(T-s)\mathrm{d}s}\left\langle\tilde{\mathbf{x}}_{t}-\mathbf{y}_{t},\frac{1}{2}(g(T-t))^{2}\left(\nabla\log p_{T-t}(\tilde{\mathbf{x}}_{t})-\nabla\log p_{T-t}(\mathbf{y}_{t})\right)\right\rangle \mathrm{d}t
\nonumber
\\
&\leq
e^{\int_{0}^{t}m(T-s)\mathrm{d}s}\left(m(T-t)+2f(T-t)-(g(T-t))^{2}a(T-t)\right)
\Vert\tilde{\mathbf{x}}_{t}-\mathbf{y}_{t}\Vert^{2}\mathrm{d}t
\nonumber
\\
&=0.
\end{align*}
This implies that
\begin{equation}\label{L:2:constract}
\Vert\tilde{\mathbf{x}}_{t}-\mathbf{y}_{t}\Vert^{2}e^{\int_{0}^{t}m(T-s)\mathrm{d}s}
\leq
\Vert\tilde{\mathbf{x}}_{0}-\mathbf{y}_{0}\Vert^{2},
\end{equation}
so that
\begin{equation}\label{eq:int-mt}
\mathbb{E}\Vert\tilde{\mathbf{x}}_{T}-\mathbf{y}_{T}\Vert^{2}
\leq e^{-\int_{0}^{T}m(T-s)\mathrm{d}s}\mathbb{E}\Vert\tilde{\mathbf{x}}_{0}-\mathbf{y}_{0}\Vert^{2}.
\end{equation}
Consider a coupling of $(\tilde{\mathbf{x}}_{0},\mathbf{y}_{0})$ such that $\tilde{\mathbf{x}}_{0}\sim p_{T}$, $\mathbf{y}_{0}\sim\hat{p}_{T}$
and $\mathbb{E}\Vert\tilde{\mathbf{x}}_{0}-\mathbf{y}_{0}\Vert^{2}=\mathcal{W}_{2}^{2}(p_{T},\hat{p}_{T})$.

Next, we recall from Lemma~\ref{lem:0} that
\begin{equation}\label{eq:phatp} 
\mathcal{W}_{2}(p_{T},\hat{p}_{T})
\leq
e^{-\int_{0}^{T}f(s)\mathrm{d}s}\Vert\mathbf{x}_{0}\Vert_{L_{2}}.
\end{equation}

By combining \eqref{eq:int-mt} with \eqref{eq:phatp}, we conclude that
\begin{align*}
\mathcal{W}_{2}^{2}(\mathcal{L}(\mathbf{y}_{T}),p_{0})
&=\mathcal{W}_{2}(\mathcal{L}(\mathbf{y}_{T}), \mathcal{L}(\tilde{\mathbf{x}}_{T}))
\leq\mathbb{E}\Vert\tilde{\mathbf{x}}_{T}-\mathbf{y}_{T}\Vert^{2}
\nonumber
\\
&\leq e^{-\int_{0}^{T}m(T-s)\mathrm{d}s}\mathcal{W}_{2}^{2}(p_{T},\hat{p}_{T})
\nonumber
\\
&\leq e^{-\int_{0}^{T}m(s)\mathrm{d}s}e^{-2\int_{0}^{T}f(s)\mathrm{d}s}\Vert\mathbf{x}_{0}\Vert_{L_{2}}^{2}
=e^{-2\int_{0}^{T}\mu(t)\mathrm{d}t}\Vert\mathbf{x}_{0}\Vert_{L_{2}}^{2},
\end{align*}
where
\begin{equation*}
\mu(t)=f(t) + \frac{m(t)}{2} =\frac{m_{0}(g(t))^{2}}{2\left(e^{-2\int_{0}^{t}f(s)\mathrm{d}s}+m_{0}\int_{0}^{t}e^{-2\int_{s}^{t}f(v)\mathrm{d}v}(g(s))^{2}\mathrm{d}s\right)},
\end{equation*}
and we have used \eqref{eq:mT-t}. The proof is complete. 
\end{proof}

\subsubsection{Proof of Proposition~\ref{prop:iterates}}

We first state a key technical lemma, which will be used in the proof of Proposition~\ref{prop:iterates}. 

\begin{lemma}[\cite{gao2023wasserstein}]\label{lem:smooth}
Suppose that Assumption~\ref{assump:p0} holds.
Then, $\nabla_{\mathbf{x}}\log p_{T-t}(\mathbf{x})$ is $L(T-t)$-Lipschitz in $\mathbf{x}$, where  $L(T-t)$ is given in \eqref{eq:Lt}.
\end{lemma}

\begin{proof}[Proof of Proposition~\ref{prop:iterates}]
First, we recall that for any $(k-1)\eta\leq t\leq k\eta$,
\begin{align*}
&\mathbf{y}_{t}=\mathbf{y}_{(k-1)\eta}+\int_{(k-1)\eta}^{t}\left[f(T-s)\mathbf{y}_{s}+\frac{1}{2}(g(T-s))^{2}\nabla\log p_{T-s}(\mathbf{y}_{s})\right]\mathrm{d}s,
\\
&\hat{\mathbf{u}}_{t}=\hat{\mathbf{u}}_{(k-1)\eta}+\int_{(k-1)\eta}^{t}\left[f(T-s)\hat{\mathbf{u}}_{s}+\frac{1}{2}(g(T-s))^{2}\boldsymbol{s}_{\theta}\left(\hat{\mathbf{u}}_{(k-1)\eta},T-(k-1)\eta\right)\right]\mathrm{d}s,
\end{align*}
which implies that
\begin{align*}
&\mathbf{y}_{k\eta}=e^{\int_{(k-1)\eta}^{k\eta}f(T-t)\mathrm{d}t}\mathbf{y}_{(k-1)\eta}
+\frac{1}{2}\int_{(k-1)\eta}^{k\eta}e^{\int_{t}^{k\eta}f(T-s)\mathrm{d}s}(g(T-t))^{2}\nabla\log p_{T-t}(\mathbf{y}_{t})\mathrm{d}t,
\\
&\hat{\mathbf{u}}_{k\eta}=e^{\int_{(k-1)\eta}^{k\eta}f(T-t)\mathrm{d}t}\hat{\mathbf{u}}_{(k-1)\eta}
\\
&\qquad\qquad
+\frac{1}{2}\int_{(k-1)\eta}^{k\eta}e^{\int_{t}^{k\eta}f(T-s)\mathrm{d}s}(g(T-t))^{2}\boldsymbol{s}_{\theta}\left(\hat{\mathbf{u}}_{(k-1)\eta},T-(k-1)\eta\right)\mathrm{d}t.
\end{align*}

It follows that
\begin{align*}
&\mathbf{y}_{k\eta}-\hat{\mathbf{u}}_{k\eta}
\nonumber
\\
&=e^{\int_{(k-1)\eta}^{k\eta}f(T-t)\mathrm{d}t}\left(\mathbf{y}_{(k-1)\eta}-\hat{\mathbf{u}}_{(k-1)\eta}\right)
\nonumber
\\
&\qquad\qquad
+\int_{(k-1)\eta}^{k\eta}\frac{1}{2}e^{\int_{t}^{k\eta}f(T-s)\mathrm{d}s}(g(T-t))^{2}
\\
&\qquad\qquad\qquad\qquad\qquad\qquad\cdot\left(\nabla\log p_{T-t}(\mathbf{y}_{(k-1)\eta})-\nabla\log p_{T-t}\left(\hat{\mathbf{u}}_{(k-1)\eta}\right)\right)\mathrm{d}t
\nonumber
\\
&\quad
+\int_{(k-1)\eta}^{k\eta}\frac{1}{2}e^{\int_{t}^{k\eta}f(T-s)\mathrm{d}s}(g(T-t))^{2}\left(\nabla\log p_{T-t}(\mathbf{y}_{t})-\nabla\log p_{T-t}(\mathbf{y}_{(k-1)\eta})\right)\mathrm{d}t
\nonumber
\\
&\qquad
+\int_{(k-1)\eta}^{k\eta}\frac{1}{2}e^{\int_{t}^{k\eta}f(T-s)\mathrm{d}s}(g(T-t))^{2}
\\
&\qquad\qquad\qquad\qquad\qquad\qquad\cdot\left(\nabla\log p_{T-t}\left(\hat{\mathbf{u}}_{(k-1)\eta}\right)-\boldsymbol{s}_{\theta}\left(\hat{\mathbf{u}}_{(k-1)\eta},T-(k-1)\eta\right)\right)\mathrm{d}t.
\end{align*}
This implies that
\begin{align}
&\left\Vert\mathbf{y}_{k\eta}-\hat{\mathbf{u}}_{k\eta}\right\Vert_{L_{2}}
\nonumber
\\
&\leq\Bigg\Vert e^{\int_{(k-1)\eta}^{k\eta}f(T-t)\mathrm{d}t}\left(\mathbf{y}_{(k-1)\eta}-\hat{\mathbf{u}}_{(k-1)\eta}\right)
\nonumber
\\
&\qquad
+\int_{(k-1)\eta}^{k\eta}\frac{1}{2}e^{\int_{t}^{k\eta}f(T-s)\mathrm{d}s}(g(T-t))^{2}\left(\nabla\log p_{T-t}(\mathbf{y}_{(k-1)\eta})-\nabla\log p_{T-t}\left(\hat{\mathbf{u}}_{(k-1)\eta}\right)\right)\mathrm{d}t\Bigg\Vert_{L_{2}}
\nonumber
\\
&\quad
+\Bigg\Vert\int_{(k-1)\eta}^{k\eta}\frac{1}{2}e^{\int_{t}^{k\eta}f(T-s)\mathrm{d}s}(g(T-t))^{2}\left(\nabla\log p_{T-t}(\mathbf{y}_{t})-\nabla\log p_{T-t}(\mathbf{y}_{(k-1)\eta})\right)\mathrm{d}t\Bigg\Vert_{L_{2}}
\nonumber
\\
&\qquad\quad
+\Bigg\Vert\int_{(k-1)\eta}^{k\eta}\frac{1}{2}e^{\int_{t}^{k\eta}f(T-s)\mathrm{d}s}(g(T-t))^{2}
\nonumber
\\
&\qquad\qquad\qquad\qquad\qquad\qquad\cdot\left(\nabla\log p_{T-t}\left(\hat{\mathbf{u}}_{(k-1)\eta}\right)-\boldsymbol{s}_{\theta}\left(\hat{\mathbf{u}}_{(k-1)\eta},T-(k-1)\eta\right)\right)\mathrm{d}t\Bigg\Vert_{L_{2}}.\label{two:terms}
\end{align}
Next, we provide upper bounds for the three terms in \eqref{two:terms}.

\textbf{Bounding  the first term in \eqref{two:terms}.} 
We can compute that
\begin{align*}
&\Bigg\Vert e^{\int_{(k-1)\eta}^{k\eta}f(T-t)\mathrm{d}t}\left(\mathbf{y}_{(k-1)\eta}-\hat{\mathbf{u}}_{(k-1)\eta}\right)
\nonumber
\\
&\qquad
+\int_{(k-1)\eta}^{k\eta}\frac{1}{2}e^{\int_{t}^{k\eta}f(T-s)\mathrm{d}s}(g(T-t))^{2}\left(\nabla\log p_{T-t}(\mathbf{y}_{(k-1)\eta})-\nabla\log p_{T-t}\left(\hat{\mathbf{u}}_{(k-1)\eta}\right)\right)\mathrm{d}t\Bigg\Vert^{2}
\\
&=e^{2\int_{(k-1)\eta}^{k\eta}f(T-t)\mathrm{d}t}\left\Vert\mathbf{y}_{(k-1)\eta}-\hat{\mathbf{u}}_{(k-1)\eta}\right\Vert^{2}
\nonumber
\\
&\qquad
+\Bigg\Vert\int_{(k-1)\eta}^{k\eta}\frac{1}{2}e^{\int_{t}^{k\eta}f(T-s)\mathrm{d}s}(g(T-t))^{2}\left(\nabla\log p_{T-t}(\mathbf{y}_{(k-1)\eta})-\nabla\log p_{T-t}\left(\hat{\mathbf{u}}_{(k-1)\eta}\right)\right)\mathrm{d}t\Bigg\Vert^{2}
\\
&+2\int_{(k-1)\eta}^{k\eta}\Bigg\langle e^{\int_{(k-1)\eta}^{k\eta}f(T-t)\mathrm{d}t}\left(\mathbf{y}_{(k-1)\eta}-\hat{\mathbf{u}}_{(k-1)\eta}\right), 
\\
&\qquad\qquad\qquad
\frac{1}{2}e^{\int_{t}^{k\eta}f(T-s)\mathrm{d}s}(g(T-t))^{2}\left(\nabla\log p_{T-t}(\mathbf{y}_{(k-1)\eta})-\nabla\log p_{T-t}\left(\hat{\mathbf{u}}_{(k-1)\eta}\right)\right)\Bigg\rangle \mathrm{d}t.
\end{align*}
We know from \cite{gao2023wasserstein} that $\log p_{T-t}(\mathbf{x})$
is $a(T-t)$-strongly-concave, where $a(T-t)$ is given in \eqref{eq:aT-t}. Hence
we have
\begin{align*}
&\Bigg\Vert e^{\int_{(k-1)\eta}^{k\eta}f(T-t)\mathrm{d}t}\left(\mathbf{y}_{(k-1)\eta}-\hat{\mathbf{u}}_{(k-1)\eta}\right)
\nonumber
\\
&\qquad
+\int_{(k-1)\eta}^{k\eta}\frac{1}{2}e^{\int_{t}^{k\eta}f(T-s)\mathrm{d}s}(g(T-t))^{2}\left(\nabla\log p_{T-t}(\mathbf{y}_{(k-1)\eta})-\nabla\log p_{T-t}\left(\hat{\mathbf{u}}_{(k-1)\eta}\right)\right)\mathrm{d}t\Bigg\Vert^{2}
\\
&\leq\left(1-\int_{(k-1)\eta}^{k\eta}m_{k}(T-t)\mathrm{d}t\right)
e^{2\int_{(k-1)\eta}^{k\eta}f(T-t)\mathrm{d}t}
\left\Vert\mathbf{y}_{(k-1)\eta}-\hat{\mathbf{u}}_{(k-1)\eta}\right\Vert^{2}
\nonumber
\\
&\qquad\qquad
+\Bigg(\int_{(k-1)\eta}^{k\eta}\frac{1}{2}e^{\int_{t}^{k\eta}f(T-s)\mathrm{d}s}(g(T-t))^{2}L(T-t)\left\Vert\mathbf{y}_{(k-1)\eta}-\hat{\mathbf{u}}_{(k-1)\eta}\right\Vert \mathrm{d}t\Bigg)^{2}
\\
&\leq
\left(1-\int_{(k-1)\eta}^{k\eta}m_{k}(T-t)\mathrm{d}t+\frac{\eta}{2}\int_{(k-1)\eta}^{k\eta}(g(T-t))^{4}(L(T-t))^{2}\mathrm{d}t\right)
\\
&\qquad\qquad\cdot
e^{2\int_{(k-1)\eta}^{k\eta}f(T-t)\mathrm{d}t}\left\Vert\mathbf{y}_{(k-1)\eta}-\hat{\mathbf{u}}_{(k-1)\eta}\right\Vert^{2},
\end{align*}
where we applied Cauchy-Schwartz inequality and Lemma~\ref{lem:smooth}, and $m_{k}(T-t)$ is defined as:
\begin{equation}
m_{k}(T-t):=e^{-\int_{(k-1)\eta}^{t}f(T-s)\mathrm{d}s}(g(T-t))^{2}a(T-t),\qquad (k-1)\eta\leq t\leq k\eta,
\end{equation}
for every $k=1,2,\ldots,K$.
Hence, we conclude that
\begin{align}\label{first:term}
&\Bigg\Vert e^{\int_{(k-1)\eta}^{k\eta}f(T-t)\mathrm{d}t}\left(\mathbf{y}_{(k-1)\eta}-\hat{\mathbf{u}}_{(k-1)\eta}\right)
\nonumber
\\
&\quad
+\int_{(k-1)\eta}^{k\eta}\frac{1}{2}e^{\int_{t}^{k\eta}f(T-s)\mathrm{d}s}(g(T-t))^{2}\left(\nabla\log p_{T-t}\left(\mathbf{y}_{(k-1)\eta}\right)-\nabla\log p_{T-t}\left(\hat{\mathbf{u}}_{(k-1)\eta}\right)\right)\mathrm{d}t\Bigg\Vert_{L_{2}}
\nonumber \\
&\leq
\left(1-\int_{(k-1)\eta}^{k\eta}\delta_{k}(T-t)\mathrm{d}t\right)
e^{\int_{(k-1)\eta}^{k\eta}f(T-t)\mathrm{d}t}\left\Vert\mathbf{y}_{(k-1)\eta}-\hat{\mathbf{u}}_{(k-1)\eta}\right\Vert_{L_{2}},
\end{align}
where we used the inequality $\sqrt{1-x}\leq 1-\frac{x}{2}$ for any $0\leq x\leq 1$
and the definition of $\delta_{k}(T-t)$ in \eqref{mu:definition}
which can be rewritten as
\begin{equation*}
\delta_{k}(T-t):=\frac{1}{2}e^{-\int_{(k-1)\eta}^{t}f(T-s)\mathrm{d}s}(g(T-t))^{2}a(T-t)-\frac{\eta}{4}(g(T-t))^{4}(L(T-t))^{2},\quad (k-1)\eta\leq t\leq k\eta,
\end{equation*}
where $a(T-t)$ is given in \eqref{eq:aT-t}.

\textbf{Bounding  the second term in \eqref{two:terms}.} 
Using Lemma~\ref{lem:smooth}, we can compute that
\begin{align*}
&\left\Vert\int_{(k-1)\eta}^{k\eta}\frac{1}{2}e^{\int_{t}^{k\eta}f(T-s)\mathrm{d}s}(g(T-t))^{2}\left(\nabla\log p_{T-t}(\mathbf{y}_{t})-\nabla\log p_{T-t}(\mathbf{y}_{(k-1)\eta})\right)\mathrm{d}t\right\Vert^{2}
\\
&\leq
\left(\int_{(k-1)\eta}^{k\eta}\frac{1}{2}e^{\int_{t}^{k\eta}f(T-s)\mathrm{d}s}(g(T-t))^{2}L(T-t)\Vert\mathbf{y}_{t}-\mathbf{y}_{(k-1)\eta}\Vert \mathrm{d}t\right)^{2}
\\
&\leq
\eta\int_{(k-1)\eta}^{k\eta}\left[\frac{1}{2}e^{\int_{t}^{k\eta}f(T-s)\mathrm{d}s}(g(T-t))^{2}L(T-t)\right]^{2}\Vert\mathbf{y}_{t}-\mathbf{y}_{(k-1)\eta}\Vert^{2}\mathrm{d}t,
\end{align*}
which implies that 
\begin{align}\label{eq:2ndtermUB}
&\left\Vert\int_{(k-1)\eta}^{k\eta}\frac{1}{2}e^{\int_{t}^{k\eta}f(T-s)\mathrm{d}s}(g(T-t))^{2}\left(\nabla\log p_{T-t}(\mathbf{y}_{t})-\nabla\log p_{T-t}\left(\mathbf{y}_{(k-1)\eta}\right)\right)\mathrm{d}t\right\Vert_{L_{2}}
\nonumber \\
&\leq\left(\mathbb{E}\left[\eta\int_{(k-1)\eta}^{k\eta}\left[\frac{1}{2}e^{\int_{t}^{k\eta}f(T-s)\mathrm{d}s}(g(T-t))^{2}L(T-t)\right]^{2}\Vert\mathbf{y}_{t}-\mathbf{y}_{(k-1)\eta}\Vert^{2}\mathrm{d}t\right]\right)^{1/2}
\nonumber
\\
&\leq
\left(\eta\int_{(k-1)\eta}^{k\eta}\left[\frac{1}{2}e^{\int_{t}^{k\eta}f(T-s)\mathrm{d}s}(g(T-t))^{2}L(T-t)\right]^{2}\mathrm{d}t
\cdot\sup_{(k-1)\eta\leq t\leq k\eta}\mathbb{E}\Vert\mathbf{y}_{t}-\mathbf{y}_{(k-1)\eta}\Vert^{2}
\right)^{1/2}
\nonumber
\\
&=
\sqrt{\eta}\left(\int_{(k-1)\eta}^{k\eta}\left[\frac{1}{2}e^{\int_{t}^{k\eta}f(T-s)\mathrm{d}s}(g(T-t))^{2}L(T-t)\right]^{2}\mathrm{d}t\right)^{1/2}
\sup_{(k-1)\eta\leq t\leq k\eta}\left\Vert\mathbf{y}_{t}-\mathbf{y}_{(k-1)\eta}\right\Vert_{L_{2}}.
\end{align}

\textbf{Bounding the third term in \eqref{two:terms}.}  
We notice that
\begin{align*}
&\left\Vert\int_{(k-1)\eta}^{k\eta}\frac{1}{2}e^{\int_{t}^{k\eta}f(T-s)\mathrm{d}s}(g(T-t))^{2}\left(\nabla\log p_{T-t}(\hat{\mathbf{u}}_{(k-1)\eta})-\boldsymbol{s}_{\theta}\left(\hat{\mathbf{u}}_{(k-1)\eta},T-(k-1)\eta\right)\right)\mathrm{d}t\right\Vert_{L_{2}}
\\
&\leq
\left\Vert\int_{(k-1)\eta}^{k\eta}\frac{1}{2}e^{\int_{t}^{k\eta}f(T-s)\mathrm{d}s}(g(T-t))^{2}\left(\nabla\log p_{T-(k-1)\eta}(\hat{\mathbf{u}}_{(k-1)\eta})-\boldsymbol{s}_{\theta}\left(\hat{\mathbf{u}}_{(k-1)\eta},T-(k-1)\eta\right)\right)\mathrm{d}t\right\Vert_{L_{2}}
\\
&\qquad
+\left\Vert\int_{(k-1)\eta}^{k\eta}\frac{1}{2}e^{\int_{t}^{k\eta}f(T-s)\mathrm{d}s}(g(T-t))^{2}\left(\nabla\log p_{T-t}\left(\hat{\mathbf{u}}_{(k-1)\eta}\right)-\nabla\log p_{T-(k-1)\eta}\left(\hat{\mathbf{u}}_{(k-1)\eta}\right)\right)\mathrm{d}t\right\Vert_{L_{2}}.
\end{align*}
By Assumption~\ref{assump:M}, we have  
\begin{align}
&\left\Vert\int_{(k-1)\eta}^{k\eta}\frac{1}{2}e^{\int_{t}^{k\eta}f(T-s)\mathrm{d}s}(g(T-t))^{2}\left(\nabla\log p_{T-(k-1)\eta}(\hat{\mathbf{u}}_{(k-1)\eta})-\boldsymbol{s}_{\theta}\left(\hat{\mathbf{u}}_{(k-1)\eta},T-(k-1)\eta\right)\right)\mathrm{d}t\right\Vert_{L_{2}}
\nonumber
\\
&\leq\frac{M}{2}\int_{(k-1)\eta}^{k\eta}e^{\int_{t}^{k\eta}f(T-s)\mathrm{d}s}(g(T-t))^{2}\mathrm{d}t.\label{assump:can:be:Weakened}
\end{align}
Moreover, by Assumption~\ref{assump:M:1}, we have
\begin{align}
&\left\Vert\int_{(k-1)\eta}^{k\eta}\frac{1}{2}e^{\int_{t}^{k\eta}f(T-s)\mathrm{d}s}(g(T-t))^{2}\left(\nabla\log p_{T-t}\left(\hat{\mathbf{u}}_{(k-1)\eta}\right)-\nabla\log p_{T-(k-1)\eta}\left(\hat{\mathbf{u}}_{(k-1)\eta}\right)\right)\mathrm{d}t\right\Vert_{L_{2}}
\nonumber
\\
&\leq
\int_{(k-1)\eta}^{k\eta}\frac{1}{2}e^{\int_{t}^{k\eta}f(T-s)\mathrm{d}s}(g(T-t))^{2}\left\Vert\nabla\log p_{T-t}\left(\hat{\mathbf{u}}_{(k-1)\eta}\right)-\nabla\log p_{T-(k-1)\eta}\left(\hat{\mathbf{u}}_{(k-1)\eta}\right)\right\Vert_{L_{2}}\mathrm{d}t
\nonumber
\\
&\leq
\int_{(k-1)\eta}^{k\eta}\frac{1}{2}e^{\int_{t}^{k\eta}f(T-s)\mathrm{d}s}(g(T-t))^{2}L_{1}\eta\left(1+\left\Vert\hat{\mathbf{u}}_{(k-1)\eta}\right\Vert_{L_{2}}\right)\mathrm{d}t
\nonumber
\\
&\leq
\frac{L_{1}}{2}\eta\left(1+\left\Vert\mathbf{y}_{(k-1)\eta}-\hat{\mathbf{u}}_{(k-1)\eta}\right\Vert_{L_{2}}
+\Vert\mathbf{y}_{(k-1)\eta}\Vert_{L_{2}}\right)
\int_{(k-1)\eta}^{k\eta}e^{\int_{t}^{k\eta}f(T-s)\mathrm{d}s}(g(T-t))^{2}\mathrm{d}t.\label{by:apply:1}
\end{align}
Furthermore, we can compute that
\begin{equation}\label{by:apply:2}
\left\Vert\mathbf{y}_{(k-1)\eta}\right\Vert_{L_{2}}
\leq
\left\Vert\mathbf{y}_{(k-1)\eta}-\tilde{\mathbf{x}}_{(k-1)\eta}\right\Vert_{L_{2}}
+\left\Vert\tilde{\mathbf{x}}_{(k-1)\eta}\right\Vert_{L_{2}},
\end{equation}
where $\tilde{\mathbf{x}}_{t}$ is defined in \eqref{eq:ODEReverse}. 
Moreover, by \eqref{L:2:constract} in the proof of Proposition~\ref{thm:1}, we have
\begin{equation}\label{by:apply:3}
\left\Vert\mathbf{y}_{(k-1)\eta}-\tilde{\mathbf{x}}_{(k-1)\eta}\right\Vert_{L_{2}}
\leq
\left(\mathbb{E}\Vert\tilde{\mathbf{x}}_{0}-\mathbf{y}_{0}\Vert^{2}\right)^{1/2}
=e^{-\int_{0}^{T}f(s)\mathrm{d}s}\Vert\mathbf{x}_{0}\Vert_{L_{2}}
\leq
\Vert\mathbf{x}_{0}\Vert_{L_{2}},
\end{equation}
where we applied \eqref{SDE:solution} to obtain
the equality in the above equation.
Moreover, since $\tilde{\mathbf{x}}_{t}=\mathbf{x}_{T-t}$ in distribution
for any $t\in[0,T]$, we have
\begin{align}
\left\Vert\tilde{\mathbf{x}}_{(k-1)\eta}\right\Vert_{L_{2}}
=\left\Vert\mathbf{x}_{T-(k-1)\eta}\right\Vert_{L_{2}}
\leq
\sup_{0\leq t\leq T}\Vert\mathbf{x}_{t}\Vert_{L_{2}}=:\omega(T).\label{by:apply:4}
\end{align}

Next, let us show that $\omega(T)$ can be computed as given by the formula in \eqref{c:2:defn}.
By equation~\eqref{SDE:solution}, we have
\begin{align*}
\mathrm{d}\left(\Vert\mathbf{x}_{t}\Vert^{2}e^{2\int_{0}^{t}f(s)\mathrm{d}s}\right)
&=2f(t)\Vert\mathbf{x}_{t}\Vert^{2}e^{2\int_{0}^{t}f(s)\mathrm{d}s}\mathrm{d}t
\\
&\qquad\qquad
+2e^{2\int_{0}^{t}f(s)\mathrm{d}s}\langle\mathbf{x}_{t},d\mathbf{x}_{t}\rangle
+e^{2\int_{0}^{t}f(s)\mathrm{d}s}\cdot d\cdot (g(t))^{2}\mathrm{d}t.
\end{align*}
By taking expectations, we obtain
\begin{align*}
\mathrm{d}\left(\mathbb{E}\Vert\mathbf{x}_{t}\Vert^{2}e^{2\int_{0}^{t}f(s)\mathrm{d}s}\right)
=e^{2\int_{0}^{t}f(s)\mathrm{d}s}\cdot d\cdot (g(t))^{2}\mathrm{d}t,
\end{align*}
which implies that
\begin{align}\label{x:t:L:2:estimate}
\mathbb{E}\Vert\mathbf{x}_{t}\Vert^{2}
&=e^{-2\int_{0}^{t}f(s)\mathrm{d}s}\mathbb{E}\Vert\mathbf{x}_{0}\Vert^{2}
+d\int_{0}^{t}e^{-2\int_{s}^{t}f(v)\mathrm{d}v}(g(s))^{2}\mathrm{d}s.
\end{align}
Therefore, we conclude that
\begin{align}\label{c:2:source}
\omega(T)=\sup_{0\leq t\leq T}\Vert\mathbf{x}_{t}\Vert_{L_{2}}
=\sup_{0\leq t\leq T}\left(e^{-2\int_{0}^{t}f(s)\mathrm{d}s}\Vert\mathbf{x}_{0}\Vert_{L_{2}}^{2}
+d\int_{0}^{t}e^{-2\int_{s}^{t}f(v)\mathrm{d}v}(g(s))^{2}\mathrm{d}s\right)^{1/2}.
\end{align}

Therefore, by applying \eqref{by:apply:1}, \eqref{by:apply:2}, \eqref{by:apply:3} and \eqref{by:apply:4}, we have
\begin{align*}
&\left\Vert\int_{(k-1)\eta}^{k\eta}\frac{1}{2}e^{\int_{t}^{k\eta}f(T-s)\mathrm{d}s}(g(T-t))^{2}\left(\nabla\log p_{T-t}\left(\hat{\mathbf{u}}_{(k-1)\eta}\right)-\nabla\log p_{T-(k-1)\eta}\left(\hat{\mathbf{u}}_{(k-1)\eta}\right)\right)\mathrm{d}t\right\Vert_{L_{2}}
\\
&\leq  
\frac{L_{1}}{2}\eta\left\Vert\mathbf{y}_{(k-1)\eta}-\hat{\mathbf{u}}_{(k-1)\eta}\right\Vert_{L_{2}}
\int_{(k-1)\eta}^{k\eta}e^{\int_{t}^{k\eta}f(T-s)\mathrm{d}s}(g(T-t))^{2}\mathrm{d}t
\\
&\qquad\qquad\qquad
+\frac{L_{1}}{2}\eta\left(1+\Vert\mathbf{x}_{0}\Vert_{L_{2}}
+\omega(T)\right)\int_{(k-1)\eta}^{k\eta}e^{\int_{t}^{k\eta}f(T-s)\mathrm{d}s}(g(T-t))^{2}\mathrm{d}t.
\end{align*}
It follows that the third term in \eqref{two:terms} is upper bounded by
\begin{align}\label{eq:3rdterm}
&\frac{L_{1}}{2}\eta\left\Vert\mathbf{y}_{(k-1)\eta}-\hat{\mathbf{u}}_{(k-1)\eta}\right\Vert_{L_{2}}
\int_{(k-1)\eta}^{k\eta}e^{\int_{t}^{k\eta}f(T-s)\mathrm{d}s}(g(T-t))^{2}\mathrm{d}t
\nonumber \\
&\qquad
+\frac{L_{1}}{2}\eta\left(1+\Vert\mathbf{x}_{0}\Vert_{L_{2}}
+\omega(T)\right)\int_{(k-1)\eta}^{k\eta}e^{\int_{t}^{k\eta}f(T-s)\mathrm{d}s}(g(T-t))^{2}\mathrm{d}t 
\nonumber
\\
&\qquad\qquad
+ \frac{M}{2}\int_{(k-1)\eta}^{k\eta}e^{\int_{t}^{k\eta}f(T-s)\mathrm{d}s}(g(T-t))^{2}\mathrm{d}t.
\end{align}


\textbf{Bounding \eqref{two:terms}.}
On combining \eqref{first:term}, \eqref{eq:2ndtermUB} and \eqref{eq:3rdterm}, we conclude that
\begin{align} \label{eq:final2t}
\left\Vert\mathbf{y}_{k\eta}-\hat{\mathbf{u}}_{k\eta}\right\Vert_{L_{2}}^{2}
&\leq\Bigg\{\left(1-\int_{(k-1)\eta}^{k\eta}\delta_{k}(T-t)\mathrm{d}t+\frac{L_{1}}{2}\eta\int_{(k-1)\eta}^{k\eta}(g(T-t))^{2}\mathrm{d}t\right)
\nonumber
\\
&\qquad\qquad\qquad\qquad\qquad\qquad
\cdot e^{\int_{(k-1)\eta}^{k\eta}f(T-t)\mathrm{d}t}
\left\Vert\mathbf{y}_{(k-1)\eta}-\hat{\mathbf{u}}_{(k-1)\eta}\right\Vert_{L_{2}}
\nonumber
\\
&\quad
+\frac{L_{1}}{2}\eta\left(1+\Vert\mathbf{x}_{0}\Vert_{L_{2}}
+\omega(T)\right)\int_{(k-1)\eta}^{k\eta}e^{\int_{t}^{k\eta}f(T-s)\mathrm{d}s}(g(T-t))^{2}\mathrm{d}t
\nonumber
\\
&\qquad\qquad
+\frac{M}{2}\int_{(k-1)\eta}^{k\eta}e^{\int_{t}^{k\eta}f(T-s)\mathrm{d}s}(g(T-t))^{2}\mathrm{d}t
\nonumber
\\
&\qquad
+\sqrt{\eta}\left(\int_{(k-1)\eta}^{k\eta}\left[\frac{1}{2}e^{\int_{t}^{k\eta}f(T-s)\mathrm{d}s}(g(T-t))^{2}L(T-t)\right]^{2}\mathrm{d}t\right)^{1/2}
\nonumber
\\
&\qquad\qquad\qquad\qquad\qquad\qquad\qquad\qquad\cdot
\sup_{(k-1)\eta\leq t\leq k\eta}\Vert\mathbf{y}_{t}-\mathbf{y}_{(k-1)\eta}\Vert_{L_{2}}\Bigg\}^{2}.
\end{align}
We need one more result, which provides an upper bound
for $\sup_{(k-1)\eta\leq t\leq k\eta}\Vert\mathbf{y}_{t}-\mathbf{y}_{(k-1)\eta}\Vert_{L_{2}}$. The proof of Lemma~\ref{lem:second:term} is given in Appendix~\ref{app:2ndterm}.

\begin{lemma}\label{lem:second:term}
For any $k=1,2,\ldots,K$,
\begin{align*}
&\sup_{(k-1)\eta\leq t\leq k\eta}\left\Vert\mathbf{y}_{t}-\mathbf{y}_{(k-1)\eta}\right\Vert_{L_{2}}
\\
&\leq\left(\theta(T)+\omega(T)\right)\int_{(k-1)\eta}^{k\eta}\left[f(T-s)+\frac{1}{2}(g(T-s))^{2}L(T-s)\right]\mathrm{d}s
\\
&\qquad
+\left(L_{1}T+\Vert\nabla\log p_{0}(\mathbf{0})\Vert\right)\int_{(k-1)\eta}^{k\eta}\frac{1}{2}(g(T-s))^{2}\mathrm{d}s.
\end{align*}
where we recall from \eqref{c:1:defn}-\eqref{c:2:defn} that
\begin{align*}
&\theta(T):=\sup_{0\leq t\leq T}e^{-\frac{1}{2}\int_{0}^{t}m(T-s)\mathrm{d}s}e^{-\int_{0}^{T}f(s)\mathrm{d}s} \Vert\mathbf{x}_{0}\Vert_{L_{2}} ,
\\
&\omega(T):=\sup_{0\leq t\leq T}\left(e^{-2\int_{0}^{t}f(s)\mathrm{d}s} \Vert\mathbf{x}_{0}\Vert_{L_{2}}^2 
+d\int_{0}^{t}e^{-2\int_{s}^{t}f(v)\mathrm{d}v}(g(s))^{2}\mathrm{d}s\right)^{1/2}.
\end{align*}
\end{lemma}

By applying Lemma~\ref{lem:second:term}, 
we conclude from \eqref{eq:final2t} that
\begin{align*}
&\left\Vert\mathbf{y}_{k\eta}-\hat{\mathbf{u}}_{k\eta}\right\Vert_{L_{2}}^{2}
\nonumber
\\
&\leq\Bigg\{\left(1-\int_{(k-1)\eta}^{k\eta}\delta_{k}(T-t)\mathrm{d}t+\frac{L_{1}}{2}\eta\int_{(k-1)\eta}^{k\eta}(g(T-t))^{2}\mathrm{d}t\right)
\\
&\qquad\qquad\cdot
e^{\int_{(k-1)\eta}^{k\eta}f(T-t)\mathrm{d}t}\left\Vert\mathbf{y}_{(k-1)\eta}-\hat{\mathbf{u}}_{(k-1)\eta}\right\Vert_{L_{2}}
\nonumber
\\
&\qquad
+\left(\frac{L_{1}}{2}\eta\left(1+ \Vert\mathbf{x}_{0}\Vert_{L_{2}} 
+\omega(T)\right)
+\frac{M}{2}\right)\int_{(k-1)\eta}^{k\eta}e^{\int_{t}^{k\eta}f(T-s)\mathrm{d}s}(g(T-t))^{2}\mathrm{d}t
\nonumber
\\
&\qquad\qquad
+\sqrt{\eta}\nu_{k,\eta}\left(\int_{(k-1)\eta}^{k\eta}\left[\frac{1}{2}e^{\int_{t}^{k\eta}f(T-s)\mathrm{d}s}(g(T-t))^{2}L(T-t)\right]^{2}\mathrm{d}t\right)^{1/2}\Bigg\}^{2},
\end{align*}
where $\nu_{k,\eta}$ is defined in \eqref{h:k:eta:main}. The proof of Proposition~\ref{prop:iterates} is hence complete.
\end{proof}


\subsection{Proof of Proposition~\ref{prop:lower:bound}}

\begin{proof}[Proof of Proposition~\ref{prop:lower:bound}]
First of all, we have
\begin{align*}
&\text{RHS of \eqref{main:thm:upper:bound}}
\nonumber
\\
&\geq 
e^{-\int_{0}^{K\eta}\mu(t)\mathrm{d}t} \Vert\mathbf{x}_{0}\Vert_{L_{2}} 
+\sum_{k=1}^{K}
\prod_{j=k+1}^{K}\left(1-\int_{(j-1)\eta}^{j\eta}\delta_{j}(T-t)\mathrm{d}t\right)\nonumber
\\
&\quad\qquad\qquad\qquad\qquad\qquad\qquad\qquad
\cdot\sqrt{\eta}\nu_{k,\eta}\left(\int_{(k-1)\eta}^{k\eta}\left[\frac{1}{2}(g(T-t))^{2}L(T-t)\right]^{2}\mathrm{d}t\right)^{1/2}
\nonumber
\\
&\geq
e^{-\int_{0}^{K\eta}\mu(t)\mathrm{d}t} \Vert\mathbf{x}_{0}\Vert_{L_{2}} 
+\sum_{k=1}^{K}\left(1-\eta\max_{1\leq j\leq K}\max_{(j-1)\eta\leq t\leq j\eta}\delta_{j}(T-t)\right)^{K-k}
\nonumber
\\
&\qquad\qquad\qquad\qquad\qquad\qquad\qquad\qquad\qquad
\cdot\eta \nu_{k,\eta}\min_{0\leq t\leq T}\left(\frac{1}{2}(g(t))^{2}L(t)\right)
\nonumber
\\
&\geq
e^{-\int_{0}^{K\eta}\mu(t)\mathrm{d}t} \Vert\mathbf{x}_{0}\Vert_{L_{2}} 
+\sum_{k=1}^{K}\left(1-\max_{0\leq t\leq T}\mu(T-t)\right)^{K-k}
\cdot\eta \nu_{k,\eta}\min_{0\leq t\leq T}\left(\frac{1}{2}(g(t))^{2}L(t)\right)
\nonumber
\\
&=e^{-\int_{0}^{K\eta}\mu(t)\mathrm{d}t}\Vert\mathbf{x}_{0}\Vert_{L_{2}} 
+\frac{1-\left(1-\eta\max_{0\leq t\leq T}\mu(t)\right)^{K}}{\max_{0\leq t\leq T}\mu(t)}
\cdot \nu_{k,\eta}\min_{0\leq t\leq T}\left(\frac{1}{2}(g(t))^{2}L(t)\right)
\nonumber
\\
&\geq
e^{-\int_{0}^{K\eta}\mu(t)\mathrm{d}t} \Vert\mathbf{x}_{0}\Vert_{L_{2}} 
+\frac{1-e^{-K\eta\max_{0\leq t\leq T}\mu(t)}}{\max_{0\leq t\leq T}\mu(t)}
\cdot \nu_{k,\eta}\min_{0\leq t\leq T}\left(\frac{1}{2}(g(t))^{2}L(t)\right),
\end{align*}
where $\mu(t)$ is given in \eqref{c:t:defn},  $\delta_{j}(T-t)$ is defined in \eqref{mu:definition}
and the equality above is due to the formula for the finite sum of a geometric series 
and we used the inequality that $1-x\leq e^{-x}$ for any $0\leq x\leq 1$ to obtain the last inequality above.

Next, by the definition of $\nu_{k,\eta}$ in \eqref{h:k:eta:main}, we have
\begin{equation*}
\nu_{k,\eta}\geq
\eta\sqrt{d}\left(\int_{0}^{T}e^{-2\int_{s}^{T}f(v)\mathrm{d}v}(g(s))^{2}\mathrm{d}s\right)^{1/2}\min_{0\leq t\leq T}(g(t))^{2}L(t).
\end{equation*}

Therefore, in order for $\text{RHS of \eqref{main:thm:upper:bound}}\leq\epsilon$, 
we must have
\begin{equation*}
e^{-\int_{0}^{K\eta}\mu(t)\mathrm{d}t}\Vert\mathbf{x}_{0}\Vert_{L_{2}} \leq\epsilon,
\end{equation*}
which implies that $K\eta\rightarrow\infty$ as $\epsilon\rightarrow 0$ and in particular 
\begin{equation}\label{T:lower:bound}
T=K\eta=\Omega(1), 
\end{equation}
(since under our assumptions on $f$ and $g$, $\mu(t)$ which is defined in \eqref{c:t:defn} is positive and continuous so that $\int_{0}^{t}\mu(s)\mathrm{d}s$ is finite for any $t\in(0,\infty)$ and strictly increasing
from $0$ to $\infty$ as $t$ increases from $0$ to $\infty$), 
and we also need
\begin{equation}\label{eqn:implies:that}
\frac{1-e^{-K\eta\max_{0\leq t\leq T}\mu(t)}}{\max_{0\leq t\leq T}\mu(t)}
\cdot
\eta\sqrt{d}\left(\int_{0}^{T}e^{-2\int_{s}^{T}f(v)\mathrm{d}v}(g(s))^{2}\mathrm{d}s\right)^{1/2}
\cdot\frac{1}{2}\left(\min_{0\leq t\leq T}(g(t))^{2}L(t)\right)^{2}
\leq\epsilon.
\end{equation}
Note that $\max_{0\leq t\leq T}\mu(t)=\max_{0\leq t\leq K\eta}\mu(t)$.
Together with $e^{-\int_{0}^{K\eta}\mu(t)\mathrm{d}t}=\mathcal{O}(\epsilon/\sqrt{d})$ (since $\Vert\mathbf{x}_{0}\Vert_{L_{2}}\leq\mathcal{O}(\sqrt{d})$ from \eqref{eq:L2-x0})
and the assumption that $\max_{0\leq s\leq t}\mu(s)\leq c_{1}\left(\int_{0}^{t}\mu(s)\mathrm{d}s\right)^{\rho}+c_{2}$
uniformly in $t$ for some $c_{1},c_{2},\rho>0$, 
it is easy to see that $\max_{0\leq t\leq K\eta}\mu(t)=\mathcal{O}\left(\left(\log\left(\sqrt{d}/\epsilon\right)\right)^{\rho}\right)$.
Moreover, since we assumed $\min_{t\geq 0}(g(t))^{2}L(t)>0$, we have $\mu(t)>0$ for any $t$.
Together with $T=K\eta = \Omega(1)$ from \eqref{T:lower:bound}, we have $\max_{0\leq t\leq T}\mu(t)\geq\Omega(1)$.
Since $K\eta\rightarrow\infty$ as $\epsilon\rightarrow 0$, we have $1-e^{-K\eta\max_{0\leq t\leq T}\mu(t)}=\Omega(1)$.
Therefore, it follows from \eqref{eqn:implies:that} that $\eta = \widetilde{\mathcal{O}}\left(\frac{\epsilon}{\sqrt{d}}\right)$,
where $\widetilde{\mathcal{O}}$ ignores the logarithmic dependence on $\epsilon$ and $d$
and we used the assumption $\min_{t\geq 0}(g(t))^{2}L(t)>0$ and $\liminf_{T\rightarrow\infty}\int_{0}^{T}e^{-2\int_{s}^{T}f(v)\mathrm{d}v}(g(s))^{2}\mathrm{d}s>0$.
Hence, we conclude that we have the following lower bound for the complexity: $K=\widetilde{\Omega}\left(\frac{\sqrt{d}}{\epsilon}\right)$, 
where $\widetilde{\Omega}$ ignores the logarithmic dependence on $\epsilon$ and $d$.
This completes the proof.
\end{proof}

\section{Additional Technical Proofs}\label{sec:additional:proofs}

\subsection{Proof of Lemma~\ref{lem:second:term}}\label{app:2ndterm}
\begin{proof}[Proof of Lemma~\ref{lem:second:term}]
We can compute that for any $(k-1)\eta\leq t\leq k\eta$,
\begin{align*}
\mathbf{y}_{t}-\mathbf{y}_{(k-1)\eta}
=\int_{(k-1)\eta}^{t}\left[f(T-s)\mathbf{y}_{s}+\frac{1}{2}(g(T-s))^{2}\nabla\log p_{T-s}(\mathbf{y}_{s})\right]\mathrm{d}s,
\end{align*}
and moreover
\begin{align}\label{tilde:x:discretization}
\tilde{\mathbf{x}}_{t}-\tilde{\mathbf{x}}_{(k-1)\eta}
=\int_{(k-1)\eta}^{t}\left[f(T-s)\tilde{\mathbf{x}}_{s}+\frac{1}{2}(g(T-s))^{2}\nabla\log p_{T-s}(\tilde{\mathbf{x}}_{s})\right]\mathrm{d}s,
\end{align}
so that
\begin{align*}
&\mathbf{y}_{t}-\mathbf{y}_{(k-1)\eta}
\\
&=\tilde{\mathbf{x}}_{t}-\tilde{\mathbf{x}}_{(k-1)\eta}
\\
&\qquad
+\int_{(k-1)\eta}^{t}\left[f(T-s)(\mathbf{y}_{s}-\tilde{\mathbf{x}}_{s})+\frac{1}{2}(g(T-s))^{2}\left(\nabla\log p_{T-s}(\mathbf{y}_{s})-\nabla\log p_{T-s}(\tilde{\mathbf{x}}_{s})\right)\right]\mathrm{d}s,
\end{align*}
and therefore
\begin{align*}
&\left\Vert\mathbf{y}_{t}-\mathbf{y}_{(k-1)\eta}\right\Vert_{L_{2}}
\\
&\leq\left\Vert\tilde{\mathbf{x}}_{t}-\tilde{\mathbf{x}}_{(k-1)\eta}\right\Vert_{L_{2}}
+\int_{(k-1)\eta}^{t}\left[f(T-s)+\frac{1}{2}(g(T-s))^{2}L(T-s)\right]\Vert\mathbf{y}_{s}-\tilde{\mathbf{x}}_{s}\Vert_{L_{2}}\mathrm{d}s.
\end{align*}
We obtained in the proof of Proposition~\ref{thm:1} that
\begin{equation*}
\Vert\mathbf{y}_{t}-\tilde{\mathbf{x}}_{t}\Vert_{L_{2}}
\leq e^{-\frac{1}{2}\int_{0}^{t}m(T-s)\mathrm{d}s}\Vert\mathbf{y}_{0}-\tilde{\mathbf{x}}_{0}\Vert_{L_{2}},
\end{equation*}
and moreover, from the proof of Proposition~\ref{thm:1}, we have $\Vert\mathbf{y}_{0}-\tilde{\mathbf{x}}_{0}\Vert_{L_{2}}\leq e^{-\int_{0}^{T}f(s)\mathrm{d}s}\Vert\mathbf{x}_{0}\Vert_{L_{2}}$
so that
\begin{align*}
\Vert\mathbf{y}_{t}-\tilde{\mathbf{x}}_{t}\Vert_{L_{2}}
\leq e^{-\frac{1}{2}\int_{0}^{t}m(T-s)\mathrm{d}s}e^{-\int_{0}^{T}f(s)\mathrm{d}s}\Vert\mathbf{x}_{0}\Vert_{L_{2}}.
\end{align*}
Therefore, we have
\begin{align} \label{eq:c1Tsource}
&\left\Vert\mathbf{y}_{t}-\mathbf{y}_{(k-1)\eta}\right\Vert_{L_{2}}
\nonumber \\
&\leq\left\Vert\tilde{\mathbf{x}}_{t}-\tilde{\mathbf{x}}_{(k-1)\eta}\right\Vert_{L_{2}}
+\theta(T)\int_{(k-1)\eta}^{k\eta}\left[f(T-s)+\frac{1}{2}(g(T-s))^{2}L(T-s)\right]\mathrm{d}s,
\end{align}
where $\theta(T)$ bounds $ \sup_{0 \le t \le T}\Vert\mathbf{y}_{t}-\tilde{\mathbf{x}}_{t}\Vert_{L_{2}}$ and it is given in \eqref{c:1:defn}. 

Next, it follows from \eqref{tilde:x:discretization} that
\begin{align}
\left\Vert\tilde{\mathbf{x}}_{t}-\tilde{\mathbf{x}}_{(k-1)\eta}\right\Vert_{L_{2}}
\leq
\int_{(k-1)\eta}^{t}\left[f(T-s)\Vert\tilde{\mathbf{x}}_{s}\Vert_{L_{2}}
+\frac{1}{2}(g(T-s))^{2}\Vert\nabla\log p_{T-s}(\tilde{\mathbf{x}}_{s})\Vert_{L_{2}}\right]\mathrm{d}s,
\end{align}
and we have
\begin{align}
\Vert\nabla\log p_{T-s}(\tilde{\mathbf{x}}_{s})\Vert_{L_{2}}
&
\leq
\left\Vert\nabla\log p_{T-s}(\tilde{\mathbf{x}}_{s})-\nabla\log p_{T-s}(\mathbf{0})\right\Vert_{L_{2}}
+\left\Vert\nabla\log p_{T-s}(\mathbf{0})\right\Vert_{L_{2}}
\nonumber
\\
&\leq L(T-s)\left\Vert\tilde{\mathbf{x}}_{s}\right\Vert_{L_{2}}
+\left\Vert\nabla\log p_{T-s}(\mathbf{0})\right\Vert
\nonumber
\\
&\leq L(T-s)\left\Vert\tilde{\mathbf{x}}_{s}\right\Vert_{L_{2}}+L_{1}T+\Vert\nabla\log p_{0}(\mathbf{0})\Vert,
\end{align}
where we applied Assumption~\ref{assump:M:1} to obtain the last inequality above.

Therefore, we have
\begin{align}
\left\Vert\tilde{\mathbf{x}}_{t}-\tilde{\mathbf{x}}_{(k-1)\eta}\right\Vert_{L_{2}}
&\leq
\omega(T)\int_{(k-1)\eta}^{t}\left[f(T-s)
+\frac{1}{2}(g(T-s))^{2}L(t-s)\right]\mathrm{d}s
\nonumber
\\
&\qquad\qquad
+\left(L_{1}T+\Vert\nabla\log p_{0}(\mathbf{0})\Vert\right)\int_{(k-1)\eta}^{t}\frac{1}{2}(g(T-s))^{2}\mathrm{d}s,
\end{align}
where we recall that
$\tilde{\mathbf{x}}_{s}$ has the same distribution
as $\mathbf{x}_{T-s}$
and we also recall from \eqref{c:2:source} that $\omega(T)=\sup_{0\leq t\leq T}\Vert\mathbf{x}_{t}\Vert_{L_{2}}$
with an explicit formula given in \eqref{c:2:defn} (see the derivation that leads
to \eqref{c:2:source} in the proof of Proposition~\ref{prop:iterates})
Hence, we conclude that uniformly for $(k-1)\eta\leq t\leq k\eta$,
\begin{align*}
\left\Vert\mathbf{y}_{t}-\mathbf{y}_{(k-1)\eta}\right\Vert_{L_{2}}
&\leq\left(\theta(T)+\omega(T)\right)\int_{(k-1)\eta}^{k\eta}\left[f(T-s)+\frac{1}{2}(g(T-s))^{2}L(T-s)\right]\mathrm{d}s
\\
&\qquad
+\left(L_{1}T+\Vert\nabla\log p_{0}(\mathbf{0})\Vert\right)\int_{(k-1)\eta}^{k\eta}\frac{1}{2}(g(T-s))^{2}\mathrm{d}s.
\end{align*}
This completes the proof.
\end{proof}



\section{Derivation of Results in Section~\ref{sec:examples}}\label{appendix:examples}

In this section, we prove the results that are summarized in Table~\ref{table:summary:complexity} in Section~\ref{sec:examples}. We discuss variance exploding SDEs in Appendix~\ref{app:VESDE}, variance preserving SDEs in Appendix~\ref{app:VPSDE}.

\subsection{Variance-Exploding SDEs} \label{app:VESDE}

In this section, we consider variance-exploding SDEs 
with $f(t)\equiv 0$ in the forward process \eqref{OU:SDE}. We can immediately obtain the following corollary of Theorem~\ref{thm:discrete:2}.

\begin{corollary}\label{cor:VE:4}
Assume that Assumptions~\ref{assump:p0}, \ref{assump:M:1}, and \ref{assump:M} hold
and $\eta\leq\bar{\eta}$, where $\bar{\eta}>0$ is defined in \eqref{bar:eta}.
Then, we have
\begin{align}
&\mathcal{W}_{2}(\mathcal{L}(\mathbf{u}_{K}),p_{0})
\\
&\leq e^{-\int_{0}^{K\eta}\mu(t)\mathrm{d}t} \Vert\mathbf{x}_{0}\Vert_{L_{2}} 
\nonumber
\\
&\qquad
+\sum_{k=1}^{K}\prod_{j=k+1}^{K}\left(1-\int_{(j-1)\eta}^{j\eta}\delta(T-t)\mathrm{d}t+\frac{L_{1}}{2}\eta\int_{(j-1)\eta}^{j\eta}(g(T-t))^{2}\mathrm{d}t\right)\nonumber
\\
&\quad
\cdot\Bigg(
\frac{L_{1}}{2}\eta\left(1+2  \Vert\mathbf{x}_{0}\Vert_{L_{2}} 
+\sqrt{d}\left(\int_{0}^{T}(g(t))^{2}\mathrm{d}t\right)^{1/2}\right)\int_{(k-1)\eta}^{k\eta}(g(T-t))^{2}\mathrm{d}t
\nonumber
\\
&\qquad
+\frac{M}{2}\int_{(k-1)\eta}^{k\eta}(g(T-t))^{2}\mathrm{d}t+\sqrt{\eta}\nu_{k,\eta}\left(\int_{(k-1)\eta}^{k\eta}\frac{1}{4}(g(T-t))^{4}(L(T-t))^{2}\mathrm{d}t\right)^{1/2}\Bigg),
\end{align}
where for any $0\leq t\leq T$:
\begin{align}\label{mu:definition:f:0}
\delta(T-t):=\frac{\frac{1}{2}(g(T-t))^{2}}{\frac{1}{m_{0}}+\int_{0}^{T-t}(g(s))^{2}\mathrm{d}s}
-\frac{\eta}{4}(g(T-t))^{4}(L(T-t))^{2},
\end{align}
where $L(T-t):=\min\left(\left(\int_{0}^{T-t}(g(s))^{2}\mathrm{d}s\right)^{-1},L_{0}\right)$ and
\begin{align}
\nu_{k,\eta}&:=\left(2\Vert\mathbf{x}_{0}\Vert_{L_{2}}+\sqrt{d}\left(\int_{0}^{T}(g(t))^{2}\mathrm{d}t\right)^{1/2}\right)\int_{(k-1)\eta}^{k\eta}
\frac{1}{2}(g(T-s))^{2}L(T-s)\mathrm{d}s
\nonumber
\\
&\qquad\qquad\qquad
+\left(L_{1}T+\Vert\nabla\log p_{0}(\mathbf{0})\Vert\right)\int_{(k-1)\eta}^{k\eta}\frac{1}{2}(g(T-s))^{2}\mathrm{d}s,\label{h:k:eta:main:f:0}
\end{align}
where $\mu(t)$ is defined as:
\begin{equation}\label{c:t:f:0}
\mu(t):=\frac{\frac{1}{2}(g(t))^{2}}{\frac{1}{m_{0}}+\int_{0}^{t}(g(s))^{2}\mathrm{d}s}.
\end{equation}
\end{corollary}


The term $\Vert\mathbf{x}_{0}\Vert_{L_{2}}$ in Corollary~\ref{cor:VE:4}
has square root dependence on the dimension $d$.
Indeed, by Lemma~11 in \cite{distMCMC}, we have
\begin{align}\label{eq:L2-x0}
\Vert\mathbf{x}_{0}\Vert_{L_{2}}\leq\sqrt{2d/m_{0}}+\Vert\mathbf{x}_{\ast}\Vert,
\end{align}
where $\mathbf{x}_{\ast}$ is the unique minimizer of $-\log p_{0}$.

In the next few sections, we consider special functions $g$ in Corollary~\ref{cor:VE:4} and
derive the corresponding results in Table~\ref{table:summary:complexity}.

\subsubsection{Example: $f(t)\equiv 0$ and $g(t)=ae^{bt}$}
When $g(t)=ae^{bt}$ for some $a,b>0$, 
we can obtain the following result from Corollary~\ref{cor:VE:4}.

\begin{corollary}[Restatement of Corollary~\ref{cor:VE:5:main:paper}]\label{cor:VE:5}
Let $g(t)=ae^{bt}$ for some $a,b>0$.
Then, we have $\mathcal{W}_{2}(\mathcal{L}(\mathbf{u}_{K}),p_{0})\leq\mathcal{O}(\epsilon)$
after $K=\mathcal{O}\left(\frac{d^{3/2}\log(d/\epsilon)}{\epsilon^{3}}\right)$ iterations
provided that $M\leq\frac{\epsilon^{2}}{\sqrt{d}}$ and
$\eta\leq\frac{\epsilon^{3}}{d^{3/2}}$.
\end{corollary}

\begin{proof}
Let $g(t)=ae^{bt}$ for some $a,b>0$.
First, we can compute that
\begin{align*}
(g(t))^{2}L(t)
=\min\left(\frac{(g(t))^{2}}{\int_{0}^{t}(g(s))^{2}\mathrm{d}s},L_{0}(g(t))^{2}\right)
=\min\left(\frac{2be^{2bt}}{e^{2bt}-1},L_{0}\frac{a^{2}}{4b^{2}}(e^{2bt}-1)^{2}\right).
\end{align*}
If $e^{2bt}\geq 2$, then $e^{2bt}-1\geq\frac{1}{2}e^{2bt}$ and
$(g(t))^{2}L(t)\leq 4b$. 
On the other hand, if $e^{2bt}<2$, then $(g(t))^{2}L(t)\leq L_{0}\frac{a^{2}}{4b^{2}}$. 
Therefore, for any $0\leq t\leq T$,
\begin{equation*}
(g(t))^{2}L(t)
\leq\max\left(4b,\frac{L_{0}a^{2}}{4b^{2}}\right).
\end{equation*}
By the definition of $\mu(t)$ in \eqref{c:t:f:0}, we can compute that
\begin{equation}\label{c:t:expression}
\mu(t)=\frac{\frac{1}{2}m_{0}(g(t))^{2}}{1+m_{0}\int_{0}^{t}(g(s))^{2}\mathrm{d}s}
=\frac{\frac{1}{2}m_{0}a^{2}e^{2bt}}{1+m_{0}\frac{a^{2}}{2b}(e^{2bt}-1)}.
\end{equation}
This implies that
\begin{equation}\label{c:t:integral}
\int_{0}^{t}\mu(s)\mathrm{d}s
=\frac{1}{2}\int_{0}^{t}\frac{2bm_{0}a^{2}e^{2bs}\mathrm{d}s}{2b-m_{0}a^{2}+m_{0}a^{2}e^{2bs}}
=\frac{1}{2}\log\left(\frac{2b-m_{0}a^{2}+m_{0}a^{2}e^{2bt}}{2b}\right).
\end{equation}
By letting $t=T=K\eta$ in \eqref{c:t:integral} and using \eqref{eq:L2-x0}, we obtain
\begin{equation*}
e^{-\int_{0}^{K\eta}\mu(t)\mathrm{d}t} \Vert\mathbf{x}_{0}\Vert_{L_{2}} 
\le \frac{\sqrt{2b}}{\sqrt{2b-m_{0}a^{2}+m_{0}a^{2}e^{2bK\eta}}}\left(\sqrt{2d/m_{0}}+\Vert\mathbf{x}_{\ast}\Vert\right).
\end{equation*}
Moreover,
\begin{align}
\nu_{k,\eta}&
\leq\left(2\sqrt{2d/m_{0}}+2\Vert\mathbf{x}_{\ast}\Vert+\sqrt{d}\frac{a}{\sqrt{2b}}\left(e^{2bT}-1\right)^{1/2}\right)
\max\left(4b,\frac{L_{0}a^{2}}{4b^{2}}\right)\frac{\eta}{2}
\nonumber
\\
&\qquad\qquad\qquad
+\left(L_{1}T+\Vert\nabla\log p_{0}(\mathbf{0})\Vert\right)
\frac{a^{2}}{4b}\left(e^{2b(T-(k-1)\eta)}-e^{2b(T-k\eta)}\right),
\end{align}
and for any $0\leq t\leq T$:
\begin{equation*}
0\leq1-\int_{(j-1)\eta}^{j\eta}
\delta(T-t)\mathrm{d}t+\frac{L_{1}}{2}\eta\int_{(j-1)\eta}^{j\eta}(g(T-t))^{2}\mathrm{d}t
\leq 1,
\end{equation*}
where $\delta(\cdot)$ is defined in \eqref{mu:definition:f:0}, provided that
the condition $\eta\leq\bar{\eta}$ (where $\bar{\eta}$ is defined in \eqref{bar:eta}) holds, that is:
\begin{equation}
\eta\leq
\min\left\{\min_{0\leq t\leq T}\left\{\frac{\frac{\frac{1}{2}(g(t))^{2}}{\frac{1}{m_{0}}+\int_{0}^{t}(g(s))^{2}\mathrm{d}s}}
{\frac{1}{4}(g(t))^{4}(L(t))^{2}+\frac{L_{1}}{2}(g(t))^{2}}\right\},
\min_{0\leq t\leq T}
\left\{\frac{\frac{1}{m_{0}}+\int_{0}^{t}(g(s))^{2}\mathrm{d}s}{\frac{1}{2}(g(t))^{2}}\right\}\right\},
\end{equation}
which holds provided that
\begin{equation}\label{implies:assump:stepsize}
\eta\leq
\min\left\{\min_{0\leq t\leq T}\left\{\frac{\frac{\frac{1}{2}a^{2}e^{2bt}}{\frac{1}{m_{0}}+\frac{a^{2}}{2b}(e^{2bt}-1)}}
{\frac{1}{4}\max\left(16b^{2},\frac{L_{0}^{2}a^{4}}{16b^{4}}\right)+\frac{L_{1}}{2}a^{2}e^{2bt}}\right\},
\min_{0\leq t\leq T}
\left\{\frac{\frac{1}{m_{0}}+\frac{a^{2}}{2b}(e^{2bt}-1)}{\frac{1}{2}a^{2}e^{2bt}}\right\}\right\}.
\end{equation}
Since $1-x\leq e^{-x}$ for any $0\leq x\leq 1$, 
we conclude that
\begin{align*}
&\prod_{j=k+1}^{K}\left(1-\int_{(j-1)\eta}^{j\eta}\delta(T-t)\mathrm{d}t+\frac{L_{1}}{2}\eta\int_{(j-1)\eta}^{j\eta}(g(T-t))^{2}\mathrm{d}t\right)
\nonumber
\\
&\leq
\prod_{j=k+1}^{K}e^{-\int_{(j-1)\eta}^{j\eta}\delta(T-t)\mathrm{d}t+\frac{L_{1}}{2}\eta\int_{(j-1)\eta}^{j\eta}(g(T-t))^{2}\mathrm{d}t}
=e^{-\int_{k\eta}^{K\eta}\delta(T-t)\mathrm{d}t+\frac{L_{1}}{2}\eta\int_{k\eta}^{K\eta}(g(T-t))^{2}\mathrm{d}t},
\end{align*}
where $\delta(\cdot)$ is defined in \eqref{mu:definition:f:0}.
Moreover,
\begin{align*}
\int_{k\eta}^{K\eta}
\delta(T-t)\mathrm{d}t
&\geq
\int_{k\eta}^{K\eta}\frac{\frac{1}{2}m_{0}a^{2}e^{2b(T-t)}\mathrm{d}t}{1+m_{0}\frac{a^{2}}{2b}(e^{2b(T-t)}-1)}
-\frac{1}{4}(K-k)\eta^{2}\max\left(16b^{2},\frac{L_{0}^{2}a^{4}}{16b^{4}}\right)
\nonumber
\\
&=\frac{1}{2}\log\left(\frac{2b-m_{0}a^{2}+m_{0}a^{2}e^{2b(T-k\eta)}}{2b-m_{0}a^{2}+m_{0}a^{2}e^{2b(T-K\eta)}}\right)
-\frac{1}{4}(K-k)\eta^{2}\max\left(16b^{2},\frac{L_{0}^{2}a^{4}}{16b^{4}}\right),
\end{align*}
and
\begin{equation*}
\frac{L_{1}}{2}\eta\int_{k\eta}^{K\eta}(g(T-t))^{2}\mathrm{d}t=\frac{L_{1}}{2}\eta\frac{a^{2}}{2b}\left(e^{2b(K-k)\eta}-1\right).
\end{equation*}
By applying Corollary~\ref{cor:VE:4} with $T=K\eta$, we conclude that
\begin{align*}
&\mathcal{W}_{2}(\mathcal{L}(\mathbf{u}_{K}),p_{0})
\nonumber
\\
&\leq\frac{\sqrt{2b}\left(\sqrt{2d/m_{0}}+\Vert\mathbf{x}_{\ast}\Vert\right)}{\sqrt{2b-m_{0}a^{2}+m_{0}a^{2}e^{2bK\eta}}}
\nonumber
\\
&\quad
+\sum_{k=1}^{K}
\frac{\sqrt{2b}}{\sqrt{2b-m_{0}a^{2}+m_{0}a^{2}e^{2b(K-k)\eta}}}e^{(K-k)\eta^{2}\frac{1}{4}\max\left(16b^{2},\frac{L_{0}^{2}a^{4}}{16b^{4}}\right)
+\frac{L_{1}}{2}\eta\frac{a^{2}}{2b}(e^{2b(K-k)\eta}-1)}\nonumber
\\
&\qquad
\cdot\Bigg(
\left(\frac{M}{2}+\frac{L_{1}}{2}\eta\left(1+2\left(\sqrt{2d/m_{0}}+\Vert\mathbf{x}_{\ast}\Vert\right)+\sqrt{d}\frac{a}{\sqrt{2b}}(e^{2bK\eta}-1)^{1/2}\right)\right)
\nonumber
\\
&\qquad\qquad\qquad\cdot\frac{a^{2}}{2b}\left(e^{2b(K-k+1)\eta}-e^{2b(K-k)\eta}\right)
\nonumber
\\
&\quad
+\frac{\eta}{2}\max\left(4b,\frac{L_{0}a^{2}}{4b^{2}}\right)\cdot\Bigg(\left(2\sqrt{2d/m_{0}}+2\Vert\mathbf{x}_{\ast}\Vert+\sqrt{d}\frac{a}{\sqrt{2b}}\left(e^{2bT}-1\right)^{1/2}\right)
\max\left(4b,\frac{L_{0}a^{2}}{4b^{2}}\right)\frac{\eta}{2}
\nonumber
\\
&\qquad\qquad\qquad
+\left(L_{1}T+\Vert\nabla\log p_{0}(\mathbf{0})\Vert\right)
\frac{a^{2}}{4b}\left(e^{2b(T-(k-1)\eta)}-e^{2b(T-k\eta)}\right)\Bigg)\Bigg).
\end{align*}

By the mean-value theorem, we have
\begin{equation*}
e^{2b(K-(k-1))\eta)}-e^{2b(K-k)\eta}
\leq 
2be^{2b(K-(k-1))\eta}\eta,
\end{equation*}
which implies that
\begin{align*}
&\mathcal{W}_{2}(\mathcal{L}(\mathbf{u}_{K}),p_{0})
\nonumber
\\
&\leq
\mathcal{O}\left(\frac{\sqrt{d}}{e^{bK\eta}}\right)
+\mathcal{O}\Bigg(e^{\frac{1}{4}K\eta^{2}\max\left(16b^{2},\frac{L_{0}^{2}a^{4}}{16b^{4}}\right)
+\frac{L_{1}}{2}\eta\frac{a^{2}}{2b}e^{2bK\eta}}
\nonumber
\\
&\qquad\cdot
\sum_{k=1}^{K}
\frac{1}{e^{b(K-k)\eta}}
\cdot\Bigg(
\left(M+L_{1}\eta\sqrt{d}e^{bK\eta}\right)e^{2b(K-(k-1))\eta}\eta+\eta e^{bK\eta}\eta\sqrt{d}
\nonumber
\\
&\qquad\qquad\qquad\qquad\qquad\qquad
+\eta^{2}e^{2b(K-(k-1))\eta}
+e^{2b(K-(k-1))\eta}\eta^{2}(K\eta)\Bigg)\Bigg)
\nonumber
\\
&\leq
\mathcal{O}\left(\frac{\sqrt{d}}{e^{bK\eta}}\right)
\\
&\qquad
+\mathcal{O}\left(e^{K\eta^{2}\max\left(16b^{2},\frac{L_{0}^{2}a^{4}}{16b^{4}}\right)}
\cdot\Bigg(
\left(M+L_{1}\eta\sqrt{d}e^{bK\eta}\right)e^{bK\eta}+K\eta^{2}\sqrt{d}
+\eta(K\eta)e^{bK\eta}\Bigg)\right)
\\
&\leq\mathcal{O}(\epsilon),
\end{align*}
and \eqref{implies:assump:stepsize} holds such that 
the condition $\eta\leq\bar{\eta}$ (where $\bar{\eta}$ is defined in \eqref{bar:eta}) holds
provided that
\begin{equation*}
K\eta=\frac{\log(\sqrt{d}/\epsilon)}{b},
\qquad
M\leq\frac{\epsilon^{2}}{\sqrt{d}},
\qquad
\eta\leq\frac{\epsilon^{3}}{d^{3/2}},
\end{equation*}
which implies that
$K\geq\mathcal{O}\left(\frac{d^{3/2}\log(d/\epsilon)}{\epsilon^{3}}\right)$.
This completes the proof.
\end{proof}

\subsubsection{Example: $f(t)\equiv 0$ and $g(t)=(b+at)^{c}$}

\cite{Karras2022} considers $f(t)\equiv 0, g(t)=\sqrt{2t}$
with non-uniform discretization time steps, where the time steps are defined according to a polynomial noise schedule. 
Inspired by \cite{Karras2022}, we next consider $g(t)=(b+at)^{c}$ for some $a,b,c>0$,
we can obtain the following result from Corollary~\ref{cor:VE:4}.


\begin{corollary}[Restatement of Corollary~\ref{cor:VE:8:main:paper}]\label{cor:VE:8}
Let $g(t)=(b+at)^{c}$ for some $a,b>0$, $c\geq 1/2$.
Then, we have $\mathcal{W}_{2}(\mathcal{L}(\mathbf{u}_{K}),p_{0})\leq\mathcal{O}(\epsilon)$
after $K=\mathcal{O}\left(\frac{d^{\frac{1}{(2c+1)}+\frac{3}{2}}}{\epsilon^{\frac{2}{2c+1}+3}}\right)$ iterations
provided that $M\leq\frac{\epsilon^{2}}{\sqrt{d}}$ and $\eta\leq\frac{\epsilon^{3}}{d^{\frac{3}{2}}}$.
\end{corollary}


\begin{proof}
When $g(t)=(b+at)^{c}$ for some $a,b,c>0$, we can compute that
\begin{align}\label{follow:step:1}
(g(t))^{2}L(t)
=\min\left(\frac{(b+at)^{c}}{\frac{1}{a(2c+1)}((b+at)^{2c+1}-b^{2c+1})},L_{0}(b+at)^{2c}\right).
\end{align}
If $t\geq\frac{b}{a}$, then
\begin{equation}\label{follow:step:2}
\frac{(b+at)^{c}}{\frac{1}{a(2c+1)}((b+at)^{2c+1}-b^{2c+1})}
\leq
\frac{(b+at)^{c}}{\frac{1}{a(2c+1)}(1-\frac{1}{2^{2c+1}})(b+at)^{2c+1}}
\leq\frac{a(2c+1)}{(1-\frac{1}{2^{2c+1}})b}.
\end{equation}
If $t\leq\frac{b}{a}$, then
\begin{equation}\label{follow:step:3}
L_{0}(b+at)^{2c}\leq L_{0}(2b)^{2c}.
\end{equation}
Therefore, it follows from \eqref{follow:step:1}, \eqref{follow:step:2} and \eqref{follow:step:3} that
\begin{align*}
(g(t))^{2}L(t)
\leq\max\left(\frac{a(2c+1)}{(1-\frac{1}{2^{2c+1}})b},L_{0}(2b)^{2c}\right).
\end{align*}
By \eqref{c:t:f:0}, we have
\begin{equation*}
e^{-\int_{0}^{K\eta}\mu(t)\mathrm{d}t}\left(\sqrt{2d/m_{0}}+\Vert\mathbf{x}_{\ast}\Vert\right)
=\frac{\sqrt{2d/m_{0}}+\Vert\mathbf{x}_{\ast}\Vert}{\sqrt{1+\frac{m_{0}}{a(2c+1)}((b+aK\eta)^{2c+1}-b^{2c+1})}}.
\end{equation*}

Furthermore,
\begin{align}
\nu_{k,\eta}&
\leq
\left(2\left(\sqrt{2d/m_{0}}+\Vert\mathbf{x}_{\ast}\Vert\right)+\sqrt{d}\left(\frac{(b+aT)^{2c+1}-b^{2c+1}}{a(2c+1)}\right)^{1/2}\right)
\nonumber
\\
&\qquad\qquad\qquad\cdot
\frac{1}{2}\max\left(\frac{a(2c+1)}{(1-\frac{1}{2^{2c+1}})b},L_{0}(2b)^{2c}\right)\eta
\nonumber
\\
&\qquad
+\left(L_{1}T+\Vert\nabla\log p_{0}(\mathbf{0})\Vert\right)
\frac{1}{2}\frac{(b+a(T-(k-1)\eta))^{2c+1}-(b+a(T-k\eta))^{2c+1}}{a(2c+1)},
\end{align}
and for any $0\leq t\leq T$:
\begin{align*}
\delta(t)
&=\frac{\frac{1}{2}(b+at)^{2c}}{\frac{1}{m_{0}}+\frac{1}{a(2c+1)}((b+at)^{2c+1}-b^{2c+1})}
\nonumber
\\
&\qquad
-\frac{1}{4}\eta\min\left(\frac{(b+at)^{2c}}{\frac{1}{a^{2}(2c+1)^{2}}((b+at)^{2c+1}-b^{2c+1})^{2}},L_{0}^{2}(b+at)^{4c}\right),
\end{align*}
(where $\delta(\cdot)$ is defined in \eqref{mu:definition:f:0}) satisfies
\begin{equation*}
0\leq 1-\int_{(j-1)\eta}^{j\eta}
\delta(T-t)\mathrm{d}t+\frac{L_{1}}{2}\eta\int_{(j-1)\eta}^{j\eta}(g(T-t))^{2}\mathrm{d}t
\leq 1,
\end{equation*}
provided that the condition $\eta\leq\bar{\eta}$ (where $\bar{\eta}$ is defined in \eqref{bar:eta}) holds.

Since $1-x\leq e^{-x}$ for any $0\leq x\leq 1$, 
we conclude that
\begin{align*}
&\prod_{j=k+1}^{K}\left(1-\int_{(j-1)\eta}^{j\eta}\delta(T-t)\mathrm{d}t+\frac{L_{1}}{2}\eta\int_{(j-1)\eta}^{j\eta}(g(T-t))^{2}\mathrm{d}t\right)
\nonumber
\\
&\leq
\prod_{j=k+1}^{K}e^{-\int_{(j-1)\eta}^{j\eta}\delta(T-t)\mathrm{d}t+\frac{L_{1}}{2}\eta\int_{(j-1)\eta}^{j\eta}(g(T-t))^{2}\mathrm{d}t}
=e^{-\int_{k\eta}^{K\eta}\delta(T-t)\mathrm{d}t+\frac{L_{1}}{2}\eta\int_{k\eta}^{K\eta}(g(T-t))^{2}\mathrm{d}t},
\end{align*}
where $\delta(\cdot)$ is defined in \eqref{mu:definition:f:0}.
Moreover,
\begin{align*}
\int_{k\eta}^{K\eta}
\delta(T-t)\mathrm{d}t
&\geq
\frac{1}{2}\log\left(1+\frac{m_{0}}{a(2c+1)}\left((b+a(K-k)\eta)^{2c+1}-b^{2c+1}\right)\right)
\nonumber
\\
&\qquad
-\frac{1}{4}(K-k)\eta^{2}\max\left(\frac{a^{2}(2c+1)^{2}}{(1-\frac{1}{2^{2c+1}})^{2}b^{2}},L_{0}^{2}(2b)^{4c}\right),
\end{align*}
and we can compute that
\begin{equation*}
\frac{1}{2}L_{1}\eta\int_{k\eta}^{K\eta}(g(T-t))^{2}\mathrm{d}t
=\frac{1}{2}\frac{L_{1}\eta}{a(2c+1)}\left((b+a(K-k)\eta)^{2c+1}-b^{2c+1}\right).
\end{equation*}
By applying Corollary~\ref{cor:VE:4} with $T=K\eta$ and \eqref{eq:L2-x0}, we conclude that
\begin{align*}
&\mathcal{W}_{2}(\mathcal{L}(\mathbf{u}_{K}),p_{0})
\nonumber
\\
&\leq
\frac{\sqrt{2d/m_{0}}+\Vert\mathbf{x}_{\ast}\Vert}{\sqrt{1+\frac{m_{0}}{a(2c+1)}((b+aK\eta)^{2c+1}-b^{2c+1})}}
\nonumber
\\
&\quad
+\sum_{k=1}^{K}\frac{e^{\frac{1}{4}(K-k)\eta^{2}\max\left(\frac{a^{2}(2c+1)^{2}}{(1-\frac{1}{2^{2c+1}})^{2}b^{2}},L_{0}^{2}(2b)^{4c}\right)+\frac{1}{2}\frac{L_{1}\eta}{a(2c+1)}((b+a(K-k)\eta)^{2c+1}-b^{2c+1})}}{\sqrt{1+\frac{m_{0}}{a(2c+1)}\left((b+a(K-k)\eta)^{2c+1}-b^{2c+1}\right)}}
\nonumber
\\
&\quad
\cdot\Bigg(
\left(\frac{L_{1}}{2}\eta\left(1+2\left(\sqrt{2d/m_{0}}+\Vert\mathbf{x}_{\ast}\Vert\right)
+\sqrt{d}\left(\frac{(b+aK\eta)^{2c+1}-b^{2c+1}}{a(2c+1)}\right)^{1/2}\right)+\frac{M}{2}\right)
\\
&\qquad\qquad\qquad
\cdot\frac{(b+a(K\eta-(k-1)\eta))^{2c+1}-(b+a(K\eta-k\eta))^{2c+1}}{a(2c+1)}
\nonumber
\\
&\qquad
+\frac{1}{2}\eta\max\left(\frac{a(2c+1)}{(1-\frac{1}{2^{2c+1}})b},L_{0}(2b)^{2c}\right)
\nonumber
\\
&\qquad
\cdot
\Bigg(
\left(2\left(\sqrt{2d/m_{0}}+\Vert\mathbf{x}_{\ast}\Vert\right)+\sqrt{d}\left(\frac{(b+aK\eta)^{2c+1}-b^{2c+1}}{a(2c+1)}\right)^{1/2}\right)
\nonumber
\\
&\qquad\qquad\qquad\cdot
\frac{1}{2}\max\left(\frac{a(2c+1)}{(1-\frac{1}{2^{2c+1}})b},L_{0}(2b)^{2c}\right)\eta
\nonumber
\\
&\qquad
+\left(L_{1}K\eta+\Vert\nabla\log p_{0}(\mathbf{0})\Vert\right)
\frac{1}{2}\frac{(b+a(K\eta-(k-1)\eta))^{2c+1}-(b+a(K\eta-k\eta))^{2c+1}}{a(2c+1)}
\Bigg)\Bigg).
\end{align*}
This implies that
\begin{align*}
&\mathcal{W}_{2}(\mathcal{L}(\mathbf{u}_{K}),p_{0})
\\
&\leq
\mathcal{O}\Bigg(\frac{\sqrt{d}}{(K\eta)^{\frac{2c+1}{2}}}+e^{\mathcal{O}((K\eta)\eta+(K\eta)^{2c+1}\eta)}
\sum_{k=1}^{K}\frac{1}{((K-k)\eta)^{\frac{2c+1}{2}}}
\\
&\qquad
\cdot\left(\left(\sqrt{d}(K\eta)^{\frac{2c+1}{2}}L_{1}\eta+M\right)((K-k)\eta)^{2c}\eta
+\eta\left(\eta\sqrt{d}(K\eta)^{\frac{2c+1}{2}}+(K\eta)\eta((K-k)\eta)^{2c}\right)\right)\Bigg)
\\
&\leq
\mathcal{O}\Bigg(\frac{\sqrt{d}}{(K\eta)^{\frac{2c+1}{2}}}+e^{\mathcal{O}((K\eta)\eta+(K\eta)^{2c+1}\eta)}
\nonumber
\\
&\qquad\qquad
\cdot\left(\left(\sqrt{d}(K\eta)^{\frac{2c+1}{2}}L_{1}\eta+M\right)(K\eta)^{c+\frac{1}{2}}
+\eta\left(K\eta\sqrt{d}+(K\eta)(K\eta)^{c+\frac{1}{2}}\right)\right)\Bigg)
\\
&\leq\mathcal{O}(\epsilon),
\end{align*}
and the condition $\eta\leq\bar{\eta}$ (where $\bar{\eta}$ is defined in \eqref{bar:eta}) holds
provided that
\begin{equation*}
K\eta=\frac{d^{\frac{1}{(2c+1)}}}{\epsilon^{\frac{2}{2c+1}}},
\qquad
M\leq\frac{\epsilon^{2}}{\sqrt{d}},
\qquad
\eta\leq\frac{\epsilon^{3}}{d^{\frac{3}{2}}},
\end{equation*}
so that
$K\geq\mathcal{O}\left(\frac{d^{\frac{1}{(2c+1)}+\frac{3}{2}}}{\epsilon^{\frac{2}{2c+1}+3}}\right)$.
This completes the proof.
\end{proof}

\subsection{Variance-Preserving SDEs} \label{app:VPSDE}

In this section, we consider  Variance-Preserving SDEs with $f(t)=\frac{1}{2}\beta(t)$ and $g(t)=\sqrt{\beta(t)}$ in the forward process \eqref{OU:SDE}, where $\beta(t)$ is often chosen as some non-decreasing function in practice. 
 We can obtain the following corollary of Theorem~\ref{thm:discrete:2}. 



\begin{corollary}\label{cor:VP:4}
Under the assumptions of Theorem~\ref{thm:discrete:2}, we have
\begin{align}
&\mathcal{W}_{2}(\mathcal{L}(\mathbf{u}_{K}),p_{0})
\\
&\leq\frac{\Vert\mathbf{x}_{0}\Vert_{L_{2}}}{\sqrt{m_{0}e^{\int_{0}^{K\eta}\beta(s)\mathrm{d}s}+1-m_{0}}}
\nonumber
\\
&\quad+\sum_{k=1}^{K}
\frac{e^{\int_{k\eta}^{K\eta}\frac{1}{2}\beta(K\eta-t)\mathrm{d}t+\int_{k\eta}^{K\eta}\eta\max(1,L_{0}^{2})(\beta(K\eta-t))^{2}\mathrm{d}t+\frac{1}{2}L_{1}\eta\int_{k\eta}^{K\eta}\beta(K\eta-t)\mathrm{d}t}}
{\left(m_{0}e^{\int_{0}^{(K-k)\eta}\beta(s)\mathrm{d}s}+1-m_{0}\right)^{\frac{1}{2}e^{-\frac{1}{2}\eta\max_{0\leq t\leq T}\beta(t)}}}
\nonumber
\\
&\qquad
\cdot
\Bigg(\left(\frac{L_{1}}{2}\eta\left(1+\Vert\mathbf{x}_{0}\Vert_{L_{2}}+\left(\Vert\mathbf{x}_{0}\Vert_{L_{2}}^{2}+d\right)^{1/2}\right)
+\frac{M}{2}\right)2\left(e^{\int_{(k-1)\eta}^{k\eta}\frac{1}{2}\beta(T-v)\mathrm{d}v}-1\right)
\nonumber
\\
&\quad
+\sqrt{\eta}\max(1,L_{0})\max_{(k-1)\eta\leq t\leq k\eta}\beta(K\eta-t)\left(e^{\int_{(k-1)\eta}^{k\eta}\beta(T-v)\mathrm{d}v}-1\right)^{1/2}
\nonumber
\\
&\qquad\qquad\cdot
\Bigg(\left(\Vert\mathbf{x}_{0}\Vert_{L_{2}}+\left(\Vert\mathbf{x}_{0}\Vert_{L_{2}}^{2}+d\right)^{1/2}\right)
\left(\frac{1}{2}+\max(1,L_{0})\right)
\int_{T-k\eta}^{T-(k-1)\eta}\beta(s)\mathrm{d}s
\nonumber
\\
&\qquad\qquad\qquad\qquad
+\left(L_{1}T+\Vert\nabla\log p_{0}(\mathbf{0})\Vert\right)\int_{T-k\eta}^{T-(k-1)\eta}\frac{1}{2}\beta(s)\mathrm{d}s\Bigg)
\Bigg).
\end{align}
\end{corollary}

\begin{proof}
We apply Theorem~\ref{thm:discrete:2} applied to the variance-preserving SDE ($f(t)=\frac{1}{2}\beta(t)$ and $g(t)=\sqrt{\beta(t)}$).
First, we can compute that
\begin{align*}
L(T-t)
&=\min\left(\left(\int_{0}^{T-t}e^{-2\int_{s}^{T-t}f(v)\mathrm{d}v}(g(s))^{2}\mathrm{d}s\right)^{-1},
\left(e^{\int_{0}^{T-t}f(s)\mathrm{d}s}\right)^{2}L_{0}\right)
\nonumber
\\
&=\min\left(\frac{1}{1-e^{-\int_{0}^{T-t}\beta(s)\mathrm{d}s}},e^{\int_{0}^{T-t}\beta(s)\mathrm{d}s}L_{0}\right).
\end{align*}
If $e^{\int_{0}^{T-t}\beta(s)\mathrm{d}s}\geq 2$, then $\frac{1}{1-e^{-\int_{0}^{T-t}\beta(s)\mathrm{d}s}}\leq 2$
and otherwise $e^{\int_{0}^{T-t}\beta(s)\mathrm{d}s}L_{0}\leq 2L_{0}$. 
Therefore, for any $0\leq t\leq T$, 
\begin{equation*}
L(T-t)\leq 2\max(1,L_{0}).
\end{equation*}

By applying Theorem~\ref{thm:discrete:2}, we have
\begin{align*}
&\mathcal{W}_{2}(\mathcal{L}(\mathbf{u}_{K}),p_{0})
\\
&\leq e^{-\int_{0}^{K\eta}\mu(t)\mathrm{d}t}\Vert\mathbf{x}_{0}\Vert_{L_{2}}
+\sum_{k=1}^{K}
\prod_{j=k+1}^{K}\left(1-\int_{(j-1)\eta}^{j\eta}\delta_{j}(T-t)\mathrm{d}t+\frac{L_{1}}{2}\eta\int_{(j-1)\eta}^{j\eta}(g(T-t))^{2}\mathrm{d}t\right)\nonumber
\\
&\quad\cdot e^{\int_{k\eta}^{K\eta}f(T-t)\mathrm{d}t}\Bigg(\frac{L_{1}}{2}\eta\left(1+\Vert\mathbf{x}_{0}\Vert_{L_{2}}
+\omega(T)\right)\int_{(k-1)\eta}^{k\eta}e^{\int_{t}^{k\eta}f(T-s)\mathrm{d}s}(g(T-t))^{2}\mathrm{d}t
\nonumber
\\
&+\frac{M}{2}\int_{(k-1)\eta}^{k\eta}e^{\int_{t}^{k\eta}f(T-s)\mathrm{d}s}(g(T-t))^{2}\mathrm{d}t
\\
&\qquad\qquad
+\sqrt{\eta}\nu_{k,\eta}\left(\int_{(k-1)\eta}^{k\eta}\left[\frac{1}{2}e^{\int_{t}^{k\eta}f(T-s)\mathrm{d}s}(g(T-t))^{2}L(T-t)\right]^{2}\mathrm{d}t\right)^{1/2}\Bigg),
\end{align*}
where 
\begin{align*}
\nu_{k,\eta}&:=\left(\theta(T)+\omega(T)\right)\int_{(k-1)\eta}^{k\eta}\left[f(T-s)+\frac{1}{2}(g(T-s))^{2}L(T-s)\right]\mathrm{d}s
\\
&\qquad
+\left(L_{1}T+\Vert\nabla\log p_{0}(\mathbf{0})\Vert\right)\int_{(k-1)\eta}^{k\eta}\frac{1}{2}(g(T-s))^{2}\mathrm{d}s,
\end{align*}
where
\begin{align*}
\theta(T)=e^{-\int_{0}^{T}\frac{1}{2}\beta(s)\mathrm{d}s}\Vert\mathbf{x}_{0}\Vert_{L_{2}}
\leq\Vert\mathbf{x}_{0}\Vert_{L_{2}},
\end{align*}
and
\begin{align*}
\omega(T)
&=\sup_{0\leq t\leq T}\left(e^{-2\int_{0}^{t}f(s)\mathrm{d}s}\Vert\mathbf{x}_{0}\Vert_{L_{2}}^{2}
+d\int_{0}^{t}e^{-2\int_{s}^{t}f(v)\mathrm{d}v}(g(s))^{2}\mathrm{d}s\right)^{1/2}\nonumber
\\
&=\sup_{0\leq t\leq T}\left(e^{-\int_{0}^{t}\beta(s)\mathrm{d}s}\Vert\mathbf{x}_{0}\Vert_{L_{2}}^{2}
+d\left(1-e^{-\int_{0}^{t}\beta(s)\mathrm{d}s}\right)\right)^{1/2}
\leq\left(\Vert\mathbf{x}_{0}\Vert_{L_{2}}^{2}+d\right)^{1/2}.
\end{align*}
We can compute that
\begin{align*}
\int_{(k-1)\eta}^{k\eta}e^{\int_{t}^{k\eta}f(T-s)\mathrm{d}s}(g(T-t))^{2}\mathrm{d}t
&=\int_{(k-1)\eta}^{k\eta}e^{\int_{s}^{k\eta}\frac{1}{2}\beta(T-v)\mathrm{d}v}\beta(T-s)\mathrm{d}s
\\
&=2\left(e^{\int_{(k-1)\eta}^{k\eta}\frac{1}{2}\beta(T-v)\mathrm{d}v}-1\right).
\end{align*}
Furthermore, we have
\begin{align*}
\nu_{k,\eta}&\leq
\left(\Vert\mathbf{x}_{0}\Vert_{L_{2}}+\left(\Vert\mathbf{x}_{0}\Vert_{L_{2}}^{2}+d\right)^{1/2}\right)
\left(\frac{1}{2}+\max(1,L_{0})\right)
\int_{(k-1)\eta}^{k\eta}\beta(T-s)\mathrm{d}s
\nonumber
\\
&\qquad
+\left(L_{1}T+\Vert\nabla\log p_{0}(\mathbf{0})\Vert\right)\int_{T-k\eta}^{T-(k-1)\eta}\frac{1}{2}\beta(s)\mathrm{d}s
\\
&=\left(\Vert\mathbf{x}_{0}\Vert_{L_{2}}+\left(\Vert\mathbf{x}_{0}\Vert_{L_{2}}^{2}+d\right)^{1/2}\right)
\left(\frac{1}{2}+\max(1,L_{0})\right)
\int_{T-k\eta}^{T-(k-1)\eta}\beta(s)\mathrm{d}s
\nonumber
\\
&\qquad
+\left(L_{1}T+\Vert\nabla\log p_{0}(\mathbf{0})\Vert\right)\int_{T-k\eta}^{T-(k-1)\eta}\frac{1}{2}\beta(s)\mathrm{d}s.
\end{align*}

Next, for VP-SDE, 
we have $f(t)=\frac{1}{2}\beta(t)$ and $g(t)=\sqrt{\beta(t)}$ so that we can compute:
\begin{equation}
\mu(t)=\frac{\frac{1}{2}m_{0}\beta(t)}{e^{-\int_{0}^{t}\beta(s)\mathrm{d}s}+m_{0}\int_{0}^{t}e^{-\int_{s}^{t}\beta(v)\mathrm{d}v}\beta(s)\mathrm{d}s}
=\frac{\frac{1}{2}m_{0}\beta(t)}{e^{-\int_{0}^{t}\beta(s)\mathrm{d}s}+m_{0}(1-e^{-\int_{0}^{t}\beta(s)\mathrm{d}s})}.
\end{equation}
It follows that
\begin{equation}
\int_{0}^{T}\mu(t)\mathrm{d}t
=\frac{1}{2}\int_{0}^{\int_{0}^{T}\beta(s)\mathrm{d}s}\frac{m_{0}dx}{m_{0}+(1-m_{0})e^{-x}}
=\frac{1}{2}\log\left(m_{0}e^{\int_{0}^{T}\beta(s)\mathrm{d}s}+1-m_{0}\right).
\end{equation}
Hence, we obtain
\begin{align}
e^{-\int_{0}^{T}\mu(t)\mathrm{d}t}\Vert\mathbf{x}_{0}\Vert_{L_{2}}
=\frac{\Vert\mathbf{x}_{0}\Vert_{L_{2}}}{\sqrt{m_{0}e^{\int_{0}^{T}\beta(s)\mathrm{d}s}+1-m_{0}}}.
\end{align}

Under the condition $\eta\leq\bar{\eta}$ (where $\bar{\eta}$ is defined in \eqref{bar:eta}), 
\begin{equation*}
0\leq1-\int_{(j-1)\eta}^{j\eta}
\delta_{j}(T-t)\mathrm{d}t+\frac{L_{1}}{2}\eta\int_{(j-1)\eta}^{j\eta}(g(T-t))^{2}\mathrm{d}t
\leq 1.
\end{equation*}
Since $1-x\leq e^{-x}$ for any $0\leq x\leq 1$, 
we conclude that
\begin{align*}
&\prod_{j=k+1}^{K}\left(1-\int_{(j-1)\eta}^{j\eta}\delta_{j}(T-t)\mathrm{d}t+\frac{L_{1}}{2}\eta\int_{(j-1)\eta}^{j\eta}(g(T-t))^{2}\mathrm{d}t\right)
\nonumber
\\
&\leq
\prod_{j=k+1}^{K}e^{-\int_{(j-1)\eta}^{j\eta}\delta_{j}(T-t)\mathrm{d}t+\frac{L_{1}}{2}\eta\int_{(j-1)\eta}^{j\eta}(g(T-t))^{2}\mathrm{d}t}
\\
&=e^{-\sum_{j=k+1}^{K}\int_{(j-1)\eta}^{j\eta}\delta_{j}(T-t)\mathrm{d}t+\frac{L_{1}}{2}\eta\int_{k\eta}^{K\eta}(g(T-t))^{2}\mathrm{d}t},
\end{align*}
where 
\begin{align*}
&\sum_{j=k+1}^{K}\int_{(j-1)\eta}^{j\eta}\delta_{j}(T-t)\mathrm{d}t
\\
&=\sum_{j=k+1}^{K}\int_{(j-1)\eta}^{j\eta}\frac{\frac{1}{2}e^{-\int_{(j-1)\eta}^{t}f(T-s)\mathrm{d}s}(g(T-t))^{2}}{\frac{1}{m_{0}}e^{-2\int_{0}^{T-t}f(s)\mathrm{d}s}+\int_{0}^{T-t}e^{-2\int_{s}^{T-t}f(v)\mathrm{d}v}(g(s))^{2}\mathrm{d}s}\mathrm{d}t
\\
&\qquad\qquad\qquad\qquad\qquad\qquad\qquad
-\sum_{j=k+1}^{K}\int_{(j-1)\eta}^{j\eta}\frac{\eta}{4}(g(T-t))^{4}(L(T-t))^{2}\mathrm{d}t
\\
&=\sum_{j=k+1}^{K}\int_{(j-1)\eta}^{j\eta}\frac{\frac{1}{2}e^{-\frac{1}{2}\int_{(j-1)\eta}^{t}\beta(T-s)\mathrm{d}s}\beta(T-t)}{\frac{1}{m_{0}}e^{-\int_{0}^{T-t}\beta(s)\mathrm{d}s}+\int_{0}^{T-t}e^{-\int_{s}^{T-t}\beta(v)\mathrm{d}v}\beta(s)\mathrm{d}s}\mathrm{d}t
\\
&\qquad\qquad\qquad\qquad\qquad\qquad\qquad
-\int_{k\eta}^{K\eta}\frac{\eta}{4}(\beta(T-t))^{2}(L(T-t))^{2}\mathrm{d}t
\\
&=\sum_{j=k+1}^{K}\int_{(j-1)\eta}^{j\eta}\frac{\frac{1}{2}e^{-\frac{1}{2}\int_{(j-1)\eta}^{t}\beta(T-s)\mathrm{d}s}\beta(T-t)}{\left(\frac{1}{m_{0}}-1\right)e^{-\int_{0}^{T-t}\beta(s)\mathrm{d}s}+1}\mathrm{d}t
-\int_{k\eta}^{K\eta}\frac{\eta}{4}(\beta(T-t))^{2}(L(T-t))^{2}\mathrm{d}t.
\end{align*}
We can further compute that
\begin{align*}
&\sum_{j=k+1}^{K}\int_{(j-1)\eta}^{j\eta}\frac{\frac{1}{2}e^{-\frac{1}{2}\int_{(j-1)\eta}^{t}\beta(T-s)\mathrm{d}s}\beta(T-t)}{\left(\frac{1}{m_{0}}-1\right)e^{-\int_{0}^{T-t}\beta(s)\mathrm{d}s}+1}\mathrm{d}t
\\
&\geq
e^{-\frac{1}{2}\eta\max_{0\leq t\leq T}\beta(t)\mathrm{d}t}
\int_{k\eta}^{K\eta}\frac{\frac{1}{2}\beta(T-t)}{\left(\frac{1}{m_{0}}-1\right)e^{-\int_{0}^{T-t}\beta(s)\mathrm{d}s}+1}\mathrm{d}t
\\
&=\frac{1}{2}e^{-\frac{1}{2}\eta\max_{0\leq t\leq T}\beta(t)\mathrm{d}t}
\log\left(m_{0}e^{\int_{0}^{(K-k)\eta}\beta(s)\mathrm{d}s}+1-m_{0}\right).
\end{align*}
Moreover, we can compute that
\begin{align*}
&\left(\int_{(k-1)\eta}^{k\eta}\left[\frac{1}{2}e^{\int_{t}^{k\eta}f(T-s)\mathrm{d}s}(g(T-t))^{2}L(T-t)\right]^{2}\mathrm{d}t\right)^{1/2}
\\
&\leq\max(1,L_{0})\left(\int_{(k-1)\eta}^{k\eta}e^{\int_{t}^{k\eta}\beta(T-v)\mathrm{d}v}(\beta(K\eta-t))^{2}\mathrm{d}t\right)^{1/2}
\\
&\leq\max(1,L_{0})\max_{(k-1)\eta\leq t\leq k\eta}\beta(K\eta-t)
\left(\int_{(k-1)\eta}^{k\eta}e^{\int_{t}^{k\eta}\beta(T-v)\mathrm{d}v}\beta(K\eta-t)\mathrm{d}t\right)^{1/2}
\\
&=\max(1,L_{0})\max_{(k-1)\eta\leq t\leq k\eta}\beta(K\eta-t)\left(e^{\int_{(k-1)\eta}^{k\eta}\beta(T-v)\mathrm{d}v}-1\right)^{1/2}.
\end{align*}

By using $T=K\eta$, we conclude that
\begin{align}
&\mathcal{W}_{2}(\mathcal{L}(\mathbf{u}_{K}),p_{0})
\nonumber
\\
&\leq\frac{\Vert\mathbf{x}_{0}\Vert_{L_{2}}}{\sqrt{m_{0}e^{\int_{0}^{K\eta}\beta(s)\mathrm{d}s}+1-m_{0}}}
\nonumber
\\
&\quad+\sum_{k=1}^{K}
\frac{e^{\int_{k\eta}^{K\eta}\frac{1}{2}\beta(K\eta-t)\mathrm{d}t+\int_{k\eta}^{K\eta}\eta\max(1,L_{0}^{2})(\beta(K\eta-t))^{2}\mathrm{d}t+\frac{1}{2}L_{1}\eta\int_{k\eta}^{K\eta}\beta(K\eta-t)\mathrm{d}t}}
{\left(m_{0}e^{\int_{0}^{(K-k)\eta}\beta(s)\mathrm{d}s}+1-m_{0}\right)^{\frac{1}{2}e^{-\frac{1}{2}\eta\max_{0\leq t\leq T}\beta(t)}}}
\nonumber
\\
&\qquad
\cdot
\Bigg(\left(\frac{L_{1}}{2}\eta\left(1+\Vert\mathbf{x}_{0}\Vert_{L_{2}}+\left(\Vert\mathbf{x}_{0}\Vert_{L_{2}}^{2}+d\right)^{1/2}\right)
+\frac{M}{2}\right)2\left(e^{\int_{(k-1)\eta}^{k\eta}\frac{1}{2}\beta(T-v)\mathrm{d}v}-1\right)
\nonumber
\\
&\quad
+\sqrt{\eta}\max(1,L_{0})\max_{(k-1)\eta\leq t\leq k\eta}\beta(K\eta-t)\left(e^{\int_{(k-1)\eta}^{k\eta}\beta(T-v)\mathrm{d}v}-1\right)^{1/2}
\nonumber
\\
&\qquad\qquad\cdot
\Bigg(\left(\Vert\mathbf{x}_{0}\Vert_{L_{2}}+\left(\Vert\mathbf{x}_{0}\Vert_{L_{2}}^{2}+d\right)^{1/2}\right)
\left(\frac{1}{2}+\max(1,L_{0})\right)
\int_{T-k\eta}^{T-(k-1)\eta}\beta(s)\mathrm{d}s
\nonumber
\\
&\qquad\qquad\qquad\qquad
+\left(L_{1}T+\Vert\nabla\log p_{0}(\mathbf{0})\Vert\right)\int_{T-k\eta}^{T-(k-1)\eta}\frac{1}{2}\beta(s)\mathrm{d}s\Bigg)
\Bigg).
\end{align}
This completes the proof.
\end{proof}


\subsubsection{Example: $\beta(t)\equiv b$}
We consider the special case 
$\beta(t)\equiv b$ for some $b>0$. 
This includes the special case $\beta(t)\equiv 2$
that is studied in Chen et al. (2023) \cite{chen2023probability}.

\begin{corollary}[Restatement of Corollary~\ref{cor:VP:const:main:paper}]\label{cor:VP:const}
Assume $\beta(t)\equiv b$. 
Then, we have $\mathcal{W}_{2}(\mathcal{L}(\mathbf{u}_{K}),p_{0})\leq\mathcal{O}(\epsilon)$
after
$K=\mathcal{O}\left(\frac{\sqrt{d}}{\epsilon}(\log(\frac{d}{\epsilon}))^{2}\right)$
iterations provided that $M\leq\frac{\epsilon}{\log(\sqrt{d}/\epsilon)}$ and $\eta\leq\frac{\epsilon}{\sqrt{d}\log(\sqrt{d}/\epsilon)}$.
\end{corollary}


\begin{proof}
When $\beta(t)\equiv b$ for some $b>0$, by Corollary~\ref{cor:VP:4} and \eqref{eq:L2-x0}, we can compute that
\begin{align*}
&\mathcal{W}_{2}(\mathcal{L}(\mathbf{u}_{K}),p_{0})
\\
&\leq\frac{\Vert\mathbf{x}_{0}\Vert_{L_{2}}}{\sqrt{m_{0}e^{bK\eta}+1-m_{0}}}
+\sum_{k=1}^{K}
\frac{e^{(K-k)\eta\frac{1}{2}b+(K-k)\eta^{2}\max(1,L_{0}^{2})b^{2}+\frac{1}{2}L_{1}\eta^{2}(K-k)b}}
{\left(m_{0}e^{b(K-k)\eta}+1-m_{0}\right)^{\frac{1}{2}e^{-\frac{1}{2}\eta b}}}
\nonumber
\\
&\qquad
\cdot
\Bigg(\left(\frac{L_{1}}{2}\eta\left(1+\Vert\mathbf{x}_{0}\Vert_{L_{2}}+\left(\Vert\mathbf{x}_{0}\Vert_{L_{2}}^{2}+d\right)^{1/2}\right)
+\frac{M}{2}\right)2\left(e^{\frac{1}{2}b\eta}-1\right)
\nonumber
\\
&\quad
+\sqrt{\eta}\max(1,L_{0})b\left(e^{b\eta}-1\right)^{1/2}
\cdot
\Bigg(\left(\Vert\mathbf{x}_{0}\Vert_{L_{2}}+\left(\Vert\mathbf{x}_{0}\Vert_{L_{2}}^{2}+d\right)^{1/2}\right)
\left(\frac{1}{2}+\max(1,L_{0})\right)
b\eta
\nonumber
\\
&\qquad\qquad\qquad\qquad
+\left(L_{1}T+\Vert\nabla\log p_{0}(\mathbf{0})\Vert\right)\frac{1}{2}b\eta\Bigg)
\Bigg),
\end{align*}
which implies that
\begin{align*}
\mathcal{W}_{2}(\mathcal{L}(\mathbf{u}_{K}),p_{0})
&\leq\mathcal{O}\Bigg(\frac{\sqrt{d}}{e^{\frac{1}{2}bK\eta}}
+e^{K\eta^{2}\max(1,L_{0}^{2})b^{2}+\frac{1}{2}L_{1}\eta^{2}Kb}
\cdot e^{\frac{1}{2}bK\eta\left(e^{\frac{1}{2}\eta b}-1\right)}
\nonumber
\\
&\qquad\qquad\qquad\qquad\qquad
\cdot
K\Bigg(\left(M+L_{1}\eta\sqrt{d}\right)\eta+\eta^{2}\left(\sqrt{d}+K\eta\right)\Bigg)\Bigg)
\nonumber
\\
&\leq\mathcal{O}(\epsilon),
\end{align*}
and the condition $\eta\leq\bar{\eta}$ (where $\bar{\eta}$ is defined in \eqref{bar:eta}) holds
provided that
\begin{equation*}
K\eta=\frac{2}{b}\log\left(\frac{\sqrt{d}}{\epsilon}\right), 
\qquad
M\leq\eta\sqrt{d},
\qquad
\eta\leq\frac{\epsilon}{\sqrt{d}(K\eta)},
\end{equation*}
which implies that 
\begin{equation*}
M\leq\frac{\epsilon}{\log(\sqrt{d}/\epsilon)},
\qquad
\eta\leq\frac{\epsilon}{\sqrt{d}\log(\sqrt{d}/\epsilon)}
\end{equation*}
and $K\geq\mathcal{O}\left(\frac{\sqrt{d}}{\epsilon}(\log(\frac{d}{\epsilon}))^{2}\right)$.
This completes the proof.
\end{proof}

\subsubsection{Example: $\beta(t)=(b+at)^{\rho}$}
We consider the special case 
$\beta(t)=(b+at)^{\rho}$. 
This includes the special case $\beta(t)=b+at$ when $\rho=1$
that is studied in \cite{Ho2020}.
Then we can obtain the following result from Corollary~\ref{cor:VP:4}. 

\begin{corollary}[Restatement of Corollary~\ref{cor:VP:6:main:paper}]\label{cor:VP:6}
Assume $\beta(t)=(b+at)^{\rho}$. 
Then, we have $\mathcal{W}_{2}(\mathcal{L}(\mathbf{u}_{K}),p_{0})\leq\mathcal{O}(\epsilon)$
after
$K=\mathcal{O}\left(\frac{\sqrt{d}}{\epsilon}(\log(\frac{d}{\epsilon}))^{\frac{\rho+2}{\rho+1}}\right)$
iterations provided that $M\leq\frac{\epsilon}{\log(\sqrt{d}/\epsilon)}$ and $\eta\leq\frac{\epsilon}{\sqrt{d}\log(\sqrt{d}/\epsilon)}$.
\end{corollary}


\begin{proof}
When $\beta(t)=(b+at)^{\rho}$, by Corollary~\ref{cor:VP:4} and \eqref{eq:L2-x0}, we can compute that
\begin{align*}
&\mathcal{W}_{2}(\mathcal{L}(\mathbf{u}_{K}),p_{0})
\\
&\leq\frac{\sqrt{2d/m_{0}}+\Vert\mathbf{x}_{\ast}\Vert}{\sqrt{m_{0}e^{\frac{1}{a(\rho+1)}((b+aK\eta)^{\rho+1}-b^{\rho+1})}+1-m_{0}}}
\nonumber
\\
&\qquad
+\sum_{k=1}^{K}
\frac{e^{\frac{1+L_{1}\eta}{2a(\rho+1)}((b+a(K-k)\eta)^{\rho+1}-b^{\rho+1})+\eta\max(1,L_{0}^{2})\frac{1}{a(2\rho+1)}((b+a(K-k)\eta)^{2\rho+1}-b^{2\rho+1})}}{\left(e^{\frac{1}{a(\rho+1)}((b+a(K-k)\eta)^{\rho+1}-b^{\rho+1})}+1-m_{0}\right)^{\frac{1}{2}e^{-\frac{1}{2}\eta(b+aK\eta)^{\rho}}}}
\nonumber
\\
&\qquad
\cdot
\Bigg(\left(\frac{M}{2}+\frac{L_{1}}{2}\eta\left(1+2\left(\sqrt{2d/m_{0}}+\Vert\mathbf{x}_{\ast}\Vert\right)+\sqrt{d}\right)\right)
\nonumber
\\
&\qquad\qquad\qquad\qquad\cdot 
2\left(e^{\frac{1}{2}\frac{\left((b+a(K-k+1)\eta)^{\rho+1}-(b+a(K-k)\eta)^{\rho+1}\right)}{a(\rho+1)}}-1\right)
\nonumber
\\
&\quad+\sqrt{\eta}
\left(\frac{1}{2}+\max(1,L_{0})\right)
\cdot
\left(e^{\frac{\left((b+a(K-k+1)\eta)^{\rho+1}-(b+a(K-k)\eta)^{\rho+1}\right)}{a(\rho+1)}}-1\right)^{1/2}
\nonumber
\\
&\cdot
\Bigg(\left(2\left(\sqrt{2d/m_{0}}+\Vert\mathbf{x}_{\ast}\Vert\right)+\sqrt{d}\right)
\left(\frac{1}{2}+\max(1,L_{0})\right)
\frac{\left((b+a(K-k+1)\eta)^{\rho+1}-(b+a(K-k)\eta)^{\rho+1}\right)}{a(\rho+1)}
\\
&\qquad\qquad
+\left(L_{1}K\eta+\Vert\nabla\log p_{0}(\mathbf{0})\Vert\right)
\frac{1}{2}\frac{\left((b+a(K-k+1)\eta)^{\rho+1}-(b+a(K-k)\eta)^{\rho+1}\right)}{a(\rho+1)}\Bigg)\Bigg).
\end{align*}

We can compute that
\begin{align*}
&\mathcal{W}_{2}(\mathcal{L}(\mathbf{u}_{K}),p_{0})
\nonumber
\\
&\leq\mathcal{O}\Bigg(\frac{\sqrt{d}}{e^{\frac{1}{2a(\rho+1)}(b+aK\eta)^{\rho+1}}}
\nonumber
\\
&\qquad
+\sum_{k=1}^{K}
\frac{e^{\frac{1+L_{1}\eta}{2a(\rho+1)}((b+a(K-k)\eta)^{\rho+1}-b^{\rho+1})+\eta\max(1,L_{0}^{2})\frac{1}{a(2\rho+1)}(b+aK\eta)^{2\rho+1}}}{\left(m_{0}e^{\frac{1}{a(\rho+1)}((b+a(K-k)\eta)^{\rho+1}-b^{\rho+1})}+1-m_{0}\right)^{\frac{1}{2}e^{-\frac{1}{2}\eta(b+aK\eta)^{\rho}}}}
\nonumber
\\
&\qquad\qquad
\cdot
\Bigg(\left(M+L_{1}\eta\left(1+2\left(\sqrt{2d/m_{0}}+\Vert\mathbf{x}_{\ast}\Vert\right)+\sqrt{d}\right)\right)\left(e^{\frac{1}{2}(b+a(K-k+1)\eta)^{\rho}\eta}-1\right)
\nonumber
\\
&\qquad
+\sqrt{\eta}\left(e^{(b+a(K-k+1)\eta)^{\rho}\eta}-1\right)^{1/2}
\cdot
\left(\left(\sqrt{d}+K\eta\right)((K-k+1)\eta)^{\rho}\eta\right)\Bigg)\Bigg)
\nonumber
\\
&\leq\mathcal{O}\Bigg(\frac{\sqrt{d}}{e^{\frac{1}{2a(\rho+1)}(b+aK\eta)^{\rho+1}}}
+e^{\eta\max(1,L_{0}^{2})\frac{1}{a(2\rho+1)}(b+aK\eta)^{2\rho+1}}
\cdot e^{\frac{1}{2a(\rho+1)}(b+aK\eta)^{\rho+1}\left(e^{\frac{1}{2}\eta(b+aK\eta)^{\rho}}-1\right)}
\nonumber
\\
&\qquad\qquad\qquad\qquad\qquad
\cdot
K\Bigg(\left(M+L_{1}\eta\sqrt{d}\right)(K\eta)^{\rho}\eta+\eta^{2}(K\eta)^{2\rho}\left(\sqrt{d}+K\eta\right)\Bigg)\Bigg)
\nonumber
\\
&\leq\mathcal{O}(\epsilon),
\end{align*}
and the condition $\eta\leq\bar{\eta}$ (where $\bar{\eta}$ is defined in \eqref{bar:eta}) holds
provided that
\begin{equation*}
K\eta=\frac{(2a(\rho+1))^{\frac{1}{\rho+1}}}{a}\left(\log\left(\sqrt{d}/\epsilon\right)\right)^{\frac{1}{\rho+1}}-\frac{b}{a}, 
\qquad
M\leq\eta\sqrt{d},
\qquad
\eta\leq\frac{\epsilon}{\sqrt{d}(K\eta)^{\rho+1}},
\end{equation*}
which implies that 
\begin{equation*}
M\leq\frac{\epsilon}{\log(\sqrt{d}/\epsilon)},
\qquad
\eta\leq\frac{\epsilon}{\sqrt{d}\log(\sqrt{d}/\epsilon)}
\end{equation*}
and $K\geq\mathcal{O}\left(\frac{\sqrt{d}}{\epsilon}(\log(\frac{d}{\epsilon}))^{\frac{\rho+2}{\rho+1}}\right)$.
This completes the proof.
\end{proof}

\end{document}